\documentclass[final]{colt2020} 

\title[Leveraging loss predictors in contextual bandit learning]{Taking a hint: How to leverage loss predictors in contextual bandits?}

\coltauthor{%
 \Name{Chen-Yu Wei} \Email{chenyu.wei@usc.edu}\\
 \addr University of Southern California
 \AND
 \Name{Haipeng Luo} \Email{haipengl@usc.edu}\\
 \addr University of Southern California
 \AND
 \Name{Alekh Agarwal} \Email{alekha@microsoft.com}\\
 \addr Microsoft Research, Redmond
}

\usepackage{xspace}
\usepackage{amsmath}
\usepackage{times}
\usepackage{bbm}
\usepackage{algorithm}
\usepackage{framed}
\usepackage{makecell}
\usepackage{enumitem}
\usepackage{nicefrac}

\SetAlgoVlined 
\DontPrintSemicolon

\usepackage[]{color-edits}
\addauthor{HL}{red}
\addauthor{AA}{blue}

\DeclareMathOperator*{\argmin}{argmin} 

\newcommand{\calL}{\mathcal{L}}
\newcommand{\calC}{\mathcal{C}}
\newcommand{\calI}{\mathcal{I}}
\newcommand{\calS}{\mathcal{S}}
\newcommand{\calD}{\mathcal{D}}
\newcommand{\calT}{\mathcal{T}}
\newcommand{\calA}{\mathcal{A}}
\newcommand{\calX}{\mathcal{X}}

\newcommand{\calP}{\mathcal{P}}
\newcommand{\calE}{\mathcal{E}}
\newcommand{\calM}{\mathcal{M}}
\newcommand{\one}{\mathbbm{1}}

\newcommand{\hatcalC}{\widehat{\mathcal{C}}}
\newcommand{\E}{\mathbb{E}}
\newcommand{\R}{\mathbb{R}}
\newcommand{\order}{\mathcal{O}}
\newcommand{\otil}{\widetilde{\mathcal{O}}}
\newcommand{\Omegatil}{\widetilde{\Omega}}

\newcommand{\Reg}{\text{\rm Reg}}
\newcommand{\catoni}{\text{\rm Catoni}}

\newcommand{\hatReg}{\widehat{\Reg}}

\newcommand{\ellhat}{\widehat{\ell}}
\newcommand{\elltil}{\widetilde{\ell}}
\newcommand{\calEhat}{\widehat{\calE}}
\newcommand{\alphahat}{\widehat{\alpha}}

\newcommand{\expfour}{\textsc{Exp4}\xspace}
\newcommand{\minimonster}{\textsc{ILOVETOCONBANDITS}\xspace}
\newcommand{\Var}{\mathbb{V}}
\newcommand{\ERM}{\text{\rm ERM}}

\newcommand{\inner}[1]{\left\langle #1 \right\rangle}

\newcommand{\expfourOAR}{\textsc{Exp4.OAR}\xspace}
\newcommand{\expfourOVAR}{\textsc{Exp4.OVAR}\xspace}
\newcommand{\greedyAR}{\textsc{$\epsilon$-Greedy.AR}\xspace}
\newcommand{\greedyARC}{\textsc{$\epsilon$-Greedy.ARC}\xspace}
\newcommand{\greedyVAR}{\textsc{$\epsilon$-Greedy.VAR}\xspace}
\newcommand{\expfourMOAR}{\textsc{Exp4.MOAR}\xspace}
\newcommand{\ILTCBMARC}{\textsc{ILTCB.MARC}\xspace}

\newcounter{protocol}
\newcounter{algorithm saved}



\makeatletter
\def\renewtheorem#1{%
  \expandafter\let\csname#1\endcsname\relax
  \expandafter\let\csname c@#1\endcsname\relax
  \gdef\renewtheorem@envname{#1}
  \renewtheorem@secpar
}
\def\renewtheorem@secpar{\@ifnextchar[{\renewtheorem@numberedlike}{\renewtheorem@nonumberedlike}}
\def\renewtheorem@numberedlike[#1]#2{\newtheorem{\renewtheorem@envname}[#1]{#2}}
\def\renewtheorem@nonumberedlike#1{
\def\renewtheorem@caption{#1}
\edef\renewtheorem@nowithin{\noexpand\newtheorem{\renewtheorem@envname}{\renewtheorem@caption}}
\renewtheorem@thirdpar
}
\def\renewtheorem@thirdpar{\@ifnextchar[{\renewtheorem@within}{\renewtheorem@nowithin}}
\def\renewtheorem@within[#1]{\renewtheorem@nowithin[#1]}
\makeatother

\renewtheorem{corollary}{Corollary}[theorem] 

\allowdisplaybreaks

\begin{document}
\SetAlgoVlined

\maketitle

\begin{abstract}

We initiate the study of learning in contextual bandits with the help of loss predictors.
The main question we address is whether one can improve over the minimax regret $\mathcal{O}(\sqrt{T})$ for learning over $T$ rounds, when the total error of the predicted losses relative to the realized losses, denoted as $\calE \leq T$, is relatively small.
We provide a complete answer to this question, with upper and lower bounds for various settings: adversarial and stochastic environments, known and unknown $\calE$, and single and multiple predictors.
We show several surprising results, such as
1) the optimal regret is $\mathcal{O}(\min\{\sqrt{T}, \sqrt{\calE}T^\frac{1}{4}\})$ when $\calE$ is known, in contrast to the standard and better bound $\mathcal{O}(\sqrt{\calE})$ for non-contextual problems (such as multi-armed bandits);
2) the same bound cannot be achieved if $\calE$ is unknown, but as a remedy, $\mathcal{O}(\sqrt{\calE}T^\frac{1}{3})$ is achievable; 
3) with $M$ predictors, a linear dependence on $M$ is necessary, even though logarithmic dependence is possible for non-contextual problems.

We also develop several novel algorithmic techniques to achieve matching upper bounds, including
1) a key \emph{action  remapping} technique for optimal regret with known $\calE$,
2) computationally efficient implementation of Catoni's robust mean estimator via an ERM oracle in the stochastic setting with optimal regret,
3) an underestimator for $\calE$ via estimating the histogram with bins of exponentially increasing size for the stochastic setting with unknown $\calE$, and
4) a self-referential scheme for learning with multiple predictors,
all of which might be of independent interest.
\end{abstract}



%

\section{Introduction}
Online learning with the help of loss predictors has been widely studied over the past decade.
In these problems, before making a decision at each round $t$, the learner is given some prediction $m_t$ of the true gradient or loss vector $\ell_t$.
The goal is to ensure regret that is much smaller than the worst-case bound as long as these predictions are indicative of the loss vectors.
For example, for most problems with $\Theta(\sqrt{T})$ minimax regret for learning over $T$ rounds, it has been shown that a more adaptive regret bound of order $\order(\sqrt{\calE})$ is possible, where $\calE = \sum_{t=1}^T \|\ell_t-m_t\|_\infty^2$ is the total error of the predictions, which is at most $\order(T)$ but could be much smaller if a good predictor is available.
Such a bound is achievable for problems with full information feedback~\citep{rakhlin2013online, SteinhardtLi14}, as well as partial information feedback such as multi-armed bandits~\citep{wei2018more} and linear bandits~\citep{rakhlin2013online}.

In contextual bandits~\citep{AuerCeFrSc02, LangfordZh08}, a generalization of multi-armed bandits that has been proven to be useful for applications such as personalized recommendation systems in practice, loss predictors are also commonly used to construct doubly-robust estimators, both for off-policy evaluation~\citep{dudik2014doubly} and online exploration~\citep{AgarwalHsKaLaLiSc14}. The potentially lower variance of these doubly-robust estimators has been used to motivate this line of work, and resulting improvements are well established for policy evaluation in both finite sample~\citep{dudik2014doubly} and asymptotic settings~\citep{robins1995semiparametric}. In the online exploration setting, however, the extent of benefits from a good loss predictor and potential rate improvements beyond the worst case $O(\sqrt{T})$ bound have not been studied at all, despite all the works mentioned above for the simpler non-contextual settings.

In this work, we take the first attempt in addressing this question and provide a rather complete answer on upper and lower bounds for various setups: adversarial and stochastic environments, known and unknown $\calE$, and single and multiple predictors.
The main message is that good predictors indeed help reduce regret for contextual bandits, but {\it not to the same extent as the non-contextual settings}.
Specifically, our contributions are (see also Table~\ref{tab:main} for a summary):

\renewcommand{\arraystretch}{2}
\begin{table}[]
\caption{\small Summary of main results. $T$ is the total number of rounds. For single predictor, $\calE \leq T$ is the total error of predictions. For multiple predictors, $\calE^*$ is the total error of the best predictor and $M$ is the number of predictors.
Dependence on other parameters is omitted.
Note that for the case with known $\calE$ or $\calE^*$, one can achieve the minimum of $\order(\sqrt{T})$ and the stated upper bound by simply comparing the two bounds and choosing between the minimax algorithm and our algorithms.
}
    \label{tab:main}
    \centering
    \resizebox{1\textwidth}{!}{%
    \begin{tabular}{|c|c|c|c|}
    \hline
         & \makecell{Single predictor \\ with known $\calE$}
         & \makecell{Single predictor \\ with unknown $\calE$}
         & \makecell{Multiple predictors \\ with known $\calE^*$}  \\

    \hline
    \makecell{Lower bound \\ for $\calE, \calE^*, M= \order(\sqrt{T})$}
    & \makecell{$\Omega(\sqrt{\calE}T^{\frac{1}{4}})$ \\ {[}Theorem \ref{theorem: lower bound single predictor}{]}}
    &  \makecell{$\order(\sqrt{\calE}T^{\frac{1}{4}})$ is impossible\\ {[}Theorem \ref{theorem: impossibility of unknown V}{]}}
    & \makecell{ $\Omega(\sqrt{\calE^*}T^{\frac{1}{4}} + M)$ \\ {[}Theorem \ref{theorem: lower bound multiple predictors}{]}} \\

    \hline
    \makecell{Upper bound \\ in the adversarial setting}
    & \makecell{$\order(\sqrt{\calE}T^{\frac{1}{4}})$ \\ {[}Theorem \ref{thm:OExp4}{]}}
    &  \makecell{$\order(\sqrt{\calE} T^{\frac{1}{3}})$ \\ {[}Theorem \ref{thm:adaptive_OExp4}{]}}
    & \makecell{$\order(\sqrt{M\calE^*}T^{\frac{1}{4}})$ \\ {[}Theorem \ref{thm:multi_pred_OExp4}{]}} \\

    \hline
    \makecell{Upper bound \\ in the i.i.d. setting with \\ oracle-efficient algorithms}
    & \makecell{$\order(\sqrt{\calE}T^{\frac{1}{4}})$ \\ {[}Theorem \ref{thm:eps_greedy}{]}}
    & \makecell{$\order(\sqrt{\calE}T^{\frac{1}{3}})$ \\ {[}Theorem \ref{thm:adaptive_eps_greedy}{]}}
    & \makecell{$\order(M^{\frac{2}{3}} (\calE^*T)^{\frac{1}{3}})$ \\ {[}Theorem \ref{thm:multi_pred_iid}{]}}\\
    \hline
    \end{tabular}
    }
\end{table}

\begin{itemize}[leftmargin=*]
  \setlength\itemsep{0em}
\item
(Section~\ref{sec:adversarial})
In the adversarial setting where contexts, losses, and predictions are all decided by an adversary,
we show that, somewhat surprisingly, the regret is at least $\Omega(\min\{\sqrt{\calE}T^\frac{1}{4}, \sqrt{T}\})$, and we also provide an algorithm with a matching regret upper bound  when $\calE$ is known.
When $\calE$ is unknown, we show that it is impossible to achieve the same bound,
and as a remedy, we provide an adaptive version of our algorithm with regret $\order(\sqrt{\calE}T^\frac{1}{3})$, which is always sublinear and is better than $\order(\sqrt{T})$ as long as $\calE = o(T^\frac{1}{3})$.
Note that these results are in sharp contrast with the typical bound $\order(\sqrt{\calE})$, for non-contextual problems.
For multi-armed bandits, even with unknown $\calE$,  $\order(\sqrt{\calE})$ is achievable~\citep{wei2018more}, indicating that the difficulty indeed comes from the contexts, and not just the bandit feedback.

\item
(Section~\ref{sec:iid})
In the stochastic setting where contexts, losses, and predictions are jointly i.i.d. samples from a fixed and unknown distribution,
we show the exact same lower and upper bounds with known or unknown $\calE$,
but importantly our algorithms are efficient assuming access to some ERM oracle. This a typical computational model for studying efficient contextual bandits algorithms, and avoiding running time that is polynomial in the number of polices~\citep{LangfordZh08, AgarwalHsKaLaLiSc14, SyrgkanisLuKrSc16}. Somewhat surprisingly, we find that an adaptation of $\epsilon$-greedy exploration is optimal when $\calE = \order(\sqrt{T})$.

\item
(Section~\ref{sec:multiple})
Finally, we extend our results to the setting where $M$ predictors are available and the goal is to improve the $\order(\sqrt{T})$ regret as long as the total error $\calE^*$ of the best predictor is relatively small.
For simplicity we assume known $\calE^*$.
We show a lower bound  $\Omega(\min\{\sqrt{\calE^*}T^{\frac{1}{4}} + M, \sqrt{T}\})$ when $M = \order(\sqrt{T})$, as well as an upper bound of $\order(\min\{(\sqrt{M\calE^*}T^{\frac{1}{4}} + M, \sqrt{T}\})$ for the adversarial setting and an upper bound of $\order(\min\{M^{\frac{2}{3}} (\calE^*T)^{\frac{1}{3}}, \sqrt{T}\})$ for the stochastic setting with an oracle-efficient algorithm.
This is also in contrast with the non-contextual settings where the dependence on $M$ is logarithmic, even with bandit feedback~\citep{rakhlin2013online}.

\end{itemize}
\vspace*{-0.1cm}

Throughout, we focus on finite action and policy sets in this work to cleanly illustrate the key ideas. Extensions to infinite actions and policies are interesting avenues for future work.

\vspace*{0.15cm}
\noindent\textbf{Techniques.}
Our algorithms 
require several novel techniques, briefly summarized below:
\begin{itemize}[leftmargin=*]
  \setlength\itemsep{0em}
\item
Most importantly, all our algorithms rely on an \emph{action remapping} technique, which restricts the algorithm's attention to only a subset of actions at each round.
This subset consists of actions with predicted loss not larger than that of a baseline action (such as the action with the smallest predicted loss) by a certain amount,
and the algorithms pretend that all actions outside this set are just the baseline action.
For the adversarial setting with multiple predictors, we also need to apply a self-referential scheme to find the baseline and construct this subset, an idea similar to sleeping experts~\citep{freund1997using}.
We prove that this action remapping technique reduces both the exploration overhead and the variance of estimators.

\item
Our algorithms for the stochastic setting require using robust mean estimators.
In particular, we use the Catoni's estimator~\citep{catoni2012challenging} and show that it can be implemented efficiently using the ERM oracle, which might be of independent interest and useful for developing oracle-efficient algorithms for other problems.

\item
When $\calE$ is unknown, we construct a novel underestimator of $\calE$ by estimating the histogram of the distribution of $\|\ell_t - m_t\|$ with bins of exponentially increasing size in the stochastic setting.
\end{itemize}

\paragraph{Related work.}
Similar to prior work such as~\citep{rakhlin2013online} (for non-contextual problems),
we consider generic loss predictions given by any predictors as inputs of the algorithm.
A series of works focus on choosing specific predictions based on observed data and deriving data-dependent bounds in terms of the variation of the environment~\citep{HazanKa10, HazanKa11, chiang2012online, chiang2013beating, SteinhardtLi14, wei2018more, bubeck2019improved},
which are themselves useful for applications such as faster convergence to equilibria for game playing~\citep{RakhlinSr13, SyrgkanisAgLuSc15, wei2018more}.
Whether similar applications can be derived based on our results is an interesting future direction.

\expfour is the classic algorithm with optimal regret $\order(\sqrt{T})$ for the adversarial setting~\citep{AuerCeFrSc02}, albeit with running time linear in the number of policies.
For the stochastic setting, the simple $\epsilon$-greedy algorithm is oracle-efficient but with suboptimal regret $\order(T^\frac{2}{3})$~\citep{LangfordZh08}.
Later, \citet{AgarwalHsKaLaLiSc14} proposed an oracle-efficient and optimal algorithm \minimonster.
All these algorithms are building blocks for our methods.

Developing adaptive regret bounds for contextual bandits is relatively under-explored.
The only existing work on contextual learning that considers a similar setting with loss predictors is~\citep[Section~6]{SyrgkanisKrSc16}, but they only consider the easier full-information feedback.
On a different direction, \citet{allen2018make} derived the first small-loss bound for contextual bandits.

Our idea of using robust estimators is inspired by~\citep{krishnamurthy2019contextual}, which studies contextual bandits with continuous actions and uses median-of-means, a standard robust estimator, for a different purpose.
It is unclear whether median-of-means can be implemented efficiently via an ERM oracle.
Instead, we turn to Catoni's estimator~\citep{catoni2012challenging}, which provides a similar concentration guarantee and can be implemented efficiently as we show.


\section{Problem Description and Lower Bounds}
\label{sec:setup}

Contextual bandits is a generalization of the classic multi-armed bandit problem, where before choosing one of the $K$ actions at each round, the learner observes a context from some arbitrary context space $\calX$.
In addition to the context, we consider a variant where a {\it loss predictor} is also available.
Specifically, for each round $t = 1, \ldots, T$,
the environment chooses a context $x_t \in \calX$, a loss vector $\ell_t \in [0,1]^K$, and a loss predictor $m_t \in [0,1]^K$;
the learner then receives $x_t$ and $m_t$;
finally, the learner chooses an action $a_t \in [K]$ and observes its loss $\ell_t(a_t)$.

We consider both the adversarial setting and the stochastic setting.
In the former, the sequence $(x_t, \ell_t, m_t)_{1:T}$ can be arbitrary and even depend on the learner's strategy.
For simplicity we assume it is decided ahead of time before the game starts (also known as the oblivious setting).
In the latter, each triple $(x_t, m_t, \ell_t)$ is drawn independently from a fixed and unknown distribution $\calD$.

As in the standard contextual bandits setup, the learner has access to some fixed policy class $\Pi \subseteq [K]^\calX$, assumed to be finite (for simplicity) with cardinality $N$, and her goal is to minimize the (pseudo) regret against the best fixed policy:
$$
  \textstyle  \Reg(\calA) \triangleq \max_{\pi\in\Pi} \E\big[\sum_{t=1}^T \ell_t(a_t) - \sum_{t=1}^T \ell_t(\pi(x_t))\big],
$$
where the expectation is with respect to the randomness of the learner, denoted as the algorithm $\calA$ which chooses $a_1, \ldots, a_T$, and also that of the environment in the stochastic case.
When it is clear from the context, we omit the dependence on $\calA$ and simply denote the regret by $\Reg$.
It is well-known that the optimal worst-case regret is $\order(\sqrt{dT})$ where we define $d \triangleq K\ln N$.\footnote{%
Throughout the paper, we do not make an effort to optimize the dependence on $K$ and $\ln N$. For example, we often relax $K^2\ln N$ by $d^2$ for ease of presentation.
For most discussions, we also ignore the dependence on $d$ and only focus on the dependence on $T$ and $\calE$.
}
The key question we address in this work is whether one could improve upon this worst-case  bound when the predictor is accurate.
More specifically, we denote the total loss of the predictor by $\calE \triangleq \sum_{t=1}^T  \|\ell_t-m_t\|_\infty^2$ for the adversarial setting and $\calE \triangleq T\E_{(x, \ell, m)\sim\calD}\left[\|\ell-m\|_\infty^2\right]$ for the stochastic setting, and we ask the following question:
\[
\noindent \textbf{(Q1)} \quad \textit{Can we improve the regret over $\order(\sqrt{dT})$ if $\calE = o(T)$? }
\]
Note that for the special case of multi-armed bandits where $\Pi$ consists of $K$ constant mappings that always pick one of the $K$ actions (that is, contexts are ignored), \citet{wei2018more} show that $\order(\sqrt{d\calE})$ is achievable,
an improvement over $\order(\sqrt{dT})$ {\it as long as} $\calE = o(T)$.
A natural guess would be that the same holds true for contextual bandits.
However, somewhat surprisingly, in the following theorem we show that this is not the case (proofs for all lower bounds are deferred to Appendix~\ref{app:lower_bounds}).

\begin{theorem}
    \label{theorem: lower bound single predictor}
    For any algorithm and any value $V\in[0,T]$, there exists a (stochastic or adversarial) environment with $\calE\leq V$ and $N=\Theta(\sqrt{KT})$ such that $\Reg(\calA) = \Omegatil\big(\min\big\{ \sqrt{V}(KT)^{\frac{1}{4}}, \sqrt{KT} \big\}\big)$.\footnote{Note that while seemingly a lower bound for the stochastic environments should imply the same for the adversarial environments, there is a subtle technical difference due to the slightly different definitions of $\calE$ in these two cases. We provide proofs for both environments.}
\end{theorem}

This theorem gives a negative answer to \textbf{(Q1)} when $\calE = \Omega(\sqrt{T})$.
Even when $\calE = \order(1)$, the theorem shows that the best one can achieve is $\order(T^{\frac{1}{4}})$, a sharp contrast with the non-contextual case. Note that we require $N \geq K$ due to~\citet{wei2018more}, but perhaps the condition of $N=\Theta(\sqrt{KT})$ can be further weakened.
In Sections~\ref{sec:adversarial} and~\ref{sec:iid}, we develop algorithms with matching upper bounds for adversarial and stochastic environments respectively, thus completely answering \textbf{(Q1)} and confirming that in the worst case, loss predictors are helpful {\it if and only if} $\calE = o(\sqrt{T})$.

\paragraph{Robustness when $\calE$ is unknown.}
One shortcoming of our algorithms with matching upper bounds is that they require knowing the value of $\calE$, which is clearly undesirable in practice.
Put differently, for each possible value of $\calE$, we need a different setting of algorithm parameters to achieve the optimal bound.
Therefore, the next general question we ask is:
\[
\noindent
\textbf{(Q2)}\; \textit{Is there an algorithm with regret $o(\sqrt{T})$ simultaneously for all environments with $\calE = o(\sqrt{T})$?}
\]

One standard method in online learning to deal with unknown parameters is the so-called doubling trick, which is applicable even for some partial-information settings~\citep{HazanKa11, wei2018more}.
However, we show yet another surprising result that the answer to \textbf{(Q2)} is {\it no}.

\begin{theorem}
    \label{theorem: impossibility of unknown V}
    If an algorithm $\calA$ achieves $\Reg(\calA)=o(\sqrt{T})$ for all environments with $\calE=0$, then there exists another environment with $\calE=o(\sqrt{T})$ and $N=\Omega(T)$ for which $\Reg(\calA)=\omega(\sqrt{T})$.
    Thus, no algorithm can achieve $\Reg(\calA)=\order\big(\min\big\{ \sqrt{\calE}(dT)^{\frac{1}{4}}, \sqrt{dT} \big\}\big)$ simultaneously for all $\calE$.
\end{theorem}

The theorem asserts that no algorithm can improve over $\order(\sqrt{T})$ when good predictors are available while simultaneously maintaining $\order(\sqrt{T})$ worst-case robustness.
As a remedy, nevertheless, we develop adaptive versions of our algorithms with regret $\order(\sqrt{\calE}T^\frac{1}{3})$ for all environments simultaneously.
This bound is $o(\sqrt{T})$ whenever $\calE = o(T^\frac{1}{3})$ and at the same time provides a robustness guarantee of $\order(T^\frac{5}{6})$.
As a comparison, a bound of order $\order(\calE T^\frac{1}{4})$, achievable by naively setting the parameters of our algorithms independent of $\calE$, is $o(\sqrt{T})$ only when $\calE = o(T^\frac{1}{4})$, and more importantly could be linear when $\calE$ is large and thus provides no robustness guarantee at all.

\paragraph{Learning with multiple predictors.}
Having a complete understanding of the single predictor case, we further consider a more general setup where instead of receiving one predictor $m_t$, the learner receives $M$ predictors $m_t^1, \ldots, m_t^M \in [0,1]^K$ at the beginning of each round.
In the adversarial setting, these are decided ahead of time by an adversary, and we denote by $\calE^* \triangleq \min_{i\in [M]}\sum_{t=1}^T \|\ell_t-m_t^i\|_\infty^2$, the total error of the best predictor.
On the other hand, for the stochastic setting, each tuple $(x_t, \ell_t, m_t^1, \ldots, m_t^M)$ is an i.i.d. sample of a fixed distribution $\calD$, and we denote by $\calE^* \triangleq T\min_{i \in [M]} \E_{(x, \ell, m^{1:M})\sim\calD}\left[\|\ell-m^i\|_\infty^2\right]$, the expected total error of the best predictor.

The goal of the learner is to improve over the worst-case bound as long as {\it one of the predictors} is reasonably accurate.
Specifically, we ask the following (assuming known $\calE^*$ for simplicity).
\[
\noindent \textbf{(Q3)} \quad \textit{Can we improve the regret over $\order(\sqrt{dT})$ for $\calE^* = o(\sqrt{T})$ and reasonably small $M$?}
\]

For many online learning problems (even those with partial information), achieving $\order(\sqrt{\calE + \ln M})$ is possible~\citep{rakhlin2013online}.
We already know that a worse dependence on $T$ is necessary for contextual bandits, and it turns out that, a worse dependence on $M$ is also unavoidable.
\begin{theorem}
    \label{theorem: lower bound multiple predictors}
    For any algorithm $\calA$ and any $M\leq \sqrt{T}$ and $V^*\leq \sqrt{T}$, there exists an environment (which can be stochastic or adversarial) with $\calE^*\leq V^*$ such that $\Reg(\calA)=\Omegatil(\sqrt{V^*}(KT)^{\frac{1}{4}}+M)$.
\end{theorem}
Compared to the single predictor case, the lower bound has an extra term linear in $M$.
It shows that when $M = \Omega(\sqrt{T})$, there is no hope to improve the worst-case regret even if there is a perfect predictor such that $\calE^* = 0$, again a sharp contrast with the non-contextual case.
In Section~\ref{sec:multiple}, we provide an algorithm with regret $\order(\sqrt{M\calE^*}T^{\frac{1}{4}})$ for the adversarial setting, and another oracle-efficient algorithm with regret $\order(M^{\frac{2}{3}}(\calE^*T)^{\frac{1}{3}})$  for the stochastic setting, answering \textbf{(Q3)} positively to some extent.
(Note that these bounds are larger than the lower bound when $M\leq \sqrt{T}$.) \\

\noindent\textbf{Other notations.}
We use $\otil(\cdot)$ to hide the dependence on $\ln T$, and $\Omegatil(\cdot)$ to hide the dependence on $1/\ln T$;
for an integer $n$, $[n]$ represents $\{1, \ldots, n\}$;
for a random variable $Z$, $\Var[Z]$ denotes its variance;
$\Delta_\Pi$ and $\Delta_K$ are the sets of all distributions over the polices and the actions respectively.

\section{Algorithms for Adversarial Environments}
\label{sec:adversarial}

In this section, we describe our algorithm for the adversarial setting with one predictor.
Similar to existing works on online learning with loss predictors, our algorithm is based on the optimistic Online Mirror Descent (OMD) framework~\citep{rakhlin2013online}.
In particular, with the entropy regularizer, the optimistic OMD update maintains a sequence of distributions
\begin{equation*} Q_1', \ldots, Q_T' \in \Delta_\Pi,~~ \mbox{such that}~~ Q_{t+1}'(\pi) \propto Q_{t}'(\pi)\exp\big(-\eta\ellhat_t(\pi(x_t))\big),\end{equation*}
where $\eta > 0$ is the learning rate and $\ellhat_t$ is some estimator for $\ell_t$.
Upon seeing a context $x_t$ and a predictor $m_t$ at time $t$,
the algorithm computes $Q_t \in \Delta_\Pi$ such that \mbox{$Q_{t}(\pi) \propto Q_{t}'(\pi)\exp\big(-\eta m_t(\pi(x_t))\big)$}, and samples a policy according to $Q_t$ and follows its suggestion to choose an action $a_t$.
Suppose $p_t \in \Delta_K$ is the distribution of $a_t$, then the standard variance-reduced loss estimator is \begin{equation}\ellhat_t(a) = \frac{(\ell_t(a) - m_t(a))\one[a_t=a]}{p_t(a)} + m_t(a).\label{eqn:loss_est}\end{equation}
When $m_t(a) = 0$ for all $t$ and $a$, this is exactly the \expfour algorithm~\citep{AuerCeFrSc02}.

While optimistic OMD with entropy regularizer has been used for problems with full-information feedback~\citep{SteinhardtLi14, SyrgkanisAgLuSc15},
it in fact cannot be directly applied to the bandit setting since typical analysis requires $\ellhat_t(a)-m_t(a)$ to be lower bounded by $-1/\eta$, which does not hold if $\ell_t(a_t) \leq m_t(a_t)$ and $p_t(a_t)$ is too small.
Intuitively this is also the hard case because the predictor over-predicts the loss of a good action and prevents the algorithm from realizing it due to the bandit feedback.
A naive approach of enforcing uniform exploration so that $p_t(a_t) \geq \eta$ contributes $\eta TK$ regret already, which eventually leads to $\Omega(\sqrt{T})$ regret.
Indeed, to get around this issue for multi-armed bandits, \citet{wei2018more} uses a different regularizer called log-barrier, but this does not work for contextual bandits either since it inevitably introduces polynomial dependence on the number of policies $N$ for the regret.

\paragraph{Our solutions.}
Our first key observation is that, despite the range of the loss estimators, Optimistic Exp4 in fact {\it always} guarantees the following ({\it cf.} Lemma~\ref{lem:key_observation}): for any $\pi^* \in \Pi$,
\begin{equation}\label{eqn:key_observation}
\sum_{t=1}^T\sum_{\pi \in\Pi} Q_t(\pi) \ellhat_t(\pi(x_t)) - \sum_{t=1}^T\ellhat_t(\pi^*(x_t))
\leq \frac{\ln N}{\eta}  + 2\eta \sum_{t=1}^T (\ellhat_t(a_t) - m_t(a_t))^2 .
\end{equation}
Readers familiar with the \expfour analysis would find that $p_t(a_t)$ is missing in the last term compared to the standard analysis when $\ellhat_t(a) - m_t(a) \geq -1/\eta$ holds.
To see why Eq.~\eqref{eqn:key_observation} is useful, first take expectation (over $a_t$) on both sides so the last term is bounded by $2\eta K \sum_t \frac{\|\ell_t - m_t\|_\infty^2}{\min_a p_t(a)}$.
Then consider enforcing uniform exploration so that $p_t(a) \geq \mu/K$ holds for some $\mu \in [0,1]$.
Since this contributes $\mu T$ extra regret, using Eq.~\eqref{eqn:key_observation} we have $\Reg = \order(\frac{\ln N}{\eta} + \frac{\eta K^2\calE}{\mu} + \mu T)$, which,  with the optimal tuning of $\eta$ and $\mu$, already gives a nontrivial bound $\Reg = \order((\calE T)^\frac{1}{3})$!
This bound is also $o(\sqrt{T})$ whenever $\calE = o(\sqrt{T})$, but is worse than the bound $\Reg = \order(\sqrt{\calE}T^\frac{1}{4})$ we are aiming for.

To further improve the algorithm, we introduce a novel \emph{action remapping} technique.
Specifically, let $a_t^* = \argmin_{a\in [K]} m_t(a)$ be the action with smallest predicted loss and let $\calA_t$ (Equation~\ref{eqn:mapping}) be the set of actions with predicted loss not larger than that of $a_t^*$ by $\sigma$, for some threshold $\sigma \geq 0$.
Then, we rename the actions according to a mapping $\phi_t: [K]\rightarrow \calA_t$ such that $\phi_t(a) = a$ for $a\in\calA_t$ and $\phi_t(a) = a_t^*$ for $a\notin \calA_t$.
In other words, we pretend that every action outside $\calA_t$ was just $a_t^*$.
We call our algorithm \expfourOAR and show its pseudocode in Algorithm~\ref{alg:OExp4}.

\setcounter{AlgoLine}{0}
\begin{algorithm}[t]
    \caption{\expfourOAR: Optimistic \expfour with Action Remapping}
    \label{alg:OExp4}
\textbf{Parameter}: learning rate $\eta>0$, threshold $\sigma>0$, exploration probability $\mu\in[0,1]$.

\textbf{Initialize}: $Q_1'(\pi) = \nicefrac{1}{N}$ for all $\pi\in\Pi$.

\For{$t=1, \ldots, T$}{
\nl    Receive $x_t$ and $m_t$. Define $a_t^* = \argmin_{a\in [K]} m_t(a)$,
    \begin{equation}\label{eqn:mapping}
    \calA_t = \{a\in [K]: m_{t}(a)\leq m_{t}(a_t^*) + \sigma \},\quad\text{and}\quad
    \phi_t(a) = \begin{cases}a, &\text{if $a\in\calA_t$,} \\ a_t^*, &\text{otherwise.}\end{cases}
    \end{equation}
\nl    Calculate $Q_t\in\Delta_\Pi$: \quad  $Q_{t}(\pi) \propto Q_{t}'(\pi)\exp\left(-\eta m_{t}\left(\phi_t(\pi(x_t))\right) \right)$.  \label{line:OMD1}

\nl    Calculate $p_t \in \Delta_K$: \quad
$
        p_t(a) =
        (1-\mu) \sum_{\pi: \phi_t(\pi(x_t))=a} Q_t(\pi) + \frac{\mu}{|\calA_t|}\one[a\in\calA_t]
$.

\nl    Sample $a_t\sim p_t$ and receive $\ell_t(a_t)$.

\nl    Construct estimator: \quad $\ellhat_t(a)= \frac{\ell_t(a)-m_t(a)}{p_t(a)}\one[a_t=a] + m_t(a)   $ for all  $a\in \calA_t$.

\nl    Calculate $Q_{t+1}'\in\Delta_\Pi$: \quad $Q_{t+1}'(\pi) \propto Q_t'(\pi)\exp\left(-\eta \ellhat_t\left(\phi_t(\pi(x_t))\right)\right).$ \label{line:OMD2}
}
\end{algorithm}

To see why this action remapping is useful,
first consider the regret compared to $\sum_{t=1}^T \ell_t(\phi_t(\pi^*(x_t)))$ due to exploration.
Note that we only explore actions in $\calA_t$ and all actions in this set have predicted loss $\sigma$-close to each other.
Therefore, exploration leads to regret $\mu T\sigma + 2\mu\sum_t \|\ell_t - m_t\|_\infty \leq \mu T\sigma + 2\mu\sqrt{\calE T}$, instead of $\mu T$ compared to the naive approach.
On the other hand, the bias due to remapping $\ell_t(\phi_t(\pi^*(x_t))) - \ell_t(\pi^*(x_t))$ is either zero if $\pi^*(x_t) \in \calA_t$ or at most $2\|\ell_t - m_t\|_\infty - \sigma$ otherwise (by adding and subtracting $m_t(a_t^*)$ and $m_t(\pi^*(x_t))$).
Using the AM-GM inequality and summing over $t$ gives $\calE/\sigma$.
Combining everything we prove the following theorem.

\begin{theorem}\label{thm:OExp4}
\expfourOAR (Algorithm~\ref{alg:OExp4}) ensures $\Reg \leq \frac{\ln N}{\eta}  + \frac{2\eta K^2\calE}{\mu} + \mu T\sigma + 2\mu\sqrt{\calE T} + \frac{\calE}{\sigma}$.
Picking $\mu=\min\left\{\frac{d}{\sqrt{T}},1\right\}$, $\eta=\sqrt{\frac{\mu \ln N}{K^2\calE}}$, and $\sigma = \sqrt{\frac{\calE}{\mu T}}$ gives
$
\Reg = \order\big(\sqrt{d\calE}(T)^{\frac{1}{4}} + d\sqrt{\calE}\big).
$
\end{theorem}

See Appendix~\ref{app:OExp4} for the complete proof.
This theorem indicates that whenever the predictor is good enough with $\calE = o(\sqrt{T})$, our algorithm improves over \expfour and achieves $o(\sqrt{T})$ regret.
Note that this bound requires setting the parameters in terms of the quantity $\calE$,
and in the case when $\calE = \Omega(\sqrt{T})$, one can simply switch to \expfour and achieve regret $\order(\sqrt{dT})$.
Therefore, our result indeed matches the lower bound stated in Theorem~\ref{theorem: lower bound single predictor} (except for a slightly worse dependence on $d$).


\noindent\textbf{Adaptive version with unknown $\calE$.} Next, we discuss the case when $\calE$ is unknown.
Recall that there is no hope to maintain the same bound of Theorem~\ref{thm:OExp4} in this case, as indicated by Theorem~\ref{theorem: impossibility of unknown V}.
Standard doubling trick does not work due to the large magnitude of loss estimators (more specifically, the last round before each restart causes some technical problems),
even though it works for non-contextual problems with bandit feedback~\citep{HazanKa11, wei2018more}.

In light of Eq.~\eqref{eqn:key_observation}, our solution is to use a time-varying learning rate $\eta_t$ that is roughly of order $\big(\sum_{s\leq t}(\ellhat_s(a_s) - m_s(a_s))^2 \big)^{-\nicefrac{1}{2}}$ to minimize the right hand side of Eq.~\eqref{eqn:key_observation} for each time. While standard analysis requires using the same learning rate in Line~\ref{line:OMD1} and Line~\ref{line:OMD2}, due to technical issues we are unable to do so while achieving the desired regret bound. Instead, we use $\eta_{t-1}$ in Line~\ref{line:OMD1} and $\eta_t$ in Line~\ref{line:OMD2}, and carefully bound the bias introduced by this learning rate mismatch. More details are provided in Appendix~\ref{app:adaptive_OExp4}.
Our algorithm (Algorithm~\ref{alg:OExp4-adaptive} in Appendix~\ref{app:adaptive_OExp4}) is completely adaptive, requiring no prior information about $\calE$. The following theorem gives its regret guarantee.
\begin{theorem}\label{thm:adaptive_OExp4}
\expfourOVAR (Algorithm~\ref{alg:OExp4-adaptive} in Appendix~\ref{app:adaptive_OExp4}) ensures $\Reg = \otil\left(d\sqrt{\frac{\calE}{\mu}} + \mu T\right)$.
Specifically, setting $\mu = \min\left\{1, (d/T)^{\frac{2}{3}}\right\}$ gives $\Reg = \otil\left(d\sqrt{\calE} + \sqrt{\calE}(d^2 T)^{\frac{1}{3}} \right)$.
\end{theorem}
This shows that our algorithm is robust and always ensures sublinear regret, since in the worst case $\Reg = \otil(T^{5/6})$ (when $\calE = T$).
Also, our algorithm improves over \expfour whenever $\calE = o(T^\frac{1}{3})$.

\section{Algorithms for Stochastic Environments}
\label{sec:iid}

In this section, we consider learning in a stochastic environment with one predictor.
Recall that a stochastic environment is parameterized by an unknown distribution $\calD$ such that each triple $(x_t, \ell_t, m_t)$ is an i.i.d. sample from $\calD$ and the total prediction error is $\calE = T\E_{(x, \ell, m)\sim\calD}\left[\|\ell-m\|_\infty^2\right]$.
Clearly, this is a special case of the adversarial environment, and our goal is to derive the same results but with {\it oracle-efficient} algorithms.

Specifically, an ERM oracle is a procedure that takes any set $\calS$ of context-loss pairs $(x, c) \in \calX \times \R^K$ as inputs and outputs a policy $\ERM(\calS) \in \argmin_{\pi\in\Pi} \sum_{(x,c)\in\calS} c(\pi(x))$.
An algorithm is oracle-efficient if its total running time and the number of oracle calls are both
polynomial in $T$ and $d$, excluding the running time of the oracle itself.
Oracle-efficiency has been proven to be impossible for adversarial environments~\citep{hazan2016computational}, but achievable for stochastic environments.
The simplest oracle-efficient algorithm is $\epsilon$-greedy~\citep{LangfordZh08}, with suboptimal regret $\order(T^\frac{2}{3})$.
However, somewhat surprisingly, we are able to build our algorithm on top of $\epsilon$-greedy and achieve optimal results when $\calE = o(\sqrt{T})$.

We first review the $\epsilon$-greedy algorithm and point out the difficulties of improving its regret with loss predictors.
In each round $t$, the algorithm with probability $\mu$ samples an action $a_t$ uniformly at random,
and with probability $1-\mu$ follows the empirically best policy $\pi_t =\ERM\big(\{x_s, \ellhat_s\}_{s<t}\big)$ by choosing $a_t = \pi_t(x_t)$, where $\ellhat_s$ is the standard importance-weighted estimator for round $s$.

By standard concentration arguments (Freedman inequality), it holds with high probability that
the difference between the average estimated loss and the expected loss of any policy $\pi$ is bounded as
$$
\textstyle
\big|\frac{1}{t}\sum_{s=1}^t \ellhat_s(\pi(x_s)) - \E_{(x,\ell,m)\sim\calD}[\ell(\pi(x))]\big|
 \leq \otil\Big(\frac{1}{t}\sqrt{(\ln N)\sum_{s=1}^t \Var_s\big[\ellhat_s(\pi(x_s))\big]} + \frac{d}{\mu t}\Big),
$$
where $\Var_s[\ellhat_s(\pi(x_s))]$ is the conditional variance (given everything before round $s$) and is at most $\nicefrac{K}{\mu}$.
By the optimality of $\pi_t$, it is then clear that the total regret of following the empirically best policy is $\otil\big(\sum_t \big(\sqrt{\nicefrac{d}{\mu t}} + \nicefrac{d}{\mu t}\big)\big) = \otil\big(\sqrt{\nicefrac{dT}{\mu}} + \nicefrac{d}{\mu}\big)$.
Further taking the uniform exploration into account shows that the regret of $\epsilon$-greedy has three components: the {\it variance term} $\otil(\sqrt{\nicefrac{dT}{\mu}})$, the {\it lower-order term} $\otil(\nicefrac{d}{\mu})$, and the {\it exploration term} $\order(\mu T)$.
Picking the optimal $\mu$ gives $\order\big(T^\frac{2}{3}\big)$ regret.
To improve the bound, we improve each of these three terms as described below. \\

\noindent \textbf{Improving variance/exploration terms via action remapping.} One natural idea to improve the variance term
is to deploy the same variance-reduced (also known as doubly-robust)  estimator $\widehat{\ell_t}$ (Eq.~\eqref{eqn:loss_est}) as in the adversarial case.
However, the law of total variance implies:
$$
\textstyle
\Var_t[ \widehat{\ell}_t(\pi(x_t)) ]
    = \E_{x_t,m_t,\ell_t}[\Var_{a_t}[ \widehat{\ell}_t(\pi(x_t)) | x_t,m_t,\ell_t ]] + \Var_{x_t,m_t,\ell_t}[\E_{a_t}[\widehat{\ell}_t(\pi(x_t)) |x_t,m_t,\ell_t ]],
$$
where one can verify that the first term is at most $\frac{K\calE}{\mu T}$,
but the second term is just $\Var_{x_t,m_t,\ell_t}[\ell_t(\pi(x_t))]$ and is not related to $\calE$.
Simply bounding the second term by $1$ leads to $\Omega(\sqrt{T})$ regret already.

We propose to address this issue by first shifting the variance-reduced estimator by $m_t(a_t^*)$, where $a_t^* = \argmin_{a\in [K]} m_t(a)$ is again the action with the smallest predicted loss.
In other words, we use a new biased estimator:
$
\elltil_t(a) = \ellhat_t(a) - m_t(a_t^*) = \frac{\ell_t(a)-m_t(a)}{p_t(a)}\one[a_t=a] + m_t(a)  - m_t(a_t^*).
$
Moreover, we apply the same action remapping technique using the mapping $\phi_t: [K]\rightarrow \calA_t$ as in the adversarial case (Eq.~\eqref{eqn:mapping}).
To see why this is useful, note that the variance term now becomes
\begin{align}
    \hspace*{-0.4cm}\Var_t\left[ \widetilde{\ell}_t(\phi_t(\pi(x_t))) \right]
    &\leq \E_{x_t,m_t,\ell_t,a_t}\left[\left(\ellhat_t(\phi_t(\pi(x_t))) - m_t(a_t^*)\right)^2\right] \notag \\
    &\leq 2\mathbb{E}_{x_t,m_t,\ell_t, a_t}\left[\frac{(\ell_{t}(\phi_t(\pi(x_t)))-m_{t}(\phi_t(\pi(x_t))))^2}{p_{t}^2(\phi_t(\pi(x_t)))}\one\left[a_t=\phi_t(\pi(x_t))\right]\right] \notag \\
    &\qquad +2\mathbb{E}_{x_t, m_t}\left[\left(m_{t}(\phi_t(\pi(x_t))) - m_{t}(a_t^*)\right)^2\right] \tag{using $(a+b)^2\leq 2a^2+2b^2$} \notag \\
    &\leq \frac{2K}{\mu}\mathbb{E}_{x_t,m_t,\ell_t} \left[(\ell_{t}(\phi_t(\pi(x_t)))-m_{t}(\phi_t(\pi(x_t))))^2\right] + 2\sigma^2
    \leq \frac{2K\calE}{\mu T} + 2\sigma^2,   \label{eqn: variance calculation}
\end{align}
which improves over the variance term $\Var_t[\widehat{\ell}_t(\pi(x_t))]$ if $\sigma$ is small.
Also note that with action remapping, we only explore actions in $\calA_t$, and thus by the exact same arguments as in the adversarial case, the exploration term also becomes $\mu T \sigma + 2\mu\sqrt{\calE T}$, again better than the naive approach as long as $\sigma$ is small.
Therefore, remapping improves both the variance and the exploration term.

It remains to analyze the bias from both the shifted estimator and the remapping.
The former in fact does not introduce any bias for the regret since the shift $m_t(a_t^*)$ is the same for all actions. The latter introduces total bias $\order(\calE/\sigma)$, again by the same analysis as in the adversarial case.
With these modifications, we achieve $\otil((\calE T)^\frac{1}{3})$ regret already (even with the presence of the lower-order term). This is summarized in the following theorem (see Appendix~\ref{app:eps_greedy} for the proof).

\begin{theorem}\label{thm:eps_greedy_weak}
\greedyAR (Algorithm~\ref{alg:eps_greedy} Option I) ensures $\Reg = \otil\big(\sqrt{\frac{d\calE}{\mu}} + \sigma\sqrt{dT} + \frac{d}{\mu} + \mu T\sigma + \mu\sqrt{\calE T} + \frac{\calE}{\sigma} \big)$.
For $\calE \leq \sqrt{T}$, picking $\mu=\min\big\{\big(\tfrac{d^2}{\calE T}\big)^\frac{1}{3},1\big\}$ and $\sigma = \big(\tfrac{\calE^2}{dT}\big)^\frac{1}{3}$ gives
$
\Reg = \order\big((d\calE T)^{\frac{1}{3}} + \sqrt{d\calE} + d\big).
$
\end{theorem}

\begin{figure}[t]

\begin{algorithm}[H]

\caption{$\epsilon$-Greedy with Action Remapping (and Catoni's estimator)}
\label{alg:eps_greedy}
\textbf{Parameters}:  threshold $\sigma>0$, exploration probability $\mu\in[0,1]$.

\For{$t=1, \ldots, T$}{
    Receive $x_t$ and $m_t$. Define $a_t^*$, $\calA_t$ and $\phi_t$ as in Eq.~\eqref{eqn:mapping}.

    Find $\pi_t \begin{cases} = \argmin_{\pi\in\Pi}\ \ERM\left(\{x_s, \elltil_s \circ \phi_s\}_{s<t}\right), &\text{(Option I, termed \greedyAR)} \\
    \approx \argmin_{\pi\in\Pi}\  \catoni_\alpha\left(\left\{ \elltil_s(\phi_s(\pi(x_s)))\right\}_{s<t}\right) \\ \text{\quad using Algorithm~\ref{alg:binary_search} with $\alpha=\sqrt{\frac{2\ln(TN)}{\left(\sigma^2 t+K\calE /\mu \right)}}$}. &\text{(Option II, termed \greedyARC)}
    \end{cases}$

    Calculate $p_t \in \Delta_K$: \quad $p_t(a) = (1-\mu)\one[a = \phi_t(\pi_t(x_t))] +  \frac{\mu}{|\calA_t|}\one[a\in\calA_t]$.

    Sample $a_t \sim p_t$ and receive $\ell_t(a_t)$.

   Construct estimator: \quad $\elltil_t(a)= \frac{\ell_t(a)-m_t(a)}{p_t(a)}\one[a_t=a] + m_t(a) - m_t(a_t^*)$ for all  $a\in \calA_t$.
}

\end{algorithm}

\begin{algorithm}[H]
\caption{Finding the Policy with the Smallest Catoni's Mean}
\label{alg:binary_search}
\textbf{Input}: context $x_s$, loss estimator $\elltil_s$, remapping function $\phi_s$, for $s=1,\ldots, t-1$, and parameter $\alpha$.
\textbf{Define}: 
$\psi(y) = \begin{cases}
\ln(1+y+y^2/2), &\text{if $y\geq0$,} \\
-\ln(1-y+y^2/2), &\text{else.}
\end{cases}$
\qquad\textbf{Initialize}: $z_\text{right} = \frac{K}{\mu}+1$, $z_\text{left} = -z_\text{right}$.

\While{$z_\text{\rm right} - z_\text{\rm left} \geq 1/T$}{
	Let $z_\text{mid} = (z_\text{left}+z_\text{right})/2$.
	
	Construct $c_s \in \R^K$ for all $s<t$ such that $c_s(a) = \psi\left(\alpha\left(\elltil_s(\phi_s(a)) - z_\text{mid}\right)\right)$. 

Invoke oracle $\pi = \ERM\left(\{x_s, c_s\}_{s<t}\right)$.
	
	\lIf{$\sum_{s<t} c_s(\pi(x_s)) \geq 0$}{$z_\text{left} = z_\text{mid}$, \textbf{else} $z_\text{right} = z_\text{mid}$.}
}

Construct $c_s \in \R^K$ for all $s<t$ such that $c_s(a) = \psi\left(\alpha\left(\elltil_s(\phi_s(a)) - z_\text{right}\right)\right)$.

Return $\pi_t = \ERM\left(\{x_s, c_s\}_{s<t}\right)$.
\end{algorithm}
\end{figure}

\noindent\textbf{Removing the lower-order term via Catoni's estimator.}
To further improve the regret bound to $\otil(\sqrt{\calE}T^\frac{1}{4})$, we need to improve the lower-order term as well.
Fortunately, it turns out that this lower-order term can be completely removed using {\it robust mean estimators} for heavy-tailed distributions, such as median of means, trimmed-mean, and Catoni's estimator (see the survey~\citep{lugosi2019mean}).
In particular, we use Catoni's estimator, as we show that it can be implemented efficiently via the ERM oracle.

More specifically, instead of following the policy with the smallest average estimated loss, we follow the policy with the smallest Catoni's mean:
$
\argmin_{\pi\in\Pi} \catoni_\alpha\big(\big\{ \elltil_s(\phi_s(\pi(x_s)))\big\}_{s<t}\big)
$
where $\catoni_\alpha(y_1, \ldots, y_n)$ is the root of the function $f(z) = \sum_{j=1}^n \psi(\alpha(y_j - z))$ for some increasing function $\psi$ (defined in Algorithm~\ref{alg:binary_search}) and coefficient $\alpha > 0$.
Generalizing the proof of Theorem 5 in~\citet{lugosi2019mean} for i.i.d. random variables to a martingale sequence, we obtain a concentration result without the lower-order term (see Lemma~\ref{lemma: catoni concentration} in Appendix~\ref{app:concentration}).
Furthermore, we prove that a close approximation of this policy can be found efficiently via a binary search invoking $\order(\ln (TK/\mu))$ calls of the ERM oracle, detailed in Algorithm~\ref{alg:binary_search}.

\begin{lemma}\label{lem:binary_search}
Algorithm~\ref{alg:binary_search} invokes the ERM oracle at most $\order(\ln (TK/\mu))$ times and returns a policy $\pi_t$ such that:
$
\catoni_\alpha\left(\left\{ \elltil_s(\phi_s(\pi_t(x_s)))\right\}_{s<t}\right)
\leq \min_{\pi\in\Pi}\catoni_\alpha\left(\left\{ \elltil_s(\phi_s(\pi(x_s)))\right\}_{s<t}\right) + \frac{1}{T}.
$
\end{lemma}

The proof is based on the monotonicity of $\psi$ (Appendix~\ref{app:eps_greedy}).
We remark that this result might be of independent interest and useful for developing oracle-efficient algorithms for other problems.

Combining the two key techniques above, we improve all the three terms and prove the following theorem (see Algorithm~\ref{alg:eps_greedy} for the pseudocode and Appendix~\ref{app:eps_greedy} for the complete proof).

\begin{theorem}\label{thm:eps_greedy}
\greedyARC (Algorithm~\ref{alg:eps_greedy} Option II) ensures $\Reg = \otil\big(\sqrt{\frac{d\calE}{\mu}} + \sigma\sqrt{dT} + \mu T\sigma + \mu\sqrt{\calE T} + \frac{\calE}{\sigma} \big)$.
Picking $\mu=\min\big\{\sqrt{\frac{d}{T}},1\big\}$ and $\sigma = \sqrt{\calE}(dT)^{\frac{-1}{4}}$ gives
$
\Reg = \order\big(\sqrt{\calE}(dT)^{\frac{1}{4}} + \sqrt{d\calE}\big).
$
\end{theorem}

Similarly, this requires setting $\sigma$ in terms of $\calE$, and when $\calE = \Omega(\sqrt{T})$, one could switch to the optimal algorithm~\citep{AgarwalHsKaLaLiSc14} and achieve $\order(\sqrt{T})$ regret.
Therefore, our bound again matches the lower bound in Theorem~\ref{theorem: lower bound single predictor}.
In fact, it also enjoys a better dependence on $d$ compared to the adversarial case (Theorem~\ref{thm:OExp4}).

\subsection{Adaptive version with unknown $\calE$}
When $\calE$ is unknown, the same bound is not achievable (Theorem~\ref{theorem: impossibility of unknown V}) and we relax our goal to achieve a bound that is robust and always sublinear, and at the same time improves over $\order(\sqrt{T})$ when $\calE$ is relatively small.
We achieve this goal with a different approach compared to the adversarial case, by exploiting the stochasticity of the environment so we can directly estimate $\calE$ in the early rounds.
Specifically, we spend the first $B$ rounds for pure exploration (i.e., pick $a_t$ uniformly at random) to collect a set of data $\{a_t, \ell_t(a_t), m_t(a_t)\}_{t\leq B}$.
Then we design a novel {\it underestimator} $\calEhat$ defined as
$$\textstyle
       \calEhat = T\sum_{i=0}^{\lceil \log_2 T \rceil} \left[ \alphahat_i - \frac{30\log T}{B} \right]_+ 2^{-2i},
        \;\;\text{where}\; \alphahat_i = \frac{1}{B}\sum_{t=1}^B \one\left[|\ell_t(a_t)-m_t(a_t)|\in \left(2^{-i-1}, 2^{-i}\right]\right],
$$
and $[\cdot]_+ = \max\{\cdot, 0\}$.
For the rest of the game, we simply run Algorithm~\ref{alg:eps_greedy} with Option I, $\sigma =  \sqrt{\calEhat}(dT)^{-\frac{1}{3}}$, and $\mu=\min\big\{d^{2/3}T^{-1/3},1\big\}$.
Note that here we use the simpler Option I, as it can be verified that even without the lower-order term the regret would still be the same in this case (moreover, Option II requires setting $\alpha$ in terms of $\calE$).

The idea behind this estimator is as follows.
First, $\alphahat_i$ is clearly an unbiased estimator of \mbox{$\alpha_i = \frac{1}{K}\sum_{a=1}^K \Pr\left[ |\ell_t(a)-m_t(a)|\in (2^{-i-1}, 2^{-i}] \right]$}, and is thus basically estimating the histogram of the distribution of $\ell_t - m_t$ with bins of exponentially increasing size.
Therefore, $\sum_i \alphahat_i 2^{-2i}$ is an approximation of $\frac{1}{K}\E\left[\|\ell_t - m_t\|_2^2\right]$.
In the definition of $\calEhat$, we subtract a deviation term $30\log T/B$ from $\alphahat_i$ to make sure that $\calEhat$ is an underestimator.
It turns out that both the idea of underestimating and that of estimating the histogram are critical for the analysis,
allowing us to prove the following guarantee (see Appendix~\ref{app:adaptive_eps_greedy} for the complete pseudocode and proof).
Note that this is the same bound as in the adversarial case (Theorem~\ref{thm:adaptive_OExp4}), ignoring the dependence on $d$.

\begin{theorem}\label{thm:adaptive_eps_greedy}
\greedyVAR (Algorithm~\ref{alg:explore then epsilon greedy} in Appendix~\ref{app:adaptive_eps_greedy}) guarantees $\Reg = \otil\Big(B+ K\sqrt{\frac{\calE T}{B}} + K^2\sqrt{\calE}(dT)^{\frac{1}{3}}\Big)$. Setting $B = T^\frac{1}{3}$ gives $\Reg = \otil( K^2\sqrt{\calE}(dT)^\frac{1}{3})$.
\end{theorem}

\section{Algorithms for Multiple Predictors}
\label{sec:multiple}

Finally, we extend our setting and consider learning with multiple predictors.
That is, the learner receives $M$ predictors $m_t^1, \ldots, m_t^M$ before choosing $a_t$ at each round.
Recall that in the adversarial setting, the total error of the best predictor is measured by
$\calE^* = \min_{i\in [M]}\sum_{t=1}^T \|\ell_t-m_t^i\|_\infty^2$, while in the stochastic setting,
it is measured by $\calE^* = T\min_{i \in [M]} \E_{(x, \ell, m^{1:M})\sim\calD}\left[\|\ell-m^i\|_\infty^2\right]$.

Our goal is to improve over $\order(\sqrt{T})$ regret whenever $\calE^*$ and $M$ are relatively small, assuming $\calE^*$ is known for simplicity.
In both cases, we deploy a natural idea: maintain an active set of predictors $\calP_t \subseteq [M]$ (starting from $\calP_1 = [M]$), and eliminate a predictor from this set whenever its observed total error exceeds $\calE^*$.
We define $m_t(a) = \min_{i \in \calP_t} m_t^i(a)$ to be the smallest predicted loss for action $a$ among the active predictors,
and follow similar ideas of the single predictor case with $m_t$ serving the role of the single predictor, which can be seen as a form of optimism.
In addition to this basic idea, however, extra new techniques are required for the two settings as described below. \\

\noindent\textbf{Adversarial Environments.}
The only extra difference compared to Algorithm~\ref{alg:OExp4} is in the construction of $\calA_t$, the set of actions with predicted loss not larger than that of a baseline by $\sigma$.
In Algorithm~\ref{alg:OExp4}, the baseline is simply the action with the smallest predicted loss $a_t^* = \argmin_{a} m_t(a)$.
However, with multiple predictors, we propose to (essentially) use $a_t$, the action {\it to be chosen} by the algorithm, as the baseline.
Before explaining why this is a good idea, we first point out that this can indeed be efficiently implemented, even though the scheme appears self-referential as $a_t$ itself depends on $\calA_t$.
Indeed, this resembles the idea of sleeping experts~\citep{freund1997using}, if we treat actions outside $\calA_t$ as asleep experts.
For implementation details,
see Algorithm~\ref{alg:multi_pred_OExp4} in Appendix~\ref{app:multi_pred_OExp4}.

The ideas of the analysis are as follows.
Using $a_t$ as the baseline gives that the exploration overhead and the bias introduced by remapping are both in terms of $\sum_{t=1}^T (\ell_t(a_t) - m_t(a_t))^2$, which is of order $\order(M\calE^*)$ because each predictor can contribute at most $\calE^*+1$ before being eliminated (Lemmas~\ref{lem:multi_pred_step2}, \ref{lem:multi_pred_step3} and~\ref{lem:a_t_difference}).
Second, note that we only need to refer to Eq.~\eqref{eqn:key_observation} (instead of the standard analysis) when $\ell_t(a_t) \leq m_t(a_t)$, in which case we have $(\ell_t(a_t) - m_t(a_t))^2/p_t(a_t) \leq (\ell_t(a_t) - m_t^{i^*}(a_t))^2/p_t(a_t)$ by the definition of $m_t$ ($i^*$ is the best predictor).
This allows us to relate the expectation of this term to $\calE^*$ as well.
Put together, we prove the following theorem.

\begin{theorem}\label{thm:multi_pred_OExp4}
With the optimal parameters and $M' = \max\{M, K\}$, \expfourMOAR (Algorithm~\ref{alg:multi_pred_OExp4} in Appendix~\ref{app:multi_pred_OExp4}) ensures
$\Reg = \order\Big(\sqrt{M'\calE^*}(dT)^\frac{1}{4} + \sqrt{dM'\calE^*} + d\Big)$.
\end{theorem}

\noindent\textbf{Stochastic Environments.}
There are two extra differences compared to Algorithm~\ref{alg:eps_greedy} in this case.
First, at the beginning of each round, we check if all predictors are consistent to some extent.
If not, that is, if there exist two predictors who disagree with each other by a large amount on some action, then we simply choose this action deterministically, since this guarantees to reveal which predictor makes a large error for this round.
Second, in the case when all predictors are consistent, instead of doing $\epsilon$-greedy as in Algorithm~\ref{alg:eps_greedy} (which we already show is optimal for single predictor),
we find that we need to resort to the minimax optimal algorithm \minimonster~\citep{AgarwalHsKaLaLiSc14} to better control the variance of the estimator.
We develop a version of it with action remapping and Catoni's estimators.
See Algorithm~\ref{alg:multi_pred_iid} for details.
Combining everything, we prove:

\begin{theorem}\label{thm:multi_pred_iid}
\ILTCBMARC (Algorithm~\ref{alg:multi_pred_iid} in Appendix~\ref{app:multi_pred_iid}) guarantees $\Reg = \order\Big(M^{\frac{2}{3}} d^{\frac{2}{5}} (\calE^*T)^{\frac{1}{3}} \Big)$.
\end{theorem}

As a final remark, we remind the reader of the lower bound $\Omega(\sqrt{\calE^*}T^{\frac{1}{4}}+M)$ for $M\leq \sqrt{T}$ given by Theorem~\ref{theorem: lower bound multiple predictors}.
Our upper bound for the adversarial case has matching dependence on $\calE^*$ and $T$, but not $M$, while our bound for the stochastic case is even looser.
Closing the gap and generalizing the results to unknown $\calE^*$ are two main future directions.


\acks{The authors would like to thank Akshay Krishnamurthy and Chicheng Zhang for introducing the idea of robust mean estimator. Part of this work was done when CYW was an intern at Microsoft Research.
HL and CYW  are supported by NSF Awards IIS-1755781 and IIS-1943607.
}

\bibliography{ref}

\appendix

\section{Concentration Inequalities}\label{app:concentration}
\begin{lemma}[Freedman's inequality, {\it cf.} Theorem~1 of~\citep{beygelzimer2011contextual}]
\label{lemma: freedman}
Let $\mathcal{F}_0 \subset \cdots \subset \mathcal{F}_n$ be a filtration, and $X_1, \ldots, X_n$ be real random variables such that $X_i$ is $\mathcal{F}_i$-measurable, $\E[X_i|\mathcal{F}_{i-1}]=0$, $|X_i|\leq b$, and $\sum_{i=1}^n \E[X_i^2|\mathcal{F}_{i-1}] \leq V$ for some fixed $b \geq 0$ and $V \geq 0$.
Then for any $\delta\in (0,1)$, we have with probability at least $1-\delta$,
\begin{align*}
    \sum_{i=1}^n   X_i \leq 2\sqrt{V_n \log(1/\delta)} + b\log(1/\delta).
\end{align*}
\end{lemma}

\begin{lemma} [Concentration inequality for Catoni's estimator]
\label{lemma: catoni concentration}
Let $\mathcal{F}_0 \subset \cdots \subset \mathcal{F}_n$ be a filtration, and $X_1, \ldots, X_n$ be real random variables such that $X_i$ is $\mathcal{F}_i$-measurable, $\E[X_i|\mathcal{F}_{i-1}]=\mu_i$ for some fixed $\mu_i$, and $ \sum_{i=1}^n \E[(X_i-\mu_i)^2|\mathcal{F}_{i-1}] \leq V$ for some fixed $V$. 
Denote $\mu\triangleq \frac{1}{n}\sum_{i=1}^n \mu_i$ and
let $\widehat{\mu}_{n,\alpha}$ be the Catoni's robust mean estimator of $X_1, \ldots, X_n$ with a fixed parameter $\alpha > 0$, that is,
$\widehat{\mu}_{n,\alpha}$ is the unique root of the function
\begin{align*}
      f(z) = \sum_{i=1}^n \psi(\alpha(X_i-z))
\end{align*}
where
\[
\psi(y) = \begin{cases}
\ln(1+y+y^2/2), &\text{if $y\geq0$,} \\
-\ln(1-y+y^2/2), &\text{else.}
\end{cases}
\]
Then for any $\delta\in (0,1)$,
as long as $n$ is large enough such that $n \geq \alpha^2(V + \sum_{i=1}^n (\mu_i-\mu)^2) + 2\log(1/\delta)$,
we have with probability at least $1-2\delta$,
\begin{align*}
     |\widehat{\mu}_{n,\alpha} - \mu| \leq \frac{\alpha (V+\sum_{i=1}^n (\mu_i - \mu)^2)}{n} + \frac{2\log(1/\delta)}{\alpha n}.
\end{align*}
In particular, if $\mu_1 = \cdots = \mu_n = \mu$, we have\footnote{In all our applications of this lemma, we have $\mu_1 = \cdots = \mu_n = \mu$.}
\begin{align*}
     |\widehat{\mu}_{n,\alpha} - \mu| \leq \frac{\alpha V}{n} + \frac{2\log(1/\delta)}{\alpha n}.
\end{align*}
\end{lemma}

\begin{proof}
The proof generalizes that of~\citep[Theorem 5]{lugosi2019mean} for i.i.d. random variables, following similar ideas used in \citep[Theorem 1]{beygelzimer2011contextual}.
First, one can verify that $\psi(y) \leq \ln(1+y+y^2/2)$ for all $y \in \R$.
Therefore, for any fixed $z \in \R$ and any $i$, we have
\begin{align*}
    &\E_i\left[ \exp\left(\psi(\alpha(X_i-z))\right) \right]   \tag{$\E_i[\cdot]\triangleq \E[\cdot \;|\; \mathcal{F}_{i-1}]$} \\
    &\leq \E_i\left[ 1 + \alpha(X_i-z) + \frac{\alpha^2(X_i-z)^2}{2}\right] \\
    &=  1 + \alpha(\mu_i-z) + \frac{\alpha^2\E_i\left[(X_i-\mu_i)^2\right] + \alpha^2(\mu_i-z)^2}{2} \\
    &\leq \exp\left(\alpha(\mu_i-z) + \frac{\alpha^2\E_i\left[(X_i-\mu_i)^2\right]+\alpha^2(\mu_i-z)^2}{2}\right).  \tag{$1+y\leq e^y$}
\end{align*}
Define random variables $Z_0=1$, and for $i\geq 1$,
\begin{align*}
     Z_i = Z_{i-1}\exp\left(\psi(\alpha(X_i-z))\right)\exp\left(-\left(\alpha(\mu_i-z) + \frac{\alpha^2\E_i\left[(X_i-\mu_i)^2\right]+\alpha^2(\mu_i-z)^2}{2}\right)\right).
\end{align*}
Then the last calculation shows $\E_i[Z_i] \leq Z_{i-1}$.
Therefore, taking expectation over all random variables $X_1, \ldots, X_n$, we have
\begin{align*}
    \E[Z_n]\leq \E[Z_{n-1}]\leq \cdots \leq \E[Z_0]=1.
\end{align*}
Further define
\begin{align*}
    g(z) &\triangleq n\alpha(\mu-z) + \frac{1}{2}\alpha^2\sum_{i=1}^n(\mu_i-z)^2 + \frac{1}{2}\alpha^2V +  \log\left(\frac{1}{\delta}\right)
\end{align*}
and note that $f(z)\geq g(z)$ implies
\begin{align*}
     &\sum_{i=1}^n \psi(\alpha(X_i-z))
     \geq n\alpha(\mu-z) + \frac{1}{2}\alpha^2\sum_{i=1}^n (\mu_i-z)^2 + \frac{1}{2}\alpha^2 \sum_{i=1}^n \E_i\left[(X_i-\mu_i)^2\right]    + \log\left(\frac{1}{\delta}\right) \tag{by the condition of $V$}\\
     &= \sum_{i=1}^n \left(\alpha(\mu_i-z) + \frac{\alpha^2(\mu_i-z)^2 + \alpha^2\E_i\left[(X_i-\mu_i)^2\right]}{2} \right) + \log\left(\frac{1}{\delta}\right), 
\end{align*}
which further implies $Z_n \geq 1/\delta$.
By Markov's inequality, we then have $\Pr[f(z)\geq g(z)]\leq \Pr[Z_n\geq 1/\delta]\leq \Pr[Z_n\geq \E[Z_n]/\delta] \leq \delta$.
Note further that we can rewrite $g(z)$ as 
\begin{align*}
g(z) &= n\alpha(\mu-z) + \frac{1}{2}\alpha^2(nz^2 - 2n\mu z + \sum_{i=1}^n \mu_i^2)+ \frac{1}{2}\alpha^2V +  \log\left(\frac{1}{\delta}\right)\\
&= n\alpha(\mu-z) + \frac{1}{2}\alpha^2 (n(z - \mu)^2 - n\mu^2 + \sum_{i=1}^n \mu_i^2)+ \frac{1}{2}\alpha^2V +  \log\left(\frac{1}{\delta}\right)\\
&= n\alpha (\mu - z) + \frac{1}{2}n\alpha^2 (z-\mu)^2 + \frac{1}{2}\alpha^2\left(\sum_{i=1}^n \mu_i^2 - n\mu^2\right) + \frac{1}{2}\alpha^2V +  \log\left(\frac{1}{\delta}\right)
\end{align*}

Now we pick $z$ to be the smaller root $z_0$ of the quadratic function $g(z)$, that is,
\begin{align*}
z_0 = \mu + \frac{1}{\alpha}\left(1 - \sqrt{1 - \frac{\alpha^2 (V+\sum_{i=1}^n(\mu_i - \mu)^2)}{n} - \frac{2}{n}\log\left(\frac{1}{\delta}\right)} \right)
\end{align*}
(which exists due to the condition on $n$).
By the monotonicity of $f$ and the fact $f(\widehat{\mu}_{n,\alpha}) = 0$ we then have
\[
\Pr\left[\widehat{\mu}_{n,\alpha} \geq z_0 \right]
= \Pr\left[f(z_0) \geq 0 \right]
= \Pr\left[f(z_0) \geq g(z_0) \right]  \leq \delta.
\]
In other words, with probability at least $1 - \delta$, we have
\begin{align*}
\widehat{\mu}_{n,\alpha} - \mu
&\leq \frac{1}{\alpha}\left(1 - \sqrt{1 - \frac{\alpha^2 (V+\sum_{i=1}^n (\mu_i - \mu)^2)}{n} - \frac{2}{n}\log\left(\frac{1}{\delta}\right)} \right) \\
&\leq \frac{1}{\alpha}\left(\frac{\alpha^2 (V+\sum_{i=1}^n (\mu_i - \mu)^2)}{n} + \frac{2}{n}\log\left(\frac{1}{\delta}\right)\right) \tag{$1-\sqrt{1-x}\leq x$ for $x\in[0,1]$}\\
&= \frac{\alpha (V+\sum_{i=1}^n (\mu_i - \mu)^2)}{n} + \frac{2\log(1/\delta)}{\alpha n}.
\end{align*}
Finally, via a symmetric argument one can show that $\mu - \widehat{\mu}_{n,\alpha} \leq \frac{\alpha (V+\sum_{i=1}^n (\mu_i - \mu)^2)}{n} + \frac{2\log(1/\delta)}{\alpha n}$ holds with probability at least $1 - \delta$ as well.
Applying a union bound then finishes the proof.
\end{proof}

\section{Proofs for Lower Bounds}
\label{app:lower_bounds}
In this section, we provide proofs for all the lower bounds discussed in Section~\ref{sec:setup}.  The techniques we use are reminiscent of those in several previous works on bandit problems that prove lower bounds for adaptive regret \citep[Theorem 3]{daniely2015strongly}, switching regret \citep[Theorem 4.1]{wei2016tracking}, and regret bounds in terms of the sparsity of the losses \citep[Theorem 6]{zheng2019equipping}. While their constructions are for adversarial environments, ours are for the i.i.d. case (which is stronger for lower bounds).
To make the proofs concise, we assume that numbers such as $\sqrt{\nicefrac{T}{K}}$ are integers without rounding them. \\ 

\begin{proof}\textbf{for Theorem~\ref{theorem: lower bound single predictor}. }
We first prove that for any algorithm, any $K \geq 2$, any $T\geq 8\times 10^4$, and any value $V\in[0,T]$, there exists a stochastic environment with $\calE\leq V$ and $N = (K-1)\sqrt{\nicefrac{T}{K}}+1$ such that $\Reg = \Omegatil\big(\min\big\{ \sqrt{V}(KT)^{\frac{1}{4}}, \sqrt{KT} \big\}\big)$.
The construction is as follows.
There are $\sqrt{\nicefrac{T}{K}}$ possible context-predictor-loss tuples $\{(x^{(i)}, m^{(i)}, \ell^{(i)})\}_{i=1}^{\sqrt{\nicefrac{T}{K}}}$, and in each round, $(x_t, m_t, \ell_t)$ is uniformly randomly drawn from this set. The policy set $\Pi$ contains $ (K-1)\sqrt{\nicefrac{T}{K}}+1$ policies such that: there is a policy $\pi^{(0)}$ that always chooses action $1$ given any context; other policies are indexed by $(i,k)\in [\sqrt{\nicefrac{T}{K}}] \times \{2,\ldots, K\} $ such that
\[
\pi^{(i,k)}(x) = \begin{cases}
k &\text{if $x = x^{(i)}$,} \\
1 &\text{otherwise.}
\end{cases}
\]

     Now first consider an environment with $m^{(i)} = \ell^{(i)} = (\frac{1}{2}, \frac{1}{2}+\sigma, \ldots, \frac{1}{2}+\sigma)$ for all $i$, where $\sigma=\min\left\{\frac{1}{2}, \frac{\sqrt{V}}{2(KT)^{1/4}}\right\}$.
     Note that $\calE = 0 \leq V$. Under this environment and the given algorithm, if for all $(i,k)\in [\sqrt{\nicefrac{T}{K}}]\times \{2,\ldots, K\}$, the expected total number of times where $(x_t, a_t)=(x^{(i)}, k)$ is larger than $\frac{1}{2}$, then the algorithm's regret against $\pi^{(0)}$ is
     \begin{align*}
         \E\left[\sum_{t=1}^T \ell_t(a_t) - \ell_t(\pi^{(0)}(x_t)) \right]
         &= \E\left[\sum_{t=1}^T \sum_{i=1}^{\sqrt{\nicefrac{T}{K}}} \sum_{k=2}^{K} \one[(x_t,a_t)=(x^{(i)}, k)] \left(\ell_t(k) - \ell_t(1)\right) \right] \\
         &= \E\left[\sum_{t=1}^T \sum_{i=1}^{\sqrt{\nicefrac{T}{K}}} \sum_{k=2}^{K} \one[(x_t,a_t)=(x^{(i)}, k)]\sigma \right] \\
          &\geq \sqrt{\frac{T}{K}}\times (K-1)\times \frac{1}{2}\times \sigma \geq \frac{1}{4}\sqrt{KT}\sigma.
     \end{align*}
On the other hand, if there exists a pair $(i^*,k^*)\in [\sqrt{\nicefrac{T}{K}}]\times \{2,\ldots, K\}$ such that 
\[\E\left[\sum_{t=1}^T \one[(x_t,a_t)=(x^{(i^*)}, k^*)]\right]\leq \frac{1}{2},\] then by Markov's inequality,
\begin{align*}
\Pr\left[ \sum_{t=1}^T \one[(x_t, a_t)= (x^{(i^*)},k^*)]=0 \right]
&= \Pr\left[ \sum_{t=1}^T \one[(x_t, a_t)= (x^{(i^*)},k^*)] < 1 \right] \\
&= 1 - \Pr\left[ \sum_{t=1}^T \one[(x_t, a_t)= (x^{(i^*)},k^*)] \geq 1 \right] \geq \frac{1}{2}.
\end{align*}
That is, with probability at least $\frac{1}{2}$, the learner never chooses action $k^*$ when she sees context $x^{(i^*)}$.
In this case, consider another environment where all $m^{(i)}$ and $\ell^{(i)}$ remain the same except that $\ell^{(i^*)}$ is changed to $(\frac{1}{2}, \frac{1}{2}+\sigma, \ldots, \frac{1}{2}-\sigma, \ldots, \frac{1}{2}+\sigma)$, where $\frac{1}{2}-\sigma$ appears in the $k^*$-th coordinate.
Note that in this new environment we again have $\calE = T\E_{(x, \ell, m)}\left[\|\ell-m\|_\infty^2\right] = \sqrt{TK} \times 4\sigma^2 \leq V$.
Moreover, with probability at least $\frac{1}{2}$ the learner never realizes the change of the environment and behaves exactly the same, since the only way to distinguish the two environments is to pick $k^*$ under context $x^{(i^*)}$.

It remains to calculate the regret of the learner under this new environment.
First, by Freedman's inequality (Lemma~\ref{lemma: freedman}), we have with probability at least $1-\frac{1}{T}$,
\begin{align}
     \sum_{t=1}^T \one[x_t=x^{(i^*)}]\geq \sqrt{KT} - 2\sqrt{\sqrt{KT} \log T} - \log T \geq \frac{\sqrt{KT}}{3}    \label{eqn: requirement to hold}
\end{align}
where the last step uses the condition $K\geq 2$ and $T\geq 8\times 10^4$.
Define events \[E_1=\left\{\sum_{t=1}^T \one[(x_t,a_t)=(x^{(i^*)},k^*)]=0\right\}, E_2=\left\{\sum_{t=1}^T \one[x_t=x^{(i^*)}]\geq \frac{\sqrt{KT}}{3}\right\},\] and use $\E^\prime, \Pr^\prime$ to denote the expectation and probability under the new environment. Now we lower bound the regret against $\pi^{(i^*,k^*)}$ in this environment as
\begin{align*}
    & \E^\prime\left[\sum_{t=1}^T \sum_{i=1}^{\sqrt{\nicefrac{T}{K}}} \one[x_t=x^{(i)}]\left(\ell_t(a_t) - \ell_t(\pi^{(i^*,k^*)}(x^{(i)}))\right) \right] \\
    &\geq {\Pr}' [E_1 \cap E_2] \times \E^\prime\left[\sum_{t=1}^T  \one[x_t=x^{(i^*)}]\left(\ell_t(a_t) - \ell_t(\pi^{(i^*,k^*)}(x^{(i^*)}))\right)  ~\bigg|~E_1, E_2\right] \\
    &= {\Pr}' [E_1 \cap E_2] \times \E^\prime\left[\sum_{t=1}^T  \one[x_t=x^{(i^*)}]\sigma  ~\bigg|~E_1, E_2\right] \\
    &\geq \left(\frac{1}{2}-\frac{1}{T}\right)\times \frac{\sqrt{KT}\sigma}{3}
    \geq \frac{\sqrt{KT}\sigma}{12}.
\end{align*}
To summarize, in at least one of these two environments, the learner's regret is
\[
\Omega(\sqrt{KT}\sigma) = \Omegatil\big(\min\big\{ \sqrt{V}(KT)^{\frac{1}{4}}, \sqrt{KT} \big\}\big),
\]
finishing the lower bound proof for stochastic environments.
For adversarial environments, the only change is to let each tuple $(x^{(i)}, m^{(i)}, \ell^{(i)})$ appear for exactly $\sqrt{T/K}$ times, so that $\calE \leq V$ still holds in these two constructions under the slightly different definition for $\calE$ (which is $\sum_{t=1}^T  \|\ell_t-m_t\|_\infty^2$).
It is clear that the same lower bound holds.
\end{proof}

\begin{proof}\textbf{for Theorem~\ref{theorem: impossibility of unknown V}. }
%
The idea of the proof is similar to that of Theorem~\ref{theorem: lower bound single predictor}.
Assume there is an algorithm that guarantees for some $R =o(\sqrt{KT})$,
\begin{align*}
        \mathbb{E}\left[\sum_{t=1}^T \ell_t(a_t) - \sum_{t=1}^T \ell_t(\pi^*(x_t))\right] \leq R,
\end{align*}
whenever $\calE=0$. Below we show that there is an environment with $\calE=o(\sqrt{KT})$ where the algorithm suffers $\omega(\sqrt{KT})$ regret.

The construction is as follows.
First, let $C$ be a universal constant such that $R+K\leq \sqrt{CKT}$.
Further define $\rho=\frac{R+K}{\sqrt{CKT}}\leq 1$, $\sigma=\frac{1}{2}\rho^{\frac{2}{5}}$, and $L_0=  \rho^{-\frac{3}{5}}$.
There are $\sqrt{\nicefrac{64CT}{K}}$ context-predictor-loss tuples $\{(x^{(i)}, m^{(i)}, \ell^{(i)})\}_{i=1}^{\sqrt{\nicefrac{64CT}{K}}}$, and in each round, $(x_t,m_t,\ell_t)$ is uniformly randomly drawn from this set.  The policy set contains $N=\Theta(T)$ policies such that: there is a policy $\pi^{(0)}$ that always chooses action $1$ given any contexts; other policies are indexed by $(i,j,k)\in [\sqrt{\nicefrac{64CT}{K}}]\times [\sqrt{\nicefrac{64CT}{K}}]\times \{2,\ldots, K\}$ with $i\leq j$ such that
\[
\pi^{(i, j,k)}(x) = \begin{cases}
k, &\text{if $x \in \{x^{(i)}, x^{(i+1)}, \ldots, x^{(j)}\}$.} \\
1, &\text{else.}
\end{cases}
\]


We first consider the algorithm's behavior under the environment with $m^{(i)} = \ell^{(i)} = (\frac{1}{2}, \frac{1}{2}+\sigma, \ldots, \frac{1}{2}+\sigma)$ for all $i$. In this environment, since $\calE=0$, the algorithm must guarantee that
    \begin{align}
        \mathbb{E}\left[\sum_{t=1}^T \one[a_t\neq 1]\right]\leq \frac{R}{\sigma}.  \label{eqn: sparse sampling}
    \end{align}
     This is because every time $a_t\neq 1$, the learner incurs regret $\sigma$ against $\pi^{(0)}$.
%
Next we prove the following fact: there exists $i,j\in[\sqrt{\nicefrac{64CT}{K}}]$ and $k\in \{2,\ldots, K\}$ such that $j-i+1= L_0$ and $\E\left[\sum_{t=1}^T \sum_{s=i}^j \one[x_t=x^{(s)}, a_t=k] \right] \leq \frac{1}{2}$. We prove it by contradiction. Assume that for all $(i,j)$ with $j= i -1 +  L_0$ and all $k\in\{2,\ldots, K\}$, $\E\left[\sum_{t=1}^T\sum_{s=i}^j \one[x_t=x^{(s)}, a_t=k] \right] \geq \frac{1}{2}$. Then we have
     \begin{align*}
         \E\left[\sum_{t=1}^T \one[a_t\neq 1] \right]
         &=\E\left[\sum_{t=1}^T \sum_{s=1}^{\sqrt{\nicefrac{64CT}{K}}}\sum_{k=2}^K \one[x_t=x^{(s)}, a_t=k] \right]\\
         &\geq \frac{\sqrt{\nicefrac{64CT}{K}}}{ L_0  }\times \frac{1}{2}\times (K-1) \\ 
         &\geq \frac{\sqrt{64CKT}}{ 4L_0 } \geq \frac{2R}{\rho L_0} = \frac{R}{\sigma},
     \end{align*}
    which leads to a contradiction (here, we also use the fact $1\leq L_0=\rho^{-\frac{3}{5}} \leq \rho^{-1} \leq \sqrt{\frac{64CT}{K}}$). 
    Therefore, we have shown that there exist $(i^*,j^*)$ with $j^*- i^*+1 = L_0$ and $k^*\in\{2,\ldots, K\}$ such that $\E\left[\sum_{t=1}^T\sum_{s=i^*}^{j^*} \one[x_t=x^{(s)}, a_t=k^*] \right] \leq \frac{1}{2}$. By Markov's inequality, we thus have
\begin{align*}
\Pr\left[\sum_{t=1}^T\sum_{s=i^*}^{j^*} \one[x_t=x^{(s)}, a_t=k^*] = 0\right]
= \Pr\left[\sum_{t=1}^T\sum_{s=i^*}^{j^*} \one[x_t=x^{(s)}, a_t=k^*] < 1\right]
\geq \frac{1}{2}.
\end{align*}

    Now we consider another environment, which is the same as the one above, except that for all $s=i^*, i^*+1, \ldots, j^*$, the $k^*$-th coordinate of $\ell^{(s)}$ is changed from $\frac{1}{2}+\sigma$ to $\frac{1}{2}-\sigma$. Note that in this environment,
\[
\calE = \order\left(L_0\sqrt{KT}\sigma^2\right)=\order(\sqrt{KT}\rho^{\frac{1}{5}}) = o(\sqrt{KT}),
\]
where we use the fact $\rho = o(1)$.
Moreover, with probability at least $1/2$, the algorithm never realizes the change and behaves exactly the same, since the only way to distinguish the two environments is to pick $k^*$ under one of the contexts $x^{(i^*)}, x^{(i^*+1)}, \ldots, x^{(j^*)}$.

    It remains to calculate the regret of the learner under this environment.
    Define events \[E_1=\left\{ \sum_{t=1}^T \sum_{s=i^*}^{j^*} \one[(x_t,a_t)=(x^{(s)}, k^*)] =0 \right\},\] and \[E_2=\left\{ \sum_{t=1}^T \sum_{s=i^*}^{j^*} \one[x_t=x^{(s)}] \geq \frac{ (j^*-i^*+1) \sqrt{\nicefrac{KT}{64C}}}{3} \right\}.\]

Note that in expectation, each context appears $\frac{T}{\sqrt{\nicefrac{64CT}{K}}}=\sqrt{\nicefrac{KT}{64C}}$ times. By Freedman's inequality (Lemma~\ref{lemma: freedman}), with probability at least $1-\frac{1}{T}$,
    \begin{align*}
        \sum_{t=1}^T \sum_{s=i^*}^{j^*}\one[x_t=x^{(s)}]
        &\geq (j^*-i^*+1)\sqrt{\frac{KT}{64C}} - 2\sqrt{(j^*-i^*+1)\sqrt{\frac{KT}{64C}}\log T} - \log T\\
        &\geq \frac{(j^*-i^*+1)\sqrt{\frac{KT}{64C}}}{3}.
    \end{align*}
    when $K\geq 2$ and $T>6\times 10^6C$. That is, $\Pr[E_2] \geq 1 - 1/T$.
Therefore, the expected regret against $\pi^{(i^*, j^*, k^*)}$ in this new environment is lower bounded by
\begin{align*}
    & \E\left[\sum_{t=1}^T \sum_{s=i^*}^{j^*} \one[x_t=x^{(s)}]\left(\ell_t(a_t) - \ell_t(\pi^{(i^*,j^*,k^*)}(x^{(s)}))\right) \right] \\
    &\geq \Pr[E_1 \cap E_2] \times \frac{(j^*-i^*+1)\sqrt{\nicefrac{KT}{64C}}}{3}\sigma \\
    &\geq \left(\frac{1}{2}-\frac{1}{T}\right) \times \frac{\sqrt{\nicefrac{KT}{64C}}}{3} \sigma L_0
    = \Omega\left( \sqrt{\frac{KT}{C}}\rho^{-\frac{1}{5}} \right) = \omega(\sqrt{KT}).
\end{align*}
This finishes the proof.
\end{proof}

\begin{proof}\textbf{for Theorem~\ref{theorem: lower bound multiple predictors}. }
     When $\sqrt{V^*}(KT)^{\frac{1}{4}} \geq M$, we only need to prove a lower bound of $\Omegatil(\sqrt{V^*}(KT)^{\frac{1}{4}})$, which is shown by Theorem~\ref{theorem: lower bound single predictor} already. When $\sqrt{V^*}(KT)^{\frac{1}{4}} \leq M \leq \sqrt{T}$, we construct an stochastic environment below with $\calE^*=0$, $N=M$, and $K=2$, where the regret of the algorithm is $\Omega(M)$.

     The construction is as follows (and is again similar to those in the proofs of Theorems~\ref{theorem: lower bound single predictor} and~\ref{theorem: impossibility of unknown V}). There are $M-1$ different context-predictor-loss tuples $\{x^{(i)}, m^{(0,i)}, \ldots, m^{(M-1,i)}, \ell^{(i)}\}_{i=1}^{M-1}$, and in every round, $(x_t, m_t^0, \ldots, m_t^{M-1}, \ell_t)$ is uniformly randomly sampled from this set. The policy set contains $N=M$ policies $\pi^{(0)}, \ldots, \pi^{(M-1)}$ such that: $\pi^{(0)}$ always chooses action $1$ given any contexts; for $i\in[M-1]$, $\pi^{(i)}(x) = 2$ if $x = x^{(i)}$, and otherwise $\pi^{(i)}(x) = 1$.

     Now consider an environment where $\ell^{(i)} = m^{(0,i)} = (\frac{1}{2}, 1)$ for all $i\in[M-1]$, $m^{(j,i)}=(\frac{1}{2}, 1)$ for all $i,j\in[M-1]$ with $i\neq j$, and $m^{(i,i)}=(\frac{1}{2}, 0)$ for all $i\in[M-1]$. Clearly, the predictor $m^0$ is a perfect predictor in this environment and thus $\calE^*=0$.

     In this environment, if the expected number of times the learner chooses action $2$ is larger than $\frac{M-1}{2}$, then she already suffers an expected regret of $\frac{M-1}{2}\times \frac{1}{2}$ compared to policy $\pi^{(0)}$, which always chooses action $1$. On the other hand, if the expected number of times the learner chooses action $2$ is smaller than $\frac{M-1}{2}$, then there exists an $i^*\in[M-1]$ such that the expected number of times the learner chooses action $2$ upon seeing $x^{(i^*)}$ is less than $\frac{1}{2}$. By Markov's inequality, $\sum_{t=1}^T \one[(x_t,a_t)=(x^{(i^*)}, 2)]=0$ holds with probability at least $\frac{1}{2}$. That is, with probability at least $\frac{1}{2}$, the learner never picks action $2$ when the context is $x^{(i^*)}$.

     Now consider a different environment where the only difference is that the $\ell^{(i^*)}(2)$ is changed from $1$ to $0$. With probability at least $\frac{1}{2}$, the learner does not realizes the change and behaves exactly the same. The expected regret compared to policy $\pi^{(i^*)}$ is thus $\Omega\left(\frac{T}{M-1}\times \frac{1}{2}\right)$ in this new environment. Moreover, notice that in this new environment, $\calE^*=0$ still holds because $m^{i^*}$ now becomes the perfect predictor.

     To sum up, we have shown that when there are $M>1$ predictors, even if $\calE^*=0$, the learner has to suffer $\Omega\left(\min\left\{ M-1, \frac{T}{M-1} \right\}\right) = \Omega(M)$ regret.
\end{proof}


\section{Omitted Details for Adversarial Environments}
\label{app:adversarial}

In this section, we provide omitted details for the adversarial case, including the proof of Theorem~\ref{thm:OExp4} on the guarantee of Algorithm~\ref{alg:OExp4} for the case with single predictor and known $\calE$ (Section~\ref{app:OExp4}), the adaptive version of Algorithm~\ref{alg:OExp4} and its analysis when $\calE$ is unknown (Section~\ref{app:adaptive_OExp4}),
and the algorithm and analysis for multiple predictors (Section~\ref{app:multi_pred_OExp4}).

\subsection{Proof of Theorem~\ref{thm:OExp4}}
\label{app:OExp4}

We first prove a lemma showing a somewhat non-conventional analysis of the optimistic \expfour update.
We denote the KL divergence of two distributions $Q$ and $P$ by $D(Q, P) = \sum_{\pi \in \Pi} Q(\pi)\ln\frac{Q(\pi)}{P(\pi)}$.

\begin{lemma}\label{lem:key_observation}
For any $\eta > 0$, $\calM_t, \calL_t \in \R^N$, and distribution $Q_t' \in \Delta_\Pi$,
define two distributions $Q_t, Q_{t+1}' \in \Delta_\Pi$ such that
\begin{align}
Q_t(\pi) &\propto Q_t'(\pi)\exp\left(-\eta \calM_t(\pi)\right),   \nonumber \\
Q_{t+1}'(\pi) &\propto Q_t'(\pi)\exp\left(-\eta \calL_t(\pi)\right).   \label{eqn: learning rate demo}
\end{align}
Then there exists $\xi_t \in \Delta_\Pi$ such that for any $Q^* \in \Delta_\Pi$, we have
\begin{equation}\label{eqn:exp4_bound1}
\inner{Q_t - Q^*, \calL_t} \leq \frac{D(Q^*, Q_t') - D(Q^*, Q_{t+1}')}{\eta}
+ 2\eta \sum_{\pi \in \Pi} \xi_t(\pi) \left(\calL_t(\pi) - \calM_t(\pi)\right)^2.
\end{equation}
Moreover, if $\calL_t(\pi) - \calM_t(\pi) \geq -\frac{1}{\eta}$ holds for all $\pi$, then we have for any $Q^* \in \Delta_\Pi$,
\begin{equation}\label{eqn:exp4_bound2}
\inner{Q_t - Q^*, \calL_t} \leq \frac{D(Q^*, Q_t') - D(Q^*, Q_{t+1}')}{\eta}
+ \eta \sum_{\pi \in \Pi} Q_t(\pi)\left(\calL_t(\pi) - \calM_t(\pi)\right)^2.
\end{equation}
\end{lemma}

\begin{proof}
First, we rewrite the updates in the standard optimistic online mirror descent framework:
$Q_t= \argmin_{Q\in \Delta_\Pi} F_t(Q)$ and $ Q_{t+1}' = \argmin_{Q\in \Delta_\Pi} F_t'(Q)$ where
\begin{align}
F_t(Q) &= \eta\inner{Q, \calM_t} + D(Q, Q_t'),  \nonumber \\
F_t'(Q) &= \eta\inner{Q, \calL_t} + D(Q, Q_t').  \label{eqn: OOMD update}
\end{align}
Applying Lemma~6 of~\citep{wei2018more} shows
\[
\inner{Q_t - Q^*, \calL_t} \leq \frac{D(Q^*, Q_t') - D(Q^*, Q_{t+1}')}{\eta}  + \inner{Q_t - Q_{t+1}', \calL_t - \calM_t} - \frac{1}{\eta}D(Q_{t+1}', Q_t).
\]
Next, we prove Eq.~\eqref{eqn:exp4_bound1}.
By Taylor expansion, there exists some convex combination of $Q_t$ and $Q_{t+1}$, denoted by $\xi_t$, such that
\begin{align*}
F_t'(Q_t) - F_t'(Q_{t+1}') &= \nabla F_t'(Q_{t+1}')(Q_t - Q_{t+1}') + \frac{1}{2}(Q_t - Q_{t+1}')^\top \nabla^{2} F_t'(\xi_t) (Q_t - Q_{t+1}') \\
&= \nabla F_t'(Q_{t+1}')(Q_t - Q_{t+1}') + \frac{1}{2}\sum_{\pi\in\Pi}\frac{(Q_t(\pi) - Q_{t+1}'(\pi))^2}{\xi_t(\pi)} \\
&\geq \frac{1}{2}\sum_{\pi\in\Pi}\frac{(Q_t(\pi) - Q_{t+1}'(\pi))^2}{\xi_t(\pi)},
\end{align*}
where the last step is due to the optimality of $Q_{t+1}'$.
On the other hand, we also have
\begin{align*}
F_t'(Q_t) - F_t'(Q_{t+1}') &= F_t(Q_t) - F_t(Q_{t+1}') + \eta\inner{Q_t - Q_{t+1}', \calL_t - \calM_t} \\
 &\leq \eta\inner{Q_t - Q_{t+1}', \calL_t - \calM_t} \tag{by optimality of $Q_t$}\\
 &\leq \eta\left(\sum_{\pi\in\Pi}\frac{(Q_t(\pi) - Q_{t+1}'(\pi))^2}{\xi_t(\pi)}\right)^{1/2}
   \left(\sum_{\pi\in\Pi}\xi_t(\pi)(\calL_t(\pi) - \calM_t(\pi))^2 \right)^{1/2}.  \tag{Cauchy-Schwarz inequality}
\end{align*}
Combining the two inequalities shows
\[
\inner{Q_t - Q_{t+1}', \calL_t - \calM_t} \leq 2\eta \sum_{\pi\in\Pi}\xi_t(\pi)(\calL_t(\pi) - \calM_t(\pi))^2,
\]
which proves Eq.~\eqref{eqn:exp4_bound1} (since $D(Q_{t+1}', Q_t)$ is non-negative).

To proves Eq.~\eqref{eqn:exp4_bound2}, note that $Q_{t+1}'(\pi) = \frac{1}{Z}Q_t(\pi)\exp(-\eta(\calL_t(\pi) - \calM_t(\pi)))$ where \[Z = \sum_{\pi\in\Pi} Q_t(\pi)\exp(-\eta(\calL_t(\pi) - \calM_t(\pi)))\] is the normalization factor. Direct calculation shows
\begin{align*}
&\inner{Q_t - Q_{t+1}', \calL_t - \calM_t} - \frac{1}{\eta}D(Q_{t+1}', Q_t) \\
&= \sum_{\pi\in\Pi} \inner{Q_t - Q_{t+1}', \calL_t - \calM_t} - \frac{1}{\eta} \sum_{\pi\in\Pi} Q_{t+1}'(\pi) \ln Q_{t+1}'(\pi) + \frac{1}{\eta} \sum_{\pi\in\Pi} Q_{t+1}'(\pi)\ln Q_t(\pi) \\
&=  \sum_{\pi\in\Pi} \inner{Q_t, \calL_t - \calM_t} + \frac{1}{\eta}\ln Z \\
&\leq \sum_{\pi\in\Pi} \inner{Q_t, \calL_t - \calM_t} + \frac{1}{\eta}\ln\sum_{\pi\in\Pi} Q_t(\pi)\left(1 - \eta(\calL_t(\pi) - \calM_t(\pi)) + \eta^2 (\calL_t(\pi) - \calM_t(\pi))^2\right)
\tag{by $e^{-z} \leq 1 - z + z^2$ for $z\geq -1$ and the condition $\eta(\calL_t(\pi) - \calM_t(\pi)) \geq -1$}\\
&= \sum_{\pi\in\Pi} \inner{Q_t, \calL_t - \calM_t} + \frac{1}{\eta}\ln\left(1 - \eta \inner{Q_t, \calL_t - \calM_t} + \eta^2\sum_{\pi \in \Pi} Q_t(\pi) (\calL_t(\pi) - \calM_t(\pi))^2\right) \\
&\leq \eta \sum_{\pi \in \Pi} Q_t(\pi)\left(\calL_t(\pi) - \calM_t(\pi)\right)^2. \tag{by $\ln (1+z)\leq z$}
\end{align*}
This finishes the proof.
\end{proof}

\begin{proof}[of Theorem~\ref{thm:OExp4}]
We directly apply Lemma~\ref{lem:key_observation} with $\calM_t(\pi) = m_t(\phi_t(\pi(x_t)))$ and $\calL_t(\pi) = \ellhat_t(\phi_t(\pi(x_t)))$ and use Eq.~\eqref{eqn:exp4_bound1} with $Q^*$ concentrating on the best policy $\pi^*$.
Summing over $t$ gives
\begin{align*}
&\sum_{t=1}^T\sum_{\pi \in\Pi} Q_t(\pi) \ellhat_t(\phi_t(\pi(x_t))) - \sum_{t=1}^T\ellhat_t(\phi_t(\pi^*(x_t))) \\
&\leq \frac{D(Q^*, Q_1')}{\eta}  + 2\eta \sum_{t=1}^T\sum_{a\in\calA_t}\sum_{\pi: \phi_t(\pi(x_t)) = a} \xi_t(\pi) \left(\ellhat_t(a) - m_t(a)\right)^2 . \\
&\leq \frac{\ln N}{\eta}  + 2\eta \sum_{t=1}^T\sum_{a\in\calA_t}\left(\ellhat_t(a) - m_t(a)\right)^2 \\
&= \frac{\ln N}{\eta}  + 2\eta \sum_{t=1}^T \left(\ellhat_t(a_t) - m_t(a_t)\right)^2,
\end{align*}
where in the last step we use the fact that $\ellhat_t(a) - m_t(a)$ is non-zero only if $a=a_t$.
Note that this basically proves Eq.~\eqref{eqn:key_observation} (with remapping).
The rest of the proof follows the analysis sketch in Section~\ref{sec:adversarial}.
First, we plug in the definition of $\ellhat_t$ and continue to bound the last expression by
\[
\frac{\ln N}{\eta}  + 2\eta \sum_{t=1}^T\frac{\left(\ell_t(a_t) - m_t(a_t)\right)^2}{p_t(a_t)^2}
\leq \frac{\ln N}{\eta}  + \frac{2\eta K}{\mu} \sum_{t=1}^T\frac{\left\|\ell_t - m_t\right\|_\infty^2}{p_t(a_t)},
\]
where the last step uses the fact $p_t(a_t) \geq \mu /|\calA_t| \geq \mu /K$.
Taking expectation on both sides leads to
\begin{equation}\label{eqn:intermediate}
\E\left[\sum_{t=1}^T\sum_{\pi \in\Pi} Q_t(\pi) \ell_t(\phi_t(\pi(x_t)))\right] - \sum_{t=1}^T\ell_t(\phi_t(\pi^*(x_t))) \leq \frac{\ln N}{\eta}  + \frac{2\eta K^2\calE}{\mu}.
\end{equation}
Next, consider the expected loss of the algorithm at time $t$:
\begin{align*}
\sum_{a\in \calA_t} p_t(a)\ell_t(a)
&=  (1-\mu) \sum_{a\in \calA_t}\left(\sum_{\pi: \phi_t(\pi(x_t)) = a} Q_t(\pi)\right) \ell_t(a)
+ \frac{\mu}{|\calA_t|}\sum_{a\in \calA_t} \ell_t(a) \\
&= (1-\mu) \sum_{\pi \in\Pi} Q_t(\pi) \ell_t(\phi_t(\pi(x_t))) + \frac{\mu}{|\calA_t|}\sum_{a\in \calA_t} \ell_t(a).
\end{align*}
Combining with Eq.~\eqref{eqn:intermediate} shows
\begin{align*}
\E\left[\sum_{t=1}^T\ell_t(a_t)\right]
&\leq (1-\mu)\sum_{t=1}^T\ell_t(\phi_t(\pi^*(x_t))) + \frac{\ln N}{\eta}  + \frac{2\eta K^2\calE}{\mu} + \sum_{t=1}^T \frac{\mu}{|\calA_t|}\sum_{a\in \calA_t}\ell_t(a) \\
&= \sum_{t=1}^T\ell_t(\phi_t(\pi^*(x_t))) + \frac{\ln N}{\eta}  + \frac{2\eta K^2\calE}{\mu} + \sum_{t=1}^T \frac{\mu}{|\calA_t|}\sum_{a\in \calA_t} \left(\ell_t(a) - \ell_t(\phi_t(\pi^*(x_t)))\right),
\end{align*}
where the last term can be further bounded as (by the definition of $\calA_t$):
\begin{align*}
&\ell_t(a) - \ell_t(\phi_t(\pi^*(x_t)) \\
&= \ell_t(a) - m_t(a) + m_t(a) - m_t(\phi_t(\pi^*(x_t)) + m_t(\phi_t(\pi^*(x_t)) - \ell_t(\phi_t(\pi^*(x_t)) \\
&\leq 2\|\ell_t - m_t\|_\infty + \sigma.
\end{align*}
This shows
\begin{align*}
\E\left[\sum_{t=1}^T\ell_t(a_t)\right]  &\leq
\sum_{t=1}^T\ell_t(\phi_t(\pi^*(x_t))) + \frac{\ln N}{\eta}  + \frac{2\eta K^2\calE}{\mu} + \mu T\sigma + 2\mu \sum_{t=1}^T  \|\ell_t - m_t\|_\infty \\
&\leq \sum_{t=1}^T\ell_t(\phi_t(\pi^*(x_t))) + \frac{\ln N}{\eta}  + \frac{2\eta K^2\calE}{\mu} + \mu T\sigma + 2\mu \sqrt{\calE T}. \tag{Cauchy-Schwarz inequality}
\end{align*}
It remains to bound the bias due to remapping: when $\pi^*(x_t) \neq \phi_t(\pi^*(x_t))$ we have $\phi_t(\pi^*(x_t)) = a_t^*$, $m_t(a_t^*) \leq m_t(\pi^*(x_t))  - \sigma$, and
\begin{align}
&\ell_t(\phi_t(\pi^*(x_t))) - \ell_t(\pi^*(x_t)) \notag \\
&= \ell_t(a_t^*)  - m_t(a_t^*) + m_t(a_t^*) - m_t(\pi^*(x_t)) + m_t(\pi^*(x_t)) - \ell_t(\pi^*(x_t)), \notag \\
&\leq 2\|\ell_t - m_t\|_\infty - \sigma \leq \frac{\|\ell_t - m_t\|_\infty^2}{\sigma},    \label{eqn:renaming_bias}
\end{align}
where the last step is by the AM-GM inequality.
When $\pi^*(x_t) = \phi_t(\pi^*(x_t))$, the above holds trivially.
Summing over $t$ we have thus shown
\[
\Reg \leq \frac{\ln N}{\eta}  + \frac{2\eta K^2\calE}{\mu} + \mu T\sigma + 2\mu\sqrt{\calE T} + \frac{\calE}{\sigma},
\]
finishing the proof.
\end{proof}

\subsection{Adaptive Version of Algorithm~\ref{alg:OExp4}}
\label{app:adaptive_OExp4}

The adaptive version of Algorithm~\ref{alg:OExp4} is shown in Algorithm~\ref{alg:OExp4-adaptive}. 
We observe that when $\calE$ is unknown, choosing actions only from $\calA_t$ is problematic, because in the case when the predictors are highly inaccurate (that is, large $\calE$), the environment can be such that the good actions are always outside $\calA_t$ but the learner can never realize that.
Based on this intuition, we remove the action remapping component in this case, implemented by simply setting $\sigma = 1$.

For the exploration parameter $\mu$, note that its optimal choice is independent of $\calE$ already in the known $\calE$ case (see Theorem~\ref{thm:OExp4}), which turns out to be also the case here (albeit with a different value).


Also note that standard Optimistic Online Mirror Descent analysis requires using the same learning rate in Lines~\ref{line: in algorithm update 1} and~\ref{line: in algorithm update 2} (see~\citep{wei2018more} for example).
However, using $\eta_{t}$ in both places is invalid since $a_t$ and $\ell_t(a_t)$ are unknown when executing Line~\ref{line:OMD1},
while using $\eta_{t-1}$ in both places also leads to some technical issue due to the large magnitude of loss estimators. 
Instead, we use $\eta_{t-1}$ in Line~\ref{line: in algorithm update 1}  and $\eta_t$ in Line~\ref{line: in algorithm update 2}, and carefully bound the bias introduced by this learning rate mismatch. Analyzing this learning rate mismatch is the key of our analysis, as we will show later. 

A minor but also necessary difference with Algorithm~\ref{alg:OExp4} is that we also enforce $Q_t$ and $Q'_t$ to be in the clipped simplex $\overline{\Delta}_\Pi = \left\{Q \in \Delta_\Pi: Q(\pi)\geq \frac{1}{NT}, \;\forall \pi\in \Pi \right\}$,
by writing the updates of $Q_t$ and $Q'_t$ in the Optimistic Online Mirror Descent form over $\overline{\Delta}_\Pi$. \\

\setcounter{AlgoLine}{0}
\begin{algorithm}[t]
    \caption{\expfourOVAR: Optimistic EXP4 with Variance-adaptivity and Action Remapping}
    \label{alg:OExp4-adaptive}
\textbf{Parameter}: exploration probability $\mu\in[0,1]$.

\textbf{Define}: $\overline{\Delta}_\Pi = \left\{Q \in \Delta_\Pi: Q(\pi)\geq \frac{1}{NT}, \;\forall \pi\in \Pi \right\}$ and $D(Q, P) = \sum_{\pi \in \Pi} Q(\pi)\ln\frac{Q(\pi)}{P(\pi)}$.

\textbf{Initialize}: $Q_1'(\pi) = \frac{1}{N}$ for all $\pi\in\Pi$ and $\eta_ 0 = \sqrt{\log(NT)}$.

\For{$t=1, \ldots, T$}{
\nl    Receive $x_t$ and $m_t$.

\nl    Calculate \label{line: in algorithm update 1}
\begin{align*}
     Q_t &= \argmin_{Q\in \overline{\Delta}_\Pi} \left\{ \eta_{t-1}\sum_{\pi\in\Pi} Q(\pi)m_t(\pi(x_t)) + D(Q, Q_{t}') \right\}.
\end{align*}

\nl    Calculate $p_t \in \Delta_K$: \quad
$
        p_t(a) =
        (1-\mu) \sum_{\pi: \pi(x_t)=a} Q_t(\pi) + \frac{\mu}{K}
$.

\nl    Sample $a_t\sim p_t$ and receive $\ell_t(a_t)$.

\nl    Construct estimator: \quad $\ellhat_t(a)= \frac{\ell_t(a)-m_t(a)}{p_t(a)}\one[a_t=a] + m_t(a)   $ for all  $a\in [K]$.

\nl    Calculate \label{line: in algorithm update 2}
\begin{align*}
       Q_{t+1}' = \argmin_{Q\in \overline{\Delta}_\Pi} \left\{ \eta_{t}\sum_{\pi\in\Pi} Q(\pi)\ellhat_t(\pi(x_t)) + D(Q, Q_{t}') \right\}
\end{align*}\\
\nl where \label{line: adaptive learning rate form}
\begin{align}
     \eta_t = \sqrt{\log (NT)}\left(1+\sum_{s=1}^t  \frac{(\ell_t(a_t)-m_t(a_t))^2}{p_t(a_t)^2}\right)^{-\frac{1}{2}}.  \label{eqn:eta_t}
\end{align}
}
\end{algorithm}

\begin{proof}[of Theorem~\ref{thm:adaptive_OExp4}]
Define $m_t' = \frac{\eta_{t-1}}{\eta_t}m_t$.
     Note that the update in Line~\ref{line: in algorithm update 1} and Line~\ref{line: in algorithm update 2} in Algorithm~\ref{alg:OExp4-adaptive} is the same as Eq.~\eqref{eqn: OOMD update} with $\eta=\eta_t$,
$\calM_t(\pi) = m_t'(\pi(x_t))$, and $\calL_t(\pi) = \ellhat_t(\pi(x_t))$, except that the constraint set becomes $\overline{\Delta}_\Pi$.
By the exact same arguments as the proof of Lemma~\ref{lem:key_observation}, we conclude that Eq.~\eqref{eqn:exp4_bound1} holds for any $Q^* \in \overline{\Delta}_\Pi$.
In particular, we pick $Q^* = \left(1-\frac{1}{T}\right)\mathbf{e}_{\pi^*} + \frac{1}{NT}\mathbf{1} \in  \overline{\Delta}_{\Pi}$, where $\mathbf{e}_{\pi^*}$ is the distribution that concentrates on $\pi^*$ and $\frac{1}{N}\mathbf{1}$ is the uniform distribution over $\Pi$.
With this $Q^*$, summing Eq.~\eqref{eqn:exp4_bound1} over $t$, we get
     \begin{align}
          &\sum_{t=1}^T \sum_{\pi\in\Pi} Q_t(\pi)\widehat{\ell}_t(\pi(x_t)) - \left(1-\frac{1}{T}\right)\sum_{t=1}^T \widehat{\ell}_t(\pi^*(x_t)) - \frac{1}{NT}\sum_{t=1}^T\sum_{\pi\in\Pi} \ellhat(\pi(x_t))  \nonumber  \\
          &\leq \sum_{t=1}^T \left(\frac{D(Q^*, Q_t') - D(Q^*, Q_{t+1}')}{\eta_t}\right) + 2\sum_{t=1}^T \eta_t \sum_{\pi\in\Pi} \xi_t(\pi) \left(\widehat{\ell}_t(\pi(x_t))-m_t'(\pi(x_t))\right)^2.  \label{eqn: adaptive adversarial regret tmp}
     \end{align}
     The first term on the right-hand side of Eq.~\eqref{eqn: adaptive adversarial regret tmp} is equal to
     \begin{align}
          \frac{D(Q^*, Q_1')}{\eta_1} + \sum_{t=2}^T D(Q^*, Q_{t}')\left(\frac{1}{\eta_t}-\frac{1}{\eta_{t-1}}\right) - \frac{D(Q^*, Q_{T+1}')}{\eta_T}. \label{eqn: adaptive learning rate bound 1}
     \end{align}
Note that for all $Q\in\overline{\Delta}_\Pi$, we have $D(Q^*, Q)=\sum_{\pi\in\Pi} Q^*(\pi)\log\frac{Q^*(\pi)}{Q(\pi)}\leq \sum_{\pi\in\Pi}Q^*(\pi)\log\frac{1}{1/(NT)}= \log(NT)$. Since $\frac{1}{\eta_t}\geq \frac{1}{\eta_{t-1}}$, we can thus upper bound Eq.~\eqref{eqn: adaptive learning rate bound 1} by
\begin{align*}
     \frac{\log(NT)}{\eta_1} + \sum_{t=2}^T \log(NT)\left(\frac{1}{\eta_t}-\frac{1}{\eta_{t-1}}\right) = \frac{\log(NT)}{\eta_T}.
\end{align*}
We continue to show that the second term on the right-hand side of Eq.~\eqref{eqn: adaptive adversarial regret tmp} is in fact also of order $\order\left(\frac{\log(NT)}{\eta_T}\right)$.
First, by direct calculation we have
\begin{align*}
    &\sum_{t=1}^T \eta_t \sum_{\pi\in\Pi} \xi_t(\pi)\left( \ellhat_t(\pi(x_t))-m_t'(\pi(x_t)) \right)^2 \\
    &= \sum_{t=1}^T \eta_t \sum_{\pi\in\Pi} \xi_t(\pi)\left( \frac{(\ell_t(a_t)-m_t(a_t)) \one[\pi_t(x_t)=a_t]}{p_t(a_t)} + m_t(\pi(x_t)) - \frac{\eta_{t-1}}{\eta_t}m_t(\pi(x_t)) \right)^2 \\
    &\leq \sum_{t=1}^T 2\eta_t \left(\frac{(\ell_t(a_t)-m_t(a_t))^2}{p_t(a_t)^2}\right) + \sum_{t=1}^T 2\eta_t\left(1-\frac{\eta_{t-1}}{\eta_t}\right)^2.
\end{align*}
To deal with the first term in the last expression, we define $b_t = \frac{(\ell_t(a_t)-m_t(a_t))^2}{p_t(a_t)^2}$ so that
\begin{align*}
     \sum_{t=1}^T \eta_t \left(\frac{(\ell_t(a_t)-m_t(a_t))^2}{p_t(a_t)^2}\right)
     &= (\log NT)^{\frac{1}{2}} \sum_{t=1}^T \frac{b_t}{\sqrt{1+\sum_{s=1}^t b_t}}
     \leq (\log NT)^{\frac{1}{2}} \int_{0}^{\sum_{s=1}^{T}b_t} \frac{dx}{\sqrt{1+x}} \\
     &= \order\left((\log NT)^{\frac{1}{2}} \left(1+\sum_{t=1}^T b_t\right)^{\frac{1}{2}}\right)
     =\order\left( \frac{\log(NT)}{\eta_T} \right).
\end{align*}
To deal with the second term, simply note that
\begin{align*}
    \sum_{t=1}^T \eta_t\left(1-\frac{\eta_{t-1}}{\eta_t}\right)^2
    = \sum_{t=1}^T \frac{1}{\eta_t} (\eta_t-\eta_{t-1})^2
    \leq \frac{(\log NT)^{\frac{1}{2}}}{\eta_T}\sum_{t=1}^T (\eta_{t-1}-\eta_{t})
    \leq \frac{\log (NT)}{\eta_T}.
\end{align*}
Combining everything above, we conclude that the right-hand side of Eq.~\eqref{eqn: adaptive adversarial regret tmp} is upper bounded by
\begin{align*}
    \order\left( \frac{\log(NT)}{\eta_T} \right) &= \order\left(\left( \log (NT) + \log(NT)\sum_{t=1}^T \frac{(\ell_t(a_t)-m_t(a_t))^2}{p_t(a_t)^2} \right)^{\frac{1}{2}}\right) \\
    &= \order\left(\left( \log (NT) + \frac{K\log(NT)}{\mu}\sum_{t=1}^T \frac{(\ell_t(a_t)-m_t(a_t))^2}{p_t(a_t)} \right)^{\frac{1}{2}}\right),
\end{align*}
whose expectation is upper bounded by (using Jensen's inequality)
\begin{align*}
     &\order\left(\left( \log (NT) + \frac{K\log(NT)}{\mu}\E\left[\sum_{t=1}^T \frac{(\ell_t(a_t)-m_t(a_t))^2}{p_t(a_t)}\right] \right)^{\frac{1}{2}}\right)  \\
     &=\order\left(\left(\log (NT)  + \frac{K^2\log (NT)\calE}{\mu}\right)^{\frac{1}{2}}\right)=\otil\left(d\sqrt{\frac{\calE}{\mu}}\right).
\end{align*}

Now we lower bound the expectation of the left-hand side of Eq.~\eqref{eqn: adaptive adversarial regret tmp}:
\begin{align*}
  &\E\left[\sum_{t=1}^T \sum_{\pi\in\Pi} Q_t(\pi)\widehat{\ell}_t(\pi(x_t)) - \left(1-\frac{1}{T}\right)\sum_{t=1}^T \widehat{\ell}_t(\pi^*(x_t)) - \frac{1}{NT}\sum_{t=1}^T\sum_{\pi\in\Pi} \ellhat(\pi(x_t)) \right]\\
  &\geq\E\left[\sum_{t=1}^T \sum_{\pi\in\Pi} Q_t(\pi)\ell_t(\pi(x_t)) - \sum_{t=1}^T \ell_t(\pi^*(x_t)) \right] -1 \\
  &=\E\left[\sum_{t=1}^T \sum_{a=1}^K \left( p_t(a) + \mu\sum_{\pi:\pi(x_t)=a} Q_t(\pi) - \frac{\mu}{K} \right) \ell_t(a) - \sum_{t=1}^T \ell_t(\pi^*(x_t)) \right] -1\\
  &\geq\E\left[\sum_{t=1}^T \sum_{a=1}^K p_t(a) \ell_t(a) - \sum_{t=1}^T \ell_t(\pi^*(x_t)) \right] -1-\mu T\\
  &=\E\left[\sum_{t=1}^T \ell_t(a_t) - \sum_{t=1}^T \ell_t(\pi^*(x_t)) \right] -1-\mu T.
\end{align*}
Combining the bounds for both sides of Eq.~\eqref{eqn: adaptive adversarial regret tmp} finishes the proof.
\end{proof}

\subsection{Algorithms and Analysis for Multiple Predictors}
\label{app:multi_pred_OExp4}

The pseudocode of our algorithm for multiple predictors is in Algorithm~\ref{alg:multi_pred_OExp4}.
As discussed in Section~\ref{sec:multiple},
there are several extra ingredients compared to Algorithm~\ref{alg:OExp4} in this case.
First, we maintain an active set $\calP_t$ of predictors that are still plausibly the best predictor: \[i^* = \argmin_{i\in[M]}\sum_{t=1}^T \|\ell_t - m_t^i\|_\infty^2.\]
Specifically, the variable $V^i$ maintains the remaining error ``budget'' for each predictor (starting from $\calE^*$), and is decreased by $(\ell_t(a_t) - m_t^i(a_t))^2$ at the end of each round.
Then $\calP_t$ is simply the set of predictors with a non-negative error budget.
Second, for each action $a$, we let
\[m_t(a) = \min_{i\in\calP_t} m_t^i(a),\]
 to be the smallest prediction among all active predictors, and treat $m_t$ as if it was the only prediction similar to the single predictor case, which can be seen as a form of optimism.

\begin{algorithm}[t]
\caption{\expfour.MOAR: Optimistic \expfour with Action Remapping for Multiple predictors}
\label{alg:multi_pred_OExp4}
\textbf{Parameters:} learning rate $\eta>0$, threshold $\sigma>0$, exploration probability $\mu\in[0,\tfrac{1}{2}]$, best error $\calE^*$.

\textbf{Initialize:} $Q_1'(\pi)=\frac{1}{N}$ for all $\pi\in\Pi$,
budget $V^i = \calE^*$ for all $i\in[M]$,
active set $\calP_1 = [M]$.

\For{$t=1, \ldots, T$}{
        Receive $x_t$ and $m_t^i$ for all $i\in[M]$.
        Let $m_t(a) = \min_{i\in\calP_t} m_t^i(a)$ for all $a\in[K]$.

        \ \\
        \textbf{Step 1. Jointly decide the awake action set $\calA_t$ and the action distribution $p_t$.}\\
        Let $b_1, b_2, \ldots, b_K$ be a permutation of $[K]$ such that
        \[
            m_t(b_1)\leq m_t(b_2) \leq \ldots\leq m_t(b_K).
        \]

        \For{$j=1,2,\ldots, K$}{
            Set $\calA = \{b_1, \ldots, b_j\}$.


            Calculate $p \in \Delta_K$:
            \[
                           p(a) = \begin{cases}
                   (1-\mu)\frac{\sum_{\pi: \pi(x_t)= a} Q_t'(\pi)\exp\big(-\eta m_t(\pi(x_t)) \big)}{\sum_{\pi: \pi(x_t) \in \calA} Q_t'(\pi)\exp\big(-\eta m_t(\pi(x_t)) \big)}+\frac{\mu}{|\calA|}, &\text{for\ } a\in \calA, \\
                   0, &\text{for\ } a \notin \calA.
           \end{cases} 
           \]

        \If{$j = K$ \;{\rm or}\; $m_t(b_j) \leq \inner{p, m_t}+\sigma \leq m_t(b_{j+1})$}{
        $\calA_t = \calA,  \ p_t = p$, \textbf{break}.}

        }

        \ \\
        \textbf{Step 2. Choose an action and construct loss estimators.}\\
        Sample $a_t\sim p_t$ and receive $\ell_t(a_t)$. \\
        Construct estimator: \quad
        \[
            \widehat{\ell}_t(a)=
            \begin{cases}
                \frac{(\ell_t(a)-m_t(a)) \one[a_t=a] }{p_t(a)} + m_t(a),  &\text{for\ } a\in \calA_t, \\
                \sum_{a\in \calA_t} p_t(a)\ellhat_t(a),   &\text{for\ } a \notin \calA_t.
            \end{cases}
        \]

        \ \\
        \textbf{Step 3. Make updates.}\\
        Calculate $Q_{t+1}'\in\Delta_{\Pi}$: \quad $Q_{t+1}'(\pi)  \propto Q_t'(\pi)\exp\left(-\eta \widehat{\ell}_t(\pi(x_t))\right)$.

        \For{$i\in \calP_t$}{
            ${V}^i \leftarrow {V}^{i} - (\ell_t(a_t)-m_t^{i}(a_t))^2$.
        }
        Update active set $\calP_{t+1} = \left\{i\in \calP_t: {V}^i\geq 0\right\}$.
}
\end{algorithm}

Finally and perhaps most importantly, we construct the set $\calA_t$ using a different baseline.
Essentially, the baseline is $a_t$, the action to be chosen by the algorithm, which is of course not available before constructing $\calA_t$.
However, instead of using $m_t(a_t)$ as the baseline in the definition of $\calA_t$, we use the expected prediction $\inner{p_t, m_t}$, and instead of explicitly remapping an action $a \notin \calA_t$ to be $a_t$, we change the values of $\ellhat_t(a)$ and $m_t(a)$ to $\langle p_t, \ellhat_t \rangle$ and $\inner{p_t, m_t}$ respectively for these actions.
While this is still a self-referential scheme since the construction of $p_t$ depends on $\calA_t$, we show that this can in fact be implemented efficiently by trying all the $K$ possibilities for $\calA_t$: $\{b_1\}, \{b_1, b_2\}, \ldots, \{b_1, b_2, \ldots, b_K\}$, where $b_1, \ldots, b_K$ are such that $m_t(b_1)\leq m_t(b_2) \leq \ldots\leq m_t(b_K)$.
The concrete procedure is detailed in Step 1 of Algorithm~\ref{alg:multi_pred_OExp4},
and we prove in the following lemma that it does exactly what we want.

\begin{lemma}\label{lem:properties}
Define $\calM_t, \calL_t \in \R^K$, $Q_t \in \Delta_\Pi$, and $q_t \in \Delta_{\calA_t}$ as
\[
\calM_t(a) = \begin{cases}
m_t(\pi(x_t)), &\text{if $\pi(x_t) \in \calA_t$,} \\
\inner{p_t, m_t}, &\text{otherwise}
\end{cases},
\quad\text{and}\quad
\calL_t(a) = \begin{cases}
\ellhat_t(\pi(x_t)), &\text{if $\pi(x_t) \in \calA_t$,} \\
\inner{p_t, \ellhat_t}, &\text{otherwise}
\end{cases},
\]
\[
Q_t(\pi) \propto Q_t'(\pi) \exp\left(-\eta \calM_t(\pi)\right),\quad\mbox{and}
\]
\[
q_t(a) = \frac{\sum_{\pi : \pi(x_t) = a} Q_t(\pi)}{\sum_{\pi: \pi(x_t) \in \calA_t} Q_t(\pi)}.
\]
Then Algorithm~\ref{alg:multi_pred_OExp4} ensures  the following properties:
\begin{align}
Q_{t+1}'(\pi) &\propto Q_t'(\pi) \exp\left(-\eta \calL_t(\pi)\right), \\
p_t(a) & =
\begin{cases}
(1 - \mu) q_t(a) + \frac{\mu}{|\calA_t|}, &\text{if $a\in \calA_t$}, \\
0, &\text{else},
\end{cases} \label{eqn:property_p_t}  \\
\calA_t &= \{a \in [K]: m_t(a) \leq \inner{p_t, m_t} + \sigma\}. \label{eqn:property_A_t}
\end{align}
\end{lemma}

\begin{proof}
The first property on $Q_{t+1}'$ is simply by the definition of $\calL_t$ and $\inner{p_t, \ellhat_t} = \sum_{a\in\calA_t} p_t(a)\ellhat_t(a)$.
The second equation is also clear by the definition of $Q_t$:
\[
q_t(a) 
= \frac{\sum_{\pi : \pi(x_t) = a} Q_t(\pi) }{\sum_{\pi: \pi(x_t) \in \calA_t} Q_t(\pi)}
= \frac{\sum_{\pi: \pi(x_t)= a} Q_t'(\pi)\exp\big(-\eta m_t(\pi(x_t)) \big)}{\sum_{\pi: \pi(x_t) \in \calA} Q_t'(\pi)\exp\big(-\eta m_t(\pi(x_t)) \big)}.
\]
The last equation clearly holds when $j < K$ and the condition $m_t(b_j) \leq \inner{p, m_t}+\sigma \leq m_t(b_{j+1})$ holds and triggers the ``break'' statement,
so it remains to prove Eq.~\eqref{eqn:property_A_t} if the ``break'' statement is triggered in the last iteration when $j=K$, in which case we have for all $j< K$,
\begin{equation}\label{eqn:condition}
\inner{p^j, m_t}+\sigma < m_t(b_j) \quad\text{or}\quad  \inner{p^j, m_t}+\sigma > m_t(b_{j+1})
\end{equation}
where $p^j$ is the value of $p$ in the $j$-th iteration.


Note that for all $k \leq j$, we have $p^j(b_k) \geq p^{j+1}(b_k)$ by the definition of $p$,
and also $p^{j+1}(b_{j+1}) = \sum_{k\leq j} (p^j(b_k) - p^{j+1}(b_k))$.
With these facts we prove $\inner{p^{j+1}, m_t} \geq \inner{p^{j}, m_t}$ below:
\begin{align*}
&\inner{p^{j+1}, m_t} \\
&= p^{j+1}(b_{j+1})m_t(b_{j+1})  + \sum_{k\leq j} p^{j+1}(b_k)m_t(b_k) \\
&= \sum_{k\leq j} \left(p^j(b_k) - p^{j+1}(b_k)\right)m_t(b_{j+1}) + p^{j+1}(b_k)m_t(b_k) \\
&\geq \sum_{k\leq j} \left(p^j(b_k) - p^{j+1}(b_k)\right)m_t(b_k) + p^{j+1}(b_k)m_t(b_k) \tag{$m_t(b_{j+1})  \geq m_t(b_k), \;\forall k \leq j$}\\
&= \inner{p^j, m_t}.
\end{align*}

Therefore,  realizing $\inner{p^1, m_t} + \sigma = m_t(b_1) + \sigma > m_t(b_1)$ and thus $\inner{p^1, m_t} + \sigma >  m_t(b_2)$ by Eq.~\eqref{eqn:condition},
we have
\[
\inner{p^2, m_t} + \sigma \geq \inner{p^1, m_t} + \sigma >  m_t(b_2),
\]
which in turn further implies (by repeatedly using Eq.~\eqref{eqn:condition} and $\inner{p^{j+1}, m_t} \geq \inner{p^{j}, m_t}$)
\[
\inner{p^3, m_t} + \sigma \geq \inner{p^2, m_t} + \sigma >  m_t(b_3)
\]
\[
\cdots,
\]
\[
\inner{p^K, m_t}+\sigma \geq \inner{p^{K-1}, m_t}+\sigma > m_t(b_{K}).
\]
The last statement proves Eq.~\eqref{eqn:property_A_t} again.
\end{proof}

With this fact, the analysis of the algorithm follows similar steps as in the proof of Theorem~\ref{thm:OExp4}.
First, we apply Lemma~\ref{lem:key_observation} to prove the following.

\begin{lemma}\label{lem:multi_pred_step1}
Algorithm~\ref{alg:multi_pred_OExp4} ensures for any $\pi^* \in \Pi$,
\begin{align*}
\E\left[\sum_{t=1}^T \sum_{\pi \in \Pi} Q_t(\pi) \ellhat_t(\pi(x_t))\right]
\leq \E\left[\sum_{t=1}^T \ellhat_t(\pi^*(x_t))\right] + \order\left(\frac{\ln N}{\eta}  +  \frac{\eta K^2\calE^*}{\mu} +\frac{\eta K M\calE^*}{\mu} \right).
\end{align*}
\end{lemma}

\begin{proof}
We apply Lemma~\ref{lem:key_observation} with $\calM_t$ and $\calL_t$ defined in Lemma~\ref{lem:properties}.
First note that
\[
\calL_t(\pi) - \calM_t(\pi) = \begin{cases}
\ell_t(a_t) - m_t(a_t), &\text{if $\pi(x_t) \notin \calA_t$,} \\
\frac{\ell_t(a_t) - m_t(a_t)}{p_t(a_t)} &\text{if $\pi(x_t) = a_t$,}\\
0, &\text{if $a_t \neq \pi(x_t) \in \calA_t$.}
\end{cases}
\]
Therefore, when $\ell_t(a_t) \geq m_t(a_t)$, the condition $\calL(\pi) - \calM_t(\pi) \geq -1/\eta$ holds and we apply Eq.~\eqref{eqn:exp4_bound2} and bound the last term by
\begin{align*}
& \eta \sum_{\pi \in \Pi} Q_t(\pi)\left(\calL_t(\pi) - \calM_t(\pi)\right)^2 \\
&= \eta \left(\sum_{\pi: \pi(x_t) = a_t} Q_t(\pi) \frac{\left(\ell_t(a_t) - m_t(a_t)\right)^2}{p_t^2(a_t)}
+ \sum_{\pi: \pi(x_t) \notin \calA_t} Q_t(\pi) \left(\ell_t(a_t) - m_t(a_t)\right)^2 \right)\\
&\leq \eta \left(q_t(a_t) \frac{\left(\ell_t(a_t) - m_t(a_t)\right)^2}{p_t^2(a_t)}
+ \left(\ell_t(a_t) - m_t(a_t)\right)^2 \right)\\
&\leq \eta \left(\frac{\left(\ell_t(a_t) - m_t(a_t)\right)^2}{(1-\mu)p_t(a_t)}
+ \left(\ell_t(a_t) - m_t(a_t)\right)^2 \right) \tag{by Eq.~\eqref{eqn:property_p_t}} \\
&\leq \eta \cdot \order\left(\frac{\left(\ell_t(a_t) - m_t(a_t)\right)^2}{p_t(a_t)}\right)  \tag{$\mu \leq 1/2$} \\
&\leq \frac{\eta K}{\mu} \cdot \order\left(\left(\ell_t(a_t) - m_t(a_t)\right)^2\right)  \tag{$p_t(a_t) \geq \mu/K$}
\end{align*}
On the other hand, if $\ell_t(a_t) \leq m_t(a_t)$, we apply Eq.~\eqref{eqn:exp4_bound1} and bound the last term by
\begin{align*}
& 2\eta \sum_{\pi \in \Pi} \xi_t(\pi)\left(\calL_t(\pi) - \calM_t(\pi)\right)^2 \\
&= 2\eta \left(\sum_{\pi: \pi(x_t) = a_t} \xi_t(\pi) \frac{\left(\ell_t(a_t) - m_t(a_t)\right)^2}{p_t^2(a_t)}
+ \sum_{\pi: \pi(x_t) \notin \calA_t} \xi_t(\pi) \left(\ell_t(a_t) - m_t(a_t)\right)^2 \right)\\
&\leq 2\eta \left(\frac{\left(\ell_t(a_t) - m_t(a_t)\right)^2}{p_t^2(a_t)}
+ \left(\ell_t(a_t) - m_t(a_t)\right)^2 \right)\\
&\leq \frac{\eta K}{\mu} \cdot \order\left(\frac{\left(\ell_t(a_t) - m_t(a_t)\right)^2}{p_t(a_t)}\right)   \tag{$p_t(a_t) \geq \mu /K$}, \\
&\leq \frac{\eta K}{\mu} \cdot \order\left(\frac{\left(\ell_t(a_t) - m_t^{i^*}(a_t)\right)^2}{p_t(a_t)}\right)   \tag{$\ell_t(a_t) \leq m_t(a_t) \leq m_t^{i^*}(a_t)$}.
\end{align*}
Combining the two situations, setting $Q^*$ to concentrate on $\pi^*$, and summing over $t$ show:
\begin{align*}
&\sum_{t=1}^T \sum_{\pi \in \Pi} Q_t(\pi) \ellhat_t(\pi(x_t))
\leq \sum_{t=1}^T \ellhat_t(\pi^*(x_t)) + \frac{\ln N}{\eta} \\
&\quad\quad + \frac{\eta K}{\mu} \cdot \order\left(\sum_{t=1}^T \frac{\left(\ell_t(a_t) - m_t^{i^*}(a_t)\right)^2}{p_t(a_t)}
+ \left(\ell_t(a_t) - m_t(a_t)\right)^2 \right).
\end{align*}
Taking expectation on both sides we have
\begin{align*}
&\E\left[\sum_{t=1}^T \sum_{\pi \in \Pi} Q_t(\pi) \ellhat_t(\pi(x_t))\right]
\leq \E\left[\sum_{t=1}^T \ellhat_t(\pi^*(x_t))\right] + \frac{\ln N}{\eta} \\
&\quad\quad + \frac{\eta K}{\mu} \cdot \order\left(\E\left[\sum_{t=1}^T \sum_{a\in[K]} \left(\ell_t(a) - m_t^{i^*}(a)\right)^2 + \left(\ell_t(a_t) - m_t(a_t)\right)^2 \right]\right) \\
&\leq \E\left[\sum_{t=1}^T \ellhat_t(\pi^*(x_t))\right] + \frac{\ln N}{\eta}  + \order\left(\frac{\eta K^2\calE^*}{\mu} +\frac{\eta K M\calE^*}{\mu} \right),
\end{align*}
where in the last step we use Lemma~\ref{lem:a_t_difference}.
This finishes the proof.
\end{proof}

Next, we relate the term $\E[\sum_{t=1}^T \sum_{\pi \in \Pi} Q_t(\pi) \ellhat_t(\pi(x_t)]$ to the loss of the algorithm, and the term $\E[\sum_{t=1}^T \ellhat_t(\pi^*(x_t))]$ to the loss of the best policy, in the following two lemmas respectively.
\begin{lemma}\label{lem:multi_pred_step2}
Algorithm~\ref{alg:multi_pred_OExp4} ensures
\[
\E\left[\sum_{t=1}^T \ell_t(a_t) \right]
\leq \E\left[\sum_{t=1}^T \sum_{\pi \in \Pi} Q_t(\pi) \ellhat_t(\pi(x_t))\right] +
\order\left(\sqrt{\mu M\calE^*T} + \mu T\sigma + \mu^2 T \right).
\]
\end{lemma}

\begin{proof}
With $Z_t = \sum_{\pi: \pi(x_t) \in \calA_t} Q_t(\pi)$ so that $Z_t q_t(a) =  \sum_{\pi: \pi(x_t) =a} Q_t(\pi)$, we rewrite the expected loss of the algorithm as
\begin{align*}
&\E\left[\ell_t(a_t) \right] \\
&= \E\left[\sum_{a\in\calA_t} p_t(a) \ell_t(a) \right] \\
&= \E\left[(1-\mu)\sum_{a\in\calA_t} q_t(a) \ell_t(a) + \frac{\mu}{|\calA_t|}\sum_{a\in\calA_t}  \ell_t(a) \right] \tag{by Eq.~\eqref{eqn:property_p_t}}\\
&= \E\left[(1-\mu)\sum_{a\in\calA_t} \left(Z_t q_t(a) + \left(1-Z_t \right) q_t(a)\right) \ell_t(a) + \frac{\mu}{|\calA_t|}\sum_{a\in\calA_t}  \ell_t(a) \right]  \\
&= \E\left[(1-\mu)Z_t\sum_{a \in \calA_t} q_t(a) \ell_t(a) + \left(1-Z_t \right)\sum_{a \in \calA_t}\left(p_t(a) - \frac{\mu}{|\calA_t|}\right)\ell_t(a)  + \frac{\mu}{|\calA_t|}\sum_{a\in\calA_t}  \ell_t(a) \right] \tag{by Eq.~\eqref{eqn:property_p_t}}\\
&= \E\left[Z_t\sum_{a \in \calA_t} q_t(a) \ell_t(a)  + (1-Z_t)\inner{p_t, \ell_t}
+ \mu Z_t \sum_{a\in \calA_t} \left( \frac{\ell_t(a)}{|\calA_t|} - q_t(a)\ell_t(a) \right)\right] \\
&= \E\left[\sum_{\pi \in \Pi} Q_t(\pi) \ellhat_t(\pi(x_t))
+ \mu Z_t \sum_{a\in \calA_t} \left( \frac{\ell_t(a)}{|\calA_t|} - q_t(a)\ell_t(a) \right)\right],
\end{align*}

where in the last step we use the fact
\begin{align*}
\sum_{\pi \in \Pi} Q_t(\pi) \ellhat_t(\pi(x_t)) &= \sum_{a\in\calA_t}\sum_{\pi: \pi(x_t) = a} Q_t(\pi) \ellhat_t(a)  + \sum_{\pi: \pi(x_t)\notin \calA_t} Q_t(\pi) \inner{p_t, \ellhat_t}  \\
&= Z_t\sum_{a \in \calA_t} q_t(a) \ellhat_t(a)  + (1-Z_t)\inner{p_t, \ellhat_t}
\end{align*}
by the definition of $\ellhat_t$.
It thus remains to bound $\E\left[\sum_{t=1}^T \mu Z_t \sum_{a\in \calA_t} \left( \frac{\ell_t(a)}{|\calA_t|} - q_t(a)\ell_t(a) \right)\right]$, which we decompose into four terms:
\[
\E\left[\sum_{t=1}^T \frac{\mu Z_t }{|\calA_t|}\sum_{a\in \calA_t} \left( \ell_t(a) - m_t(a) \right) \right],
\]
\[
\E\left[\sum_{t=1}^T \mu Z_t \sum_{a\in \calA_t} \left( \frac{m_t(a)}{|\calA_t|} - q_t(a)m_t(a) \right)\right],
\]
\[
\E\left[\sum_{t=1}^T \mu Z_t \sum_{a\in \calA_t} \left( q_t(a)m_t(a) - q_t(a)m_t^{i^*}(a) \right)\right],
\]
\[
\E\left[\sum_{t=1}^T \mu Z_t \sum_{a\in \calA_t} \left( q_t(a)m_t^{i^*}(a) - q_t(a)\ell_t(a) \right)\right].
\]
The first term can be bounded as
\begin{align*}
&\E\left[\sum_{t=1}^T \frac{\mu Z_t }{|\calA_t|} \sum_{a\in \calA_t}\left( \ell_t(a) - m_t(a) \right) \right]  \\
&\leq \E\left[\sum_{t=1}^T \frac{\mu}{|\calA_t|}\sum_{a\in \calA_t} |\ell_t(a) - m_t(a) | \right] \\
&\leq \E\left[\sqrt{\sum_{t=1}^T\sum_{a\in \calA_t} \frac{\mu}{|\calA_t|}}\sqrt{\sum_{t=1}^T\sum_{a\in \calA_t}  \frac{\mu}{|\calA_t|}(\ell_t(a) - m_t(a))^2} \right] \tag{Cauchy-Schwarz inequality}\\
&\leq \sqrt{\mu T}\cdot \E\left[\sqrt{\sum_{t=1}^T\sum_{a\in \calA_t} p_t(a)(\ell_t(a) - m_t(a))^2} \right] \tag{$p_t(a) \geq \mu/|\calA_t|$} \\
&\leq \sqrt{\mu T}\cdot \sqrt{\E\left[\sum_{t=1}^T (\ell_t(a_t) - m_t(a_t))^2 \right]}  \tag{Jensen's inequality} \\
&\leq \order\left(\sqrt{\mu M\calE^*T}\right). \tag{by Lemma~\ref{lem:a_t_difference}}
\end{align*}

The second term can be bounded as
\begin{align*}
&\E\left[\sum_{t=1}^T \mu Z_t \sum_{a\in \calA_t} \left( \frac{m_t(a)}{|\calA_t|} - q_t(a)m_t(a) \right)\right] \\
&\leq \E\left[\sum_{t=1}^T \mu Z_t \sum_{a\in \calA_t} \left( \frac{\inner{p_t, m_t}+\sigma}{|\calA_t|} - q_t(a)m_t(a) \right)\right] \tag{by Eq.~\eqref{eqn:property_A_t}} \\
&\leq \mu T\sigma + \E\left[\sum_{t=1}^T \mu Z_t \sum_{a\in \calA_t} \left( p_t(a)- q_t(a)\right)m_t(a) \right] \\
&= \mu T\sigma + \mu^2 \E\left[\sum_{t=1}^T Z_t \sum_{a\in \calA_t} \left(\frac{1}{|\calA_t|}- q_t(a)\right)m_t(a) \right] \tag{by Eq.~\eqref{eqn:property_p_t}} \\
&\leq \mu T\sigma + \mu^2 T.
\end{align*}

The third term is simply non-positive by the definition of $m_t$, and finally the four term can be bounded by (using Cauchy-Schwarz inequality again)
\begin{align*}
\E\left[\sum_{t=1}^T \mu Z_t \sum_{a\in \calA_t} \left( q_t(a)m_t^{i^*}(a) - q_t(a)\ell_t(a) \right)\right]
&\leq \mu \E\left[\sum_{t=1}^T \|\ell_t - m_t^{i^*}\|_\infty\right] \leq \mu \sqrt{\calE^*T},
\end{align*}
which can be absorbed by the bound of the first term since $\mu \leq 1$.
Combining all the bounds proves the lemma.
\end{proof}

\begin{lemma}\label{lem:multi_pred_step3}
Algorithm~\ref{alg:multi_pred_OExp4} ensures
\[
\E\left[\sum_{t=1}^T \ellhat_t(\pi^*(x_t)) \right]
\leq \sum_{t=1}^T \ell_t(\pi^*(x_t)) + \order\left(\frac{M\calE^*}{\sigma}\right).
\]
\end{lemma}

\begin{proof}
Note that
\[
\E\left[\ellhat_t(\pi^*(x_t)) \right]
= \ell_t(\pi^*(x_t)) + \E\left[ \one[\pi^*(x_t) \notin \calA_t] \left(\ell_t(a_t) - \ell_t(\pi^*(x_t)) \right) \right],
\]
where the second term is bounded as
\begin{align*}
&\E\left[ \one[\pi^*(x_t) \notin \calA_t] \left(\ell_t(a_t) - m_t(a_t) + m_t(a_t) - m_t(\pi^*(x_t)) + m_t^{i*}(\pi^*(x_t)) - \ell_t(\pi^*(x_t)) \right) \right] \tag{$m_t(a) \leq m_t^{i^*}(a)$} \\
&=\E\left[ \one[\pi^*(x_t) \notin \calA_t] \left(\ell_t(a_t) - m_t(a_t) + \inner{p_t, m_t} - m_t(\pi^*(x_t)) + m_t^{i*}(\pi^*(x_t)) - \ell_t(\pi^*(x_t)) \right) \right] \\
&\leq \E\left[|\ell_t(a_t) - m_t(a_t)| + \|\ell_t - m_t^{i*}\|_\infty - \sigma  \right] \tag{by Eq.~\eqref{eqn:property_A_t}} \\
&= \E\left[\left(|\ell_t(a_t) - m_t(a_t)| - \frac{\sigma}{2}\right)+ \left(\|\ell_t - m_t^{i*}\|_\infty - \frac{\sigma}{2}\right)  \right] \\
&\leq \E\left[\frac{(\ell_t(a_t) - m_t(a_t)|)^2}{2\sigma}  + \frac{\|\ell_t - m_t^{i*}\|_\infty^2}{2\sigma}  \right] \tag{AM-GM inequality}.
\end{align*}
Summing over $t$ and using the fact $\sum_{t=1}^T \|\ell_t - m_t^{i*}\|_\infty^2 = \calE^*$ and Lemma~\ref{lem:a_t_difference} complete the proof.
\end{proof}

In the proofs of all the three lemmas above, we have used the following fact:
\begin{lemma}\label{lem:a_t_difference}
Algorithm~\ref{alg:multi_pred_OExp4} ensures
$\sum_{t=1}^T (\ell_t(a_t) - m_t(a_t))^2 \leq M(\calE^*+1)$.
\end{lemma}
\begin{proof}
Let $\calT_i = \{t \in T: i \in \calP_t\}$ be the time steps where predictor $i$ is active.
Then
\begin{align*}
\sum_{t=1}^T (\ell_t(a_t) - m_t(a_t))^2
&\leq  \sum_{t=1}^T\sum_{i \in \calP_t} (\ell_t(a_t) - m_t^i(a_t))^2 \\
&= \sum_{i \in [M]} \sum_{t\in \calT_i} (\ell_t(a_t) - m_t^i(a_t))^2
\leq M(\calE^*+1),
\end{align*}
where the last step uses the fact $\sum_{t\in \calT_i} (\ell_t(a_t) - m_t^i(a_t))^2  \leq \calE^* + 1$ since the last term in the summation is bounded by one, while the rest cannot exceed $\calE^*$ because $i$ has not been removed from the active set yet.
\end{proof}

Finally, we are ready to prove Theorem~\ref{thm:multi_pred_OExp4}.

\begin{proof}[of Theorem~\ref{thm:multi_pred_OExp4}]
Combining Lemmas~\ref{lem:multi_pred_step1}, \ref{lem:multi_pred_step2}, and~\ref{lem:multi_pred_step3},
we have
\[
\Reg = \order\left(\frac{\ln N}{\eta}  +  \frac{\eta K^2\calE^*}{\mu} +\frac{\eta K M\calE^*}{\mu} + \sqrt{\mu M\calE^*T} + \mu T\sigma + \mu^2 T +   \frac{M\calE^*}{\sigma} \right).
\]
With $M' = \max\{K, M\}$, setting
\begin{align*}
\mu = \min\left\{\frac{1}{2}, \sqrt{\frac{d}{T}}\right\},
\sigma = \sqrt{\frac{M\calE^*}{\mu T}},
\eta = \sqrt{\frac{\mu \ln N}{KM' \calE^*}},
\end{align*}
gives $\Reg = \order\left(\sqrt{M'\calE^*}(dT)^\frac{1}{4} + \sqrt{dM'\calE^*} + d  \right)$.
\end{proof} 

\section{Omitted Details for Stochastic Environments}
\label{app:iid}

In this section, we provide omitted details for the stochastic case, including proofs for results with known $\calE$ and a single predictor (Section~\ref{app:eps_greedy}), the adaptive version of Algorithm~\ref{alg:eps_greedy} and its analysis when $\calE$ is unknown (Section~\ref{app:adaptive_eps_greedy}),
and the algorithm and analysis for multiple predictors (Section~\ref{app:multi_pred_iid}).

\subsection{Proofs of Lemma~\ref{lem:binary_search} and Theorems~\ref{thm:eps_greedy_weak} and~\ref{thm:eps_greedy}}
\label{app:eps_greedy}

First, we prove Lemma~\ref{lem:binary_search} which certifies the efficiency and (approximate) correctness of the binary search procedure for finding the policy with the smallest Catoni's mean (Algorithm~\ref{alg:binary_search}).

 \begin{proof}[of Lemma~\ref{lem:binary_search}]
 The fact that the algorithm stops after $\log_2 \left(2T\left(\frac{K}{\mu}+1\right)\right) = \order(\ln (KT/\mu))$ iterations is clear due to the initial value of $z_\text{left}$ and $z_\text{right}$, and the precision $1/T$.

 To prove the approximate optimality of the output $\pi_t$,
 note that the algorithm maintains the following loop invariants:
 \[
  \min_{\pi\in\Pi}\sum_{s<t} \psi\left(\alpha\left(\elltil_s(\phi_s(\pi(x_s))) - z_\text{left}\right)\right) \geq 0
 \]
 and
 \[
  \min_{\pi\in\Pi}\sum_{s<t} \psi\left(\alpha\left(\elltil_s(\phi_s(\pi(x_s))) - z_\text{right}\right)\right) \leq 0.
 \]
Therefore, by the monotonicity of $\psi$, all policies have Catoni's mean larger than $z_\text{left}$, and there exists a policy
\[
\argmin_{\pi\in\Pi}\sum_{s<t} \psi\left(\alpha\left(\elltil_s(\phi_s(\pi(x_s))) - z_\text{right}\right)\right)
\]
with Catoni's mean smaller than $z_\text{right}$.
These two facts imply that both $\catoni_\alpha\left(\left\{ \elltil_s(\phi_s(\pi_t(x_s)))\right\}_{s<t}\right)$
and  $\min_{\pi\in\Pi}\catoni_\alpha\left(\left\{ \elltil_s(\phi_s(\pi(x_s)))\right\}_{s<t}\right)$ are between $z_\text{left}$ and $z_\text{right}$,
and are thus $1/T$ away from each other since we have $z_\text{right} - z_\text{left} \leq 1/T$ after the algorithm stops.
 \end{proof}

To prove both Theorem~\ref{thm:eps_greedy_weak} and Theorem~\ref{thm:eps_greedy},
we introduce the following notation.

\begin{definition}
Denote by $\calL(\pi) \triangleq \E_{(x_t,m_t,\ell_t)\sim \calD}[\ell_t(\pi(x_t))]$ the expected loss of policy $\pi$, and by $\overline{\calL}(\pi)\triangleq \E_{(x_t,m_t,\ell_t)\sim \calD}[\ell_t(\phi_t(\pi(x_t)))]$ the expected loss of policy $\pi$ after remapping.
\end{definition}

For both theorems we make use of the following lemmas.

\begin{lemma}\label{lem:eps_greedy_alg_loss}
Algorithm~\ref{alg:eps_greedy} (with either Option I or Option II) ensures
\[
\E\left[\sum_{t=1}^T \ell_t(a_t)\right] \leq \E\left[\sum_{t=1}^T \overline{\calL}(\pi_t) \right] + \mu T\sigma + 2\mu \sqrt{\calE T}.
\]
\end{lemma}
\begin{proof}
Denote the conditional expectation given the history up to the beginning of time $t$ by $\E_t[\cdot]$.
By the choice of $a_t$ we have
\begin{align*}
    \mathbb{E}_t\left[\ell_t(a_t)\right]
    &= (1-\mu)\mathbb{E}_t\left[ \ell(\phi_t(\pi_t(x_t)) \right] + \mathbb{E}_t\left[\frac{\mu}{|\calA_t|}\sum_{a\in \calA_t} \ell_t(a) \right] \\
    &= \overline{\calL}(\pi_t) + \mathbb{E}_t\left[\frac{\mu}{|\calA_t|}\sum_{a\in \calA_t} (\ell_t(a) - \ell_t(\phi_t(\pi_t(x_t)))) \right] \\
    &\leq \overline{\calL}(\pi_t)+ \mu\mathbb{E}_t\left[\sup_{a, a'\in\calA_t} |\ell_t(a) - \ell_t(a')| \right] \\
    &\leq \overline{\calL}(\pi_t)+ \mu\mathbb{E}_t\left[\sup_{a, a'\in\calA_t} |\ell_t(a) - m_t(a)| + |m_t(a)-m_t(a')| + |m_t(a')-\ell_t(a')| \right] \\
    &\leq \overline{\calL}(\pi_t)+ \mu\mathbb{E}_t\left[ \sigma + 2\|\ell_t-m_t\|_\infty \right] \tag{by the definition of $\calA_t$}\\
    &= \overline{\calL}(\pi_t)+ \mu\sigma + 2\mu\mathbb{E}_t\left[\|\ell_t-m_t\|_\infty \right].
\end{align*}
Summing over $T$ and applying Cauchy-Schwarz inequality:
\[
\E\left[\sum_{t=1}^T \|\ell_t-m_t\|_\infty  \right] \leq \sqrt{T \E\left[\sum_{t=1}^T \|\ell_t-m_t\|_\infty^2  \right]} = \sqrt{\calE T}
\]
finish the proof.
\end{proof}

\begin{lemma}\label{lem:eps_greedy_bias}
Algorithm~\ref{alg:eps_greedy} (with either Option I or Option II) ensures
\[
T(\overline{\calL}(\pi^*)-\calL(\pi^*))  \leq \frac{\calE}{\sigma}.
\]
\end{lemma}
\begin{proof}
The proof is exactly the same as the adversarial case ({\it cf.} Eq.~\eqref{eqn:renaming_bias}).
First rewrite $\overline{\calL}(\pi^*)-\calL(\pi^*)$ as $\E\left[\ell_t(\phi_t(\pi^*(x_t))) - \ell_t(\pi^*(x_t))\right]$.
When $\pi^*(x_t) \neq \phi_t(\pi^*(x_t))$ we have $\phi_t(\pi^*(x_t)) = a_t^*$, $m_t(a_t^*) \leq m_t(\pi^*(x_t))  - \sigma$, and
\begin{align*}
&\ell_t(\phi_t(\pi^*(x_t))) - \ell_t(\pi^*(x_t)) \\
&= \ell_t(a_t^*)  - m_t(a_t^*) + m_t(a_t^*) - m_t(\pi^*(x_t)) + m_t(\pi^*(x_t)) - \ell_t(\pi^*(x_t)), \\
&\leq 2\|\ell_t - m_t\|_\infty - \sigma \leq \frac{\|\ell_t - m_t\|_\infty^2}{\sigma},
\end{align*}
where the last step is by the AM-GM inequality.
When $\pi^*(x_t) = \phi_t(\pi^*(x_t))$, the above holds trivially.
Plugging the definition of $\calE$ then finishes the proof.
\end{proof}

We are now ready to prove Theorems~\ref{thm:eps_greedy_weak} and~\ref{thm:eps_greedy}, using different concentrations according to the two different ways of calculating $\pi_t$. \\

\begin{proof}[of Theorem~\ref{thm:eps_greedy_weak}]
First, for any fix $\pi$ and $t$, we invoke Lemma~\ref{lemma:  freedman} with $X_s =  \elltil_s(\phi_s(\pi(x_s))) - \overline{\calL}(\pi) + \E_{(x,\ell,m)\sim\calD}[\min_a m(a)]$ for $s = 1, \ldots, t$, $b = \order(\frac{K}{\mu})$, and $V_t = \order(\frac{K\calE t}{\mu T} + \sigma^2 t)$ (see Eq.~\eqref{eqn: variance calculation}).
Together with a union bound over all $t$ and $\pi$, we have with probability at least $1-1/T$,
\begin{equation}\label{eqn:ell_tilde_concentration}
\begin{split}
\left|
\frac{1}{t}\sum_{s=1}^t \elltil_s(\phi_s(\pi(x_s))) - \overline{\calL}(\pi) + \E_{(x,\ell,m)\sim\calD}[\min_a m(a)]
\right| \\
= \order\left( \sqrt{\left(\frac{K\calE}{\mu T t} + \frac{\sigma^2}{t} \right) \log(NT)} + \frac{K\log(NT)}{\mu t} \right)
\end{split}
\end{equation}
for all $t \in [T]$ and $\pi \in \Pi$.
Therefore, we have
\begin{align*}
        &\overline{\calL}(\pi_t) \\
        &\leq \frac{1}{t}\sum_{s=1}^t \elltil_s(\phi_s(\pi_t(x_s))) + \E[\min_a m(a)]  + \order\left( \sqrt{\left(\frac{K\calE}{\mu T t} + \frac{\sigma^2}{t} \right) \log(NT)} + \frac{K\log(NT)}{\mu t} \right)   \tag{by Eq.~\eqref{eqn:ell_tilde_concentration}}  \\
        &\leq \frac{1}{t}\sum_{s=1}^t \elltil_s(\phi_s(\pi^*(x_s))) + \E[\min_a m(a)]  + \order\left( \sqrt{\left(\frac{K\calE}{\mu T t} + \frac{\sigma^2}{t} \right) \log(NT)} + \frac{K\log(NT)}{\mu t} \right)   \tag{by the optimality of $\pi_t$}  \\
        &\leq \overline{\calL}(\pi^*) + \order\left( \sqrt{\left(\frac{K\calE}{\mu T t} + \frac{\sigma^2}{t} \right) \log(NT)} + \frac{K\log(NT)}{\mu t} \right).  \tag{by Eq.~\eqref{eqn:ell_tilde_concentration}}
\end{align*}
Combining Lemma~\ref{lem:eps_greedy_alg_loss}, the inequality above, and Lemma~\ref{lem:eps_greedy_bias}, we arrive at
\begin{align}
\E\left[\sum_{t=1}^T \ell_t(a_t)\right] &\leq \E\left[\sum_{t=1}^T \overline{\calL}(\pi_t) \right] + \mu T\sigma + 2\mu \sqrt{\calE T} \notag\\
&\leq T\overline{\calL}(\pi^*) + \order\left(\mu T\sigma + \mu \sqrt{\calE T} + \sum_{t=1}^T \sqrt{\left(\frac{K\calE}{\mu T t} + \frac{\sigma^2}{t} \right) \log(NT)} + \frac{K\log(NT)}{\mu t} \right) \notag\\
&= T\overline{\calL}(\pi^*) + \otil\left(\mu T\sigma + \mu \sqrt{\calE T} + \sqrt{\frac{d\calE}{\mu}} + \sigma\sqrt{dT}+ \frac{d}{\mu} \right) \label{eqn:useful_for_adaptive_version}\\
&\leq \E\left[\sum_{t=1}^T \ell_t(\pi^*(x_t))\right] + \otil\left(\mu T\sigma + \mu \sqrt{\calE T} + \sqrt{\frac{d\calE}{\mu}} + \sigma\sqrt{dT}+ \frac{d}{\mu} + \frac{\calE}{\sigma}\right), \notag
\end{align}
which finishes the proof.
\end{proof}

\begin{proof}[of Theorem~\ref{thm:eps_greedy}]
First, for any fix $\pi$ and $t$, we invoke Lemma~\ref{lemma: catoni concentration} with $X_s =  \elltil_s(\phi_s(\pi(x_s)))$ for $s = 1, \ldots, t$, $\mu_1 = \cdots = \mu_t = \mu = \overline{\calL}(\pi) - \E_{(x,\ell,m)\sim\calD}[\min_a m(a)]$, and $V = \order(\frac{K\calE}{\mu} + \sigma^2 t)$ (see Eq.~\eqref{eqn: variance calculation} for the variance calculation).
Together with a union bound over all $t$ and $\pi$, and the value of $\alpha$ specified in Algorithm~\ref{alg:eps_greedy}, we have with probability at least $1-2/T$,
\begin{equation}\label{eqn:ell_tilde_concentration2}
\begin{split}
\left|
\catoni_\alpha\big(\big\{ \elltil_s(\phi_s(\pi(x_s)))\big\}_{s\leq t}\big) - \overline{\calL}(\pi) + \E_{(x,\ell,m)\sim\calD}[\min_a m(a)]
\right| \\
= \frac{1}{t}\left(\alpha V + \frac{2\log(NT^2)}{\alpha}\right)
= \order\left( \sqrt{\left(\frac{K\calE}{\mu t^2} + \frac{\sigma^2}{t} \right) \log(NT)} \right)
\end{split}
\end{equation}
for all $t \geq \alpha^2 V + 2\log(NT^2) = 4\log(NT^2)$ and $\pi \in \Pi$.
Therefore, we have for $t \geq 4\ln(NT^2)$,
\begin{align*}
        &\overline{\calL}(\pi_t) \\
        &\leq \catoni_\alpha\big(\big\{ \elltil_s(\phi_s(\pi_t(x_s)))\big\}_{s\leq t}\big) + \E[\min_a m(a)]  + \order\left( \sqrt{\left(\frac{K\calE}{\mu t^2} + \frac{\sigma^2}{t} \right) \log(NT)} \right)   \tag{by Eq.~\eqref{eqn:ell_tilde_concentration2}}  \\
        &\leq \catoni_\alpha\big(\big\{ \elltil_s(\phi_s(\pi^*(x_s)))\big\}_{s\leq t}\big) + \E[\min_a m(a)]  + \order\left( \sqrt{\left(\frac{K\calE}{\mu t^2} + \frac{\sigma^2}{t} \right) \log(NT)} + \frac{1}{T}\right)   \tag{by Lemma~\ref{lem:binary_search}}  \\
        &\leq \overline{\calL}(\pi^*) + \order\left( \sqrt{\left(\frac{K\calE}{\mu t^2} + \frac{\sigma^2}{t} \right) \log(NT)} + \frac{1}{T} \right).  \tag{by Eq.~\eqref{eqn:ell_tilde_concentration2}}
\end{align*}

Combining Lemma~\ref{lem:eps_greedy_alg_loss}, the inequality above, and Lemma~\ref{lem:eps_greedy_bias}, we arrive at
\begin{align*}
\E\left[\sum_{t=1}^T \ell_t(a_t)\right] &\leq \E\left[\sum_{t=1}^T \overline{\calL}(\pi_t) \right] + \mu T\sigma + 2\mu \sqrt{\calE T} \\
&\leq T\overline{\calL}(\pi^*) + \order\left(4\ln(NT^2) + \mu T\sigma + \mu \sqrt{\calE T} + \sum_{t=1}^T \sqrt{\left(\frac{K\calE}{\mu t^2} + \frac{\sigma^2}{t} \right) \log(NT)} \right) \\
&= T\overline{\calL}(\pi^*) + \otil\left(\mu T\sigma + \mu \sqrt{\calE T} + \sqrt{\frac{d\calE}{\mu}} + \sigma\sqrt{dT}  \right) \\
&\leq \E\left[\sum_{t=1}^T \ell_t(\pi^*(x_t))\right] + \otil\left(\mu T\sigma + \mu \sqrt{\calE T} + \sqrt{\frac{d\calE}{\mu}} + \sigma\sqrt{dT}+ \frac{\calE}{\sigma}\right),
\end{align*}
which finishes the proof.
\end{proof}

\subsection{Adaptive Version of Algorithm~\ref{alg:eps_greedy}}
\label{app:adaptive_eps_greedy}

The pseudocode of the adaptive version of Algorithm~\ref{alg:eps_greedy} in shown in Algorithm~\ref{alg:explore then epsilon greedy}.
To prove its regret guarantee, we make use of the following useful lemmas.
The first one shows the concentration of $\alphahat_i$ around $\alpha_i = \frac{1}{K}\sum_{a=1}^K \Pr\left[ |\ell_t(a)-m_t(a)|\in (2^{-i-1}, 2^{-i}] \right]$.

\begin{lemma}
    \label{lemma: coeff lower bound}
    Algorithm~\ref{alg:explore then epsilon greedy} ensures:
    \begin{itemize}
    \item If $\alpha_i > \frac{360\log T}{B}$, then with probability at least $1-1/T$,
    \begin{align}
        \left[ \alphahat_i - \frac{30\log T}{B} \right]_+ \geq \frac{1}{3} \alpha_i;  \label{eqn: lemma 21 1}
    \end{align}
    \item
    With probability $1-1/T$,
    \begin{align}
        \left[ \alphahat_i - \frac{30\log T}{B} \right]_+ \leq \frac{3}{2}\alpha_i.  \label{eqn: lemma 21 2}
    \end{align}
    \end{itemize}
\end{lemma}
\begin{proof}
    Clearly, $\E[\alphahat_i]=\alpha_i$. By Freedman's inequality (Lemma~\ref{lemma: freedman}), with probability $1-\frac{1}{T}$,
    \begin{align*}
        |\alphahat_i-\alpha_i|
        &\leq 2\sqrt{\frac{\alpha_i \log T}{B}} + \frac{\log T}{B}\\
        &\leq \frac{\alpha_i}{2} + \frac{30\log T}{B},  \tag{AM-GM inequality}
    \end{align*}
    implying both $\widehat{\alpha}_i \leq \frac{3}{2}\alpha_i + \frac{30 \log T}{B}$ and  $\frac{\alpha_i}{2}\leq \widehat{\alpha}_i + \frac{30\log T}{B}$. The former implies Eq.~\eqref{eqn: lemma 21 2}.
    Rearranging the latter gives $ \alphahat_i - \frac{30\log T}{B} \geq \frac{\alpha_i}{2} - \frac{60\log T}{B}$. If $\alpha_i > \frac{360\log T}{B}$, then $\frac{\alpha_i}{2} - \frac{60\log T}{B}$  can further be lower bounded by $\frac{\alpha_i}{2} - \frac{\alpha_i}{6} = \frac{\alpha_i}{3}$, thus proving Eq.~\eqref{eqn: lemma 21 1}.
\end{proof}

\begin{algorithm}[t]
    \caption{\greedyVAR: $\epsilon$-Greedy with Variance-adaptivity and Action Remapping}
    \label{alg:explore then epsilon greedy}
    \For{$t=1, \ldots, B$}{
        Draw $a_t\sim \text{Uniform}([K])$.
    }
    Let
    \begin{align*}
        \alphahat_i &= \frac{1}{B}\sum_{t=1}^B \one\left[|\ell_t(a_t)-m_t(a_t)|\in (2^{-i-1}, 2^{-i}]\right],  \\
        \calEhat&=T\sum_{i=0}^{\lceil \log_2 T \rceil} \left[ \alphahat_i - \frac{30\log T}{B} \right]_+ 2^{-2i}.
    \end{align*}
    Run Algorithm~\ref{alg:eps_greedy} for the remaining rounds with Option I,  $\sigma =  \sqrt{\calEhat}(dT)^{-\frac{1}{3}}$, and $\mu=\min\big\{d^{\frac{2}{3}}/T^{\frac{1}{3}},1\big\}$.
\end{algorithm}

The next lemma shows that $\calEhat$ is essentially an underestimator of $\calE$.
\begin{lemma}
    \label{lem: V' < V}
    With probability $1-\frac{1}{T}$, $\calEhat\leq 6\calE$.
\end{lemma}
\begin{proof}
    By Lemma~\ref{lemma: coeff lower bound} and the definition of $\calEhat$, with probability $1-\frac{1}{T}$ we have
    \begin{align*}
        \calEhat &\leq T\sum_{i=0}^{\lceil \log_2 T \rceil}\frac{3}{2}\alpha_i 2^{-2i} = 6 T\sum_{i=0}^{\lceil \log_2 T \rceil}\alpha_i 2^{-2i-2} \\
        &= 6T\sum_{i=0}^{\lceil \log_2 T \rceil}  \E\left[\frac{1}{K}\sum_{a=1}^K \one\left[|\ell_t(a)-m_t(a)| \in (2^{-i-1}, 2^{-i}]\right]\right]2^{-2i-2} \\
        &= 6T\E\left[\frac{1}{K}\sum_{a=1}^K \sum_{i=0}^{\lceil \log_2 T \rceil} \one\left[|\ell_t(a)-m_t(a)| \in (2^{-i-1}, 2^{-i}]\right]2^{-2i-2} \right] \\
        &\leq 6T\E\left[\frac{1}{K}\sum_{a=1}^K (\ell_t(a)-m_t(a))^2\right] \\
        &\leq 6T \E\left[\|\ell_t-m_t\|_\infty^2\right]=6\calE.
    \end{align*}
\end{proof}

The final lemma analyzes the bias due to remapping with the new value of $\sigma$, which replaces the role of Lemma~\ref{lem:eps_greedy_bias} when analyzing Algorithm~\ref{alg:explore then epsilon greedy}.

\begin{lemma}
\label{lemma: adaptive algorithm improved error term}
Algorithm~\ref{alg:explore then epsilon greedy} ensures:
    \label{lem: explore then epsilon error term}
    \begin{align*}
       (T-B)(\overline{\calL}(\pi^*)-\calL(\pi^*))
       = \otil\left(K^2\sqrt{\calE}(dT)^{\frac{1}{3}} + K\sqrt{\frac{\calE T}{B}}\right).
    \end{align*}
\end{lemma}

\begin{proof}
First we bound $(T-B)(\overline{\calL}(\pi^*)-\calL(\pi^*))$ by
$\E\left[ \sum_{t=B+1}^T \left(2\|\ell_t-m_t\|_\infty - \sigma\right) \right]$, following the exact same argument as in the proof of Lemma~\ref{lem:eps_greedy_bias}.
We then further bound the latter by
    \begin{align*}
        &\E\left[ \sum_{t=B+1}^T \left(2\|\ell_t-m_t\|_1 - \sigma\right) \right] \\
        &\leq \E\left[ \sum_{t=B+1}^T \left(2\sum_{a=1}^K \left(\sum_{i=0}^{\lceil \log_2 T\rceil} 2^{-i}\one[|\ell_t(a)-m_t(a)| \in (2^{-i-1}, 2^{-i}]] + \order\left(\frac{1}{T}\right) \right)  - \sigma \right)\right]          \tag{the $\order\left(\frac{1}{T}\right)$ term incurs when all indicators are zero} \\
        &= \E\left[ \sum_{t=B+1}^T \left( 2K \sum_{i=0}^{\lceil \log_2 T \rceil} 2^{-i} F_t(i)  - \sigma \right)\right] + \order(K),
   \end{align*}
   where we define $F_t(i) \triangleq \frac{1}{K}\sum_{a=1}^K \one[|\ell_t(a)-m_t(a)| \in (2^{-i-1}, 2^{-i}]]$. We decompose the summation above into two parts:
   \begin{align*}
        &\E\left[ \sum_{t=B+1}^T \left( K \sum_{i\in\calI} 2^{-i} F_t(i)  - \sigma \right)\right] + \E\left[ \sum_{t=B+1}^T \left( K \sum_{i\in\overline{\calI}} 2^{-i} F_t(i) \right)\right]
    \end{align*}
        where $\calI\triangleq \{i\leq \lceil \log_2 T \rceil: \alpha_i > \frac{360\log T}{B}\}$ and $\overline{\calI}\triangleq \{i\leq \lceil \log_2 T \rceil: \alpha_i \leq \frac{360\log T}{B}\}$.
    We bound the first term as:
    \begin{align*}
        &\sum_{t=B+1}^T \left( K \sum_{i\in\calI} 2^{-i} F_t(i)  - \sigma \right) \\
        &\leq \sum_{t=B+1}^T \frac{\left(K \sum_{i\in\calI} 2^{-i} F_t(i)\right)^2}{4\sigma}   \tag{AM-GM inequality}\\
        &\leq \sum_{t=B+1}^T \frac{K^2 (\log_2 T) \sum_{i\in\calI} 2^{-2i} F_t(i)^2}{4\sigma}   \tag{Cauchy-Schwarz} \\
        &= \otil\left(K^2\sum_{t=B+1}^T \frac{ \sum_{i\in\calI} 2^{-2i} F_t(i)}{\sigma}  \right) \tag{$0 \leq F_t(i) \leq 1$}.
    \end{align*}
    Now we take the expectation conditioned on all history before time $B$ and the high probability event in Lemma~\ref{lemma: coeff lower bound}. Noting that $\E[F_t(i)]=\alpha_i$ and plugging the value of $\sigma$, we arrive at
    \begin{align*}
        \otil\left(\frac{K^2T\sum_{i\in\calI} \alpha_i 2^{-2i} }{\sqrt{\calEhat}(dT)^{-\frac{1}{3}}} \right)
        &\leq \otil\left( \frac{K^2T\sum_{i\in\calI} \alpha_i 2^{-2i} }{\sqrt{T\sum_{i\in\calI}\alpha_i  2^{-2i} } (dT)^{-\frac{1}{3}}} \right) \tag{Eq.~\eqref{eqn: lemma 21 1}} \\
        &= \otil\left( K^2\sqrt{T\sum_{i\in\calI} \alpha_i 2^{-2i}} (dT)^{\frac{1}{3}} \right)\\
        &=\otil\left(K^2\sqrt{\calE}(dT)^{\frac{1}{3}}\right).
    \end{align*}
    We continue to bound the second term:
    \begin{align*}
        \E\left[ \sum_{t=B+1}^T \left( K \sum_{i\in\overline{\calI}} 2^{-i} F_t(i) \right)\right]
        &\leq KT \sum_{i\in\overline{\calI}} 2^{-i}\alpha_i \\
        &\leq K\left(T\sum_{i\in\overline{\calI}}\alpha_i\right)^{\frac{1}{2}}\left(T\sum_{i\in\overline{\calI}}2^{-2i}\alpha_i\right)^{\frac{1}{2}} \tag{Cauchy-Schwarz} \\
        &= \otil\left(K\sqrt{\frac{T}{B}}\times \sqrt{\calE}\right) \tag{definition of $\overline{\calI}$}.
    \end{align*}
    Combining the two terms finishes the proof.
\end{proof}

\begin{proof}[of Theorem~\ref{thm:adaptive_eps_greedy}]
    By the exact same argument as the proof of Theorem~\ref{thm:eps_greedy_weak} (\textit{cf.} Eq.~\eqref{eqn:useful_for_adaptive_version}), we bound the expected loss of the second phase of the algorithm by
    \begin{align*}
\E\left[\sum_{t=B+1}^T \ell_t(a_t)\right]
&= (T-B)\overline{\calL}(\pi^*) + \otil\left(\mu T\sigma + \mu \sqrt{\calE T} + \sqrt{\frac{d\calE}{\mu}} + \sigma\sqrt{dT}+ \frac{d}{\mu} \right).
\end{align*}
Further applying Lemma~\ref{lem: explore then epsilon error term} and bounding the regret of the first phase of the algorithm trivially by $B$, we have
\begin{align*}
         \Reg &= \otil\left(\mu T\sigma + \mu \sqrt{\calE T} + \sqrt{\frac{d\calE}{\mu}} + \sigma\sqrt{dT}+ \frac{d}{\mu} +  K^2\sqrt{\calE}(dT)^{\frac{1}{3}} + K\sqrt{\frac{\calE T}{B}} + B \right) \\
         &=\otil\left(\sqrt{\widehat{\calE}}(dT)^{\frac{1}{3}} +  \sqrt{\calE} (dT)^{\frac{1}{6}}  +  \sqrt{\widehat{\calE}}(dT)^{\frac{1}{6}} + (dT)^{\frac{1}{3}} + K^2\sqrt{\calE}(dT)^{\frac{1}{3}} + K\sqrt{\frac{\calE T}{B}} + B \right)  \tag{by our choices of $\mu$ and $\sigma$ defined in Algorithm~\ref{alg:explore then epsilon greedy}} \\
         &=\otil\left( K^2\sqrt{\calE}(dT)^{\frac{1}{3}} + K\sqrt{\frac{\calE T}{B}} + B  \right).   \tag{Lemma~\ref{lem: V' < V}}
    \end{align*}
   This finishes the proof.
\end{proof}


\subsection{Algorithms and Analysis for Multiple Predictors}
\label{app:multi_pred_iid}

In this section, we provide the complete pseudocode of our algorithm for learning with multiple predictors in the stochastic setting (Algorithm~\ref{alg:multi_pred_iid}) and its analysis.
As mentioned in Section~\ref{sec:multiple},
there are several extra ingredients compared to Algorithm~\ref{alg:eps_greedy}.
First, just as in Algorithm~\ref{alg:multi_pred_OExp4}, we maintain an active set of predictors $\calP_t$ by bookkeeping the remaining error budget $\widehat{V}_t^i$ for each predictor $i$.
One difference is that the budget starts from $2\calE^* + 8\log T$, which takes into account a direct deviation bound.
Another difference is that whenever the set $\calP_t$ is updated, we discard previous data and run the algorithm from scratch (see Step 3 of Algorithm~\ref{alg:multi_pred_iid}).
The reason to do so is to make sure that the data $\{x_s, \ell_s, m_s\}_{s=t_b}^t$ are i.i.d., where $m_t(a) = \min_{i\in\calP_t} m_t^i(a)$ depends on $\calP_t$.

Second, at the beginning of each round, we check if all predictors are consistent to some extent.
If not, that is, if there exist two predictors who disagree with each other by $\sigma/3$ on some action, then we simply choose this action deterministically, since this guarantees to reveal which predictor makes a large error for this round. See Step 1 of Algorithm~\ref{alg:multi_pred_iid}.
In this case, we set the loss estimators to be zero.

Finally, in the case when all predictors are consistent, instead of doing $\epsilon$-greedy as in Algorithm~\ref{alg:eps_greedy},
we deploy similar ideas as the minimax optimal algorithm \minimonster~\citep{AgarwalHsKaLaLiSc14} to come up with a sparse distribution $Q_t$ over the polices,  computed by solving an optimization problem described in Figure~\ref{fig:opt}.
At a high level, the optimization problem tries to find a policy with low empirical regret (Eq.~\eqref{eqn: minimonster condition 1}) and low empirical variance (Eq.~\eqref{eqn: minimonster condition 2}) simultaneously.
The difference compared to~\citep{AgarwalHsKaLaLiSc14} is that we apply action remapping as well as (clipped) Catoni's estimators.
The fact that this optimization problem can be solved efficiently is by the original arguments in~\citep{AgarwalHsKaLaLiSc14} and the binary search procedure we develop in Algorithm~\ref{alg:binary_search} (details omitted).

\begin{algorithm}[t]
\caption{\ILTCBMARC: \minimonster with Action Remapping and Catoni's estimator for Multiple predictors}
\label{alg:multi_pred_iid}
\textbf{Parameters: } $\calE^*, \sigma\in[0,1], \mu\in[0,1]$ \\
\textbf{Initialization}:  $\widehat{V}_1^i = 2\calE^* + 8\log T$ for all $i\in [M]$.  \\
$\calP_1 = [M]$. \\
$t_1 = 1$ \\
\For{$b=1, 2, \ldots$}{
\For{$t=t_b, \ldots$}{
    Receive $x_t$ and $m_t^i$ for all $i\in[M]$. \\
    Let $m_t(a) = \min_{i\in\calP_t} m_t^i(a)$ for all $a\in[K]$. \\
    Define $a_t^*, \calA_t, \phi_t$ according to Eq.~\eqref{eqn:mapping}. \\
    \ \\
    \textbf{Step 1. Check if the predictors are consistent, and calculate $p_t$} \\
    Let $B_t = \one[\forall a\in [K], \forall i, j
    \in \mathcal{P}_t, |m^{i}_t(a)-m^{j}_t(a)| \leq \frac{\sigma}{3}]$. \\
    Let $Q_t$ be a solution of the \textbf{Optimization Problem} defined in Figure~\ref{fig:opt}, and define
    \begin{align*}
        p_t(a) = \begin{cases}
             \one[a=a']   &\text{if\ } B_t=0 \text{\ \ \ ($a'$ is such that $\exists i,j\in\calP_t$, $|m_t^i(a')-m_t^j(a')|>\frac{\sigma}{3}$)}\\
             Q_t^\mu (a~|~x_t, m_t) &\text{if\ }B_t=1 \text{\ \ \ (see Eq.\eqref{eqn: definition of Q(|)} for the definition of $Q_t^\mu(a~|~x_t,m_t)$)}
        \end{cases}
    \end{align*}

    \ \\
    \textbf{Step 2. Choose an action and construct loss estimators} \\
    Sample $a_t \sim p_t$ and receive $\ell_t(a_t)$. \\
    Define
    \begin{align*}
        \elltil_{t}(a) =
            \left[\frac{(\ell_{t}(a)-m_{t}(a))\one[a_t=a]}{p_{t}(a)} + m_{t}(a) - m_{t}(a_t^*) \right] B_t
    \end{align*} \\
    \textbf{Step 3. Make updates} \\
    \For{$i\in\mathcal{P}_t$}{
         $\widehat{V}_{t+1}^i \leftarrow \widehat{V}_t^i - (\ell_t(a_t) - m_t^i(a_t))^2$
    }
    $\mathcal{P}_{t+1} = \left\{ i \in \mathcal{P}_t : \widehat{V}_{t+1}^i \geq 0\right\}$. \\

    \If{$\mathcal{P}_{t+1}=\emptyset$}{
        $\mathcal{P}_{t+1}\leftarrow [M], \ \ \  \widehat{V}_{t+1}^i\leftarrow 2\calE^*+8\log T,\ \forall i\in[M]$.
    }

    \ \\
    \If{$\mathcal{P}_{t+1}\neq \mathcal{P}_t$}{
         $t_{b+1} = t+1$ \\
         \textbf{break}
    }
}
}
\end{algorithm}

\begin{figure}[t]
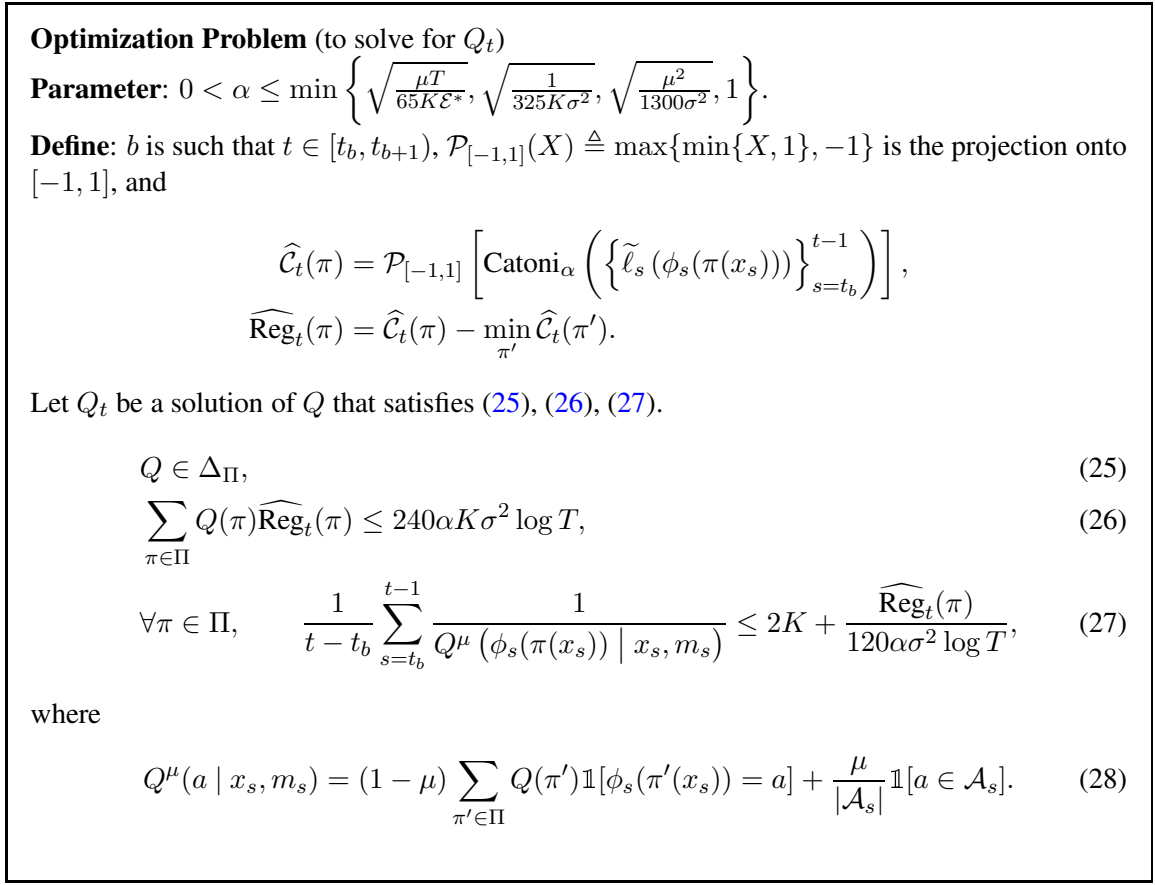

\begin{framed}
     \noindent \textbf{Optimization Problem}  (to solve for $Q_{t}$) \\
     \textbf{Parameter}: $0<\alpha \leq \min\left\{\sqrt{\frac{\mu T}{65K\calE^*}}, \sqrt{\frac{1}{325K\sigma^2}}, \sqrt{\frac{\mu^2}{1300\sigma^2}},1\right\}$.  \\
    \textbf{Define}: $b$ is such that $t\in [t_b, t_{b+1})$, $\calP_{[-1,1]}(X)\triangleq \max\{\min\{X,1\},-1\}$ is the projection onto $[-1,1]$, and
    \begin{align*}
        \hatcalC_{t}(\pi) &= \calP_{[-1,1]}\left[\catoni_\alpha\left( \left\{\elltil_s\left( \phi_s(\pi(x_s)) \right)\right\}_{s=t_b}^{t-1} \right)\right], \\
        \hatReg_{t}(\pi) &= \hatcalC_{t}(\pi)-\min_{\pi'}\hatcalC_{t}(\pi').
    \end{align*}

    Let $Q_{t}$ be a solution of $Q$ that satisfies \eqref{eqn: simplex constraint}, \eqref{eqn: minimonster condition 1}, \eqref{eqn: minimonster condition 2}.
    \begin{align}
        &Q\in \Delta_{\Pi},   \label{eqn: simplex constraint}\\
        &\sum_{\pi\in\Pi}Q(\pi)\hatReg_t(\pi) \leq 240\alpha K\sigma^2 \log T,   \label{eqn: minimonster condition 1} \\
        &\forall \pi\in\Pi, \qquad \frac{1}{t-t_b}\sum_{s=t_b}^{t-1}  \frac{1}{Q^\mu\left(\phi_s(\pi(x_s))~\big|~x_s, m_s\right)} \leq 2K + \frac{\hatReg_t(\pi)}{120\alpha\sigma^2 \log T},  \label{eqn: minimonster condition 2}
    \end{align}
    where
    \begin{align}
    Q^\mu(a~|~x_s,m_s)=
        (1-\mu)\sum_{\pi'\in\Pi}Q(\pi')\one[\phi_s(\pi'(x_s))=a] + \frac{\mu}{|\calA_s|}\one[a\in\calA_s].
    \label{eqn: definition of Q(|)}
    \end{align}
\end{framed}
\caption{An Optimization Problem for Algorithm~\ref{alg:multi_pred_iid}}
\label{fig:opt}
\end{figure}

To analyze the regret of Algorithm~\ref{alg:multi_pred_iid} and prove Theorem~\ref{thm:multi_pred_iid}, we introduce some definitions and useful lemmas.

\begin{definition}
    \label{definition: Lj}
    For some epoch $b$ of Algorithm~\ref{alg:multi_pred_iid} (with a corresponding fixed active set $\calP_{t_b}$),
    define \[\calC^{(b)}(\pi) \triangleq \E_{(x_t,m_t,\ell_t)\sim \calD}\left[ B_t\left(\ell_t(\phi_t(\pi(x_t)))  - \min_{a\in[K]} m_t(a)\right)\right]\] where $t=t_b$, and \[\Reg^{(b)}(\pi)\triangleq \calC^{(b)}(\pi) - \min_{\pi' \in \Pi}\calC^{(b)}(\pi').\]
    Also, define constant $C_0\triangleq \log(8T^4N^2)$.\footnote{Recall that $m_t$ does not depend on history once $\calP_{t_b}$ is fixed, and hence can be treated as jointly i.i.d. along with $x_t, \ell_t$ over an epoch with a fixed active set.}
\end{definition}

\begin{lemma}
    The Optimization problem defined in Figure~\ref{fig:opt} admits a solution.
\end{lemma}
\begin{proof}
    The proof follows Lemma 1 of~\citep{Luo17} (with $\beta=\frac{1}{120\alpha\sigma^2 \log T}$).
\end{proof}

\begin{lemma}
    \label{lemma: whp smaller than 2V*}
    For any $\delta \in (0,1)$, with probability at least $1-\delta$, \[\sum_{t=1}^T \|\ell_t-m^*_t\|_\infty^2 \leq 2\calE^* + 8\log(1/\delta). \]
\end{lemma}

\begin{proof}
    This is by the definition of $\calE^*$ and a direct application of Bernstein's inequality:
    \begin{align*}
        \sum_{t=1}^T \|\ell_t-m^*_t\|_\infty^2
        &\leq \calE^* + 4\sqrt{\log(1/\delta)\sum_{t=1}^T \E\left[\|\ell_t-m_t^*\|_\infty^4\right] } + 4\log(1/\delta) \\
        &\leq \calE^* + 4\sqrt{\log(1/\delta)\calE^*} + 4\log(1/\delta) \\
        &\leq 2\calE^* + 8\log(1/\delta),
    \end{align*}
    where the last step uses AM-GM inequality.
\end{proof}

\begin{lemma}
    \label{lemma: variance's variance}
    With probability at least $1-\frac{1}{T}$, for all $j$, $t\in[t_j, t_{j+1})$, all $Q\in\Delta_\Pi$, and all $\pi\in\Pi$, the following holds
    \begin{align*}
        \E_{(x_t, m_t, \ell_t)}\left[ \frac{1}{Q^\mu(\phi_t(\pi(x_t))~|~x_t, m_t)} \right] \leq \frac{6.4}{t-t_j} \sum_{s=t_j}^{t-1}\frac{1}{Q^\mu(\phi_s(\pi(x_s))~|~x_s, m_s)} + \frac{80C_0}{(t-t_j)\mu^2},
    \end{align*}
    where $C_0=\log\left(8T^4N^2\right)$.
\end{lemma}

\begin{proof}
    This lemma has appeared several times in the literature such as \cite[Theorem 6]{dudik2011efficient}, \cite[Lemma 10]{AgarwalHsKaLaLiSc14}, and \cite[Lemma 13]{chen2019new}. Basically this is a consequence of the contexts being i.i.d. generated, and is not related to the algorithm.
\end{proof}

\begin{lemma}
\label{lemma: estimator var 2}
    With probability at least $1-\frac{1}{T}$, we have for any $\pi$, $j$, and $t \in [t_j, t_{j+1})$,
    \begin{align*}
        \Var_{(x_t,m_t^\cdot,\ell_t, a_t)}\left[\elltil_t(\phi_t(\pi(x_t)))\right]\leq  \frac{4K\calE^*}{\mu T} + 20K\sigma^2 + \frac{6.4\hatReg_t(\pi)}{120\alpha \log T} + \frac{80\sigma^2C_0}{(t-t_j)\mu^2}.
    \end{align*}
\end{lemma}
\begin{proof}
We prove the lemma by the following sequence of direct calculations:
    \begin{align*}
        &\Var_{(x_t,m_t,\ell_t, a_t)}\left[\elltil_t(\phi_t(\pi(x_t)))\right] \\
        &\leq 2\mathbb{E}_{(x_t,m_t,\ell_t, a_t)}\left[\left(\frac{(\ell_{t}(\phi_t(\pi(x_t)))-m_{t}(\phi_t(\pi(x_t)))\one[a_t=\phi_t(\pi(x_t))]}{p_{t}(\phi_t(\pi(x_t)))}\right)^2 B_t \right] \\
        &\qquad \qquad + 2\mathbb{E}_{(x_t, m_t)}\left[\left(m_{t}(\phi_t(\pi(x_t))) - m_{t}(a_t^*)\right)^2\right] \\
        &\leq 4\mathbb{E}_{(x_t,m_t,\ell_t, a_t)}\left[\left(\frac{(\ell_{t}(\phi_t(\pi(x_t)))-m^*_{t}(\phi_t(\pi(x_t))))\one[a_t=\phi_t(\pi(x_t))]}{p_{t}(\phi_t(\pi(x_t)))}\right)^2 B_t \right] \\
        &\qquad \qquad+ 4\mathbb{E}_{(x_t,m_t,\ell_t, a_t)}\left[\left(\frac{(m^*_{t}(\phi_t(\pi(x_t)))-m_{t}(\phi_t(\pi(x_t))))\one[a_t=\phi_t(\pi(x_t))]}{p_{t}(\phi_t(\pi(x_t)))}\right)^2 B_t \right] \\
        &\qquad \qquad + 2\sigma^2 \\
        &\leq 4\mathbb{E}_{(x_t,m_t^\cdot,\ell_t)}\left[\frac{\left(\ell_{t}(\phi_t(\pi(x_t)))-m^*_{t}(\phi_t(\pi(x_t)))\right)^2}{p_{t}(\phi_t(\pi(x_t)))} B_t \right] \\
        &\qquad \qquad +  4\mathbb{E}_{(x_t,m_t,\ell_t)}\left[\frac{\left(m^*_{t}(\phi_t(\pi(x_t)))-m_{t}(\phi_t(\pi(x_t)))\right)^2}{p_{t}(\phi_t(\pi(x_t)))} B_t\right] \\
        &\qquad \qquad + 2\sigma^2 \\
        &\leq \frac{4K\calE^*}{\mu T} + 4 \mathbb{E}_{(x_t,m_t,\ell_t)}\left[\frac{(\frac{\sigma}{3})^2}{Q_{t}^\mu(\phi_t(\pi(x_t))\;|\; x_t, m_t )} \right]  + 2\sigma^2 \\
        &\leq \frac{4K\calE^*}{\mu T} + 2\sigma^2 +  \sigma^2 \left( \frac{6.4}{t-t_j} \sum_{s=t_j}^{t-1} \frac{1}{Q_t^\mu(\phi_s(\pi(x_s)) \;|\; x_s, m_s)} + \frac{80C_0}{(t-t_j)\mu^2} \right)   \tag{Lemma~\ref{lemma: variance's variance}} \\
        &\leq \frac{4K\calE^*}{\mu T} + 2\sigma^2 +  \sigma^2 \left( 6.4\times \left(2K+ \frac{\hatReg_t(\pi)}{120\alpha\sigma^2 \log T}\right) + \frac{80C_0}{(t-t_j)\mu^2} \right)  \tag{Eq.~\eqref{eqn: minimonster condition 2}} \\
        &\leq \frac{4K\calE^*}{\mu T} + 20K\sigma^2 + \frac{6.4\hatReg_t(\pi)}{120\alpha \log T} + \frac{80\sigma^2C_0}{(t-t_j)\mu^2}.
            \end{align*}
\end{proof}

\begin{lemma}
    \label{lemma: estimator mean 2}
    For any $\pi$, $j$, and $t\in[t_j, t_{j+1})$, we have
    \begin{align*}
        \E_{(x_t,m_t,\ell_t, a_t)}\left[\elltil_t(\phi_t(\pi(x_t)))\right] =\calC^{(j)}(\pi).
    \end{align*}
    (Recall the definition of $\calC^{(j)}(\pi)$ in Definition~\ref{definition: Lj}.)
\end{lemma}
\begin{proof}
By direct calculation, we have
    \begin{align*}
        &\E_{(x_t,m_t,\ell_t, a_t)}\left[\elltil_t(\phi_t(\pi(x_t)))\right] \\
        &= \E_{(x_t,m_t,\ell_t, a_t)}\left[ \left(\frac{(\ell_{t}(\phi_t(\pi(x_t)))-m_{t}(\phi_t(\pi(x_t))))\one[a_t=\phi_t(\pi(x_t))]}{p_{t}(\phi_t(\pi(x_t)))} + m_{t}(\phi_t(\pi(x_t))) - m_{t}(a_t^*) \right) B_t \right] \\
        &= \E_{(x_t,m_t,\ell_t)}\left[ \left(\ell_{t}(\phi_t(\pi(x_t))) - m_{t}(\phi_t(\pi(x_t))) + m_{t}(\phi_t(\pi(x_t))) - m_{t}(a_t^*) \right) B_t \right] \\
        &= \E_{(x_t,m_t,\ell_t)}\left[ B_t\ell_{t}(\phi_t(\pi(x_t))) - B_t \min_{a} m_t(a) \right] \\
        &= \calC^{(j)}(\pi),
    \end{align*}
finishing the proof.
\end{proof}

%


\begin{lemma} {\color{blue}}
    \label{lemma: relating Reg and hatReg}
    Recall the definition of $\Reg^{(j)}(\pi)$ in Definition~\ref{definition: Lj}. With probability at least $1-\frac{1}{T}$, we have for any $j$ and $t\in [t_j, t_{j+1})$,
    \begin{align}
        &\Reg^{(j)}(\pi)\leq 2 \hatReg_t(\pi) + \frac{40\alpha K\calE^*}{\mu T} + 200\alpha K\sigma^2+ \frac{800\alpha\sigma^2 C_0\log T}{(t-t_j)\mu^2} + \frac{20\log(NT^2)\log T}{\alpha(t-t_j)}, \label{eqn: induction 1}   \\
        &\hatReg_t(\pi)\leq 2\Reg^{(j)}(\pi) +  \frac{40\alpha K\calE^*}{\mu T} + 200\alpha K\sigma^2+ \frac{800\alpha\sigma^2 C_0\log T}{(t-t_j)\mu^2} + \frac{20\log(NT^2)\log T}{\alpha(t-t_j)}.   \label{eqn: induction 2}
    \end{align}
\end{lemma}
\begin{proof}
    We first notice that when $t-t_j\leq 20\log (NT^2)\log T$, both inequalities hold trivially because the left-hand side is at most $1 \leq \frac{20\log(NT^2)\log T}{t-t_j}\leq \frac{20\log (NT^2)\log T}{\alpha(t-t_j)}$.  Thus we only need to consider the case $t-t_j\geq 20\log(NT^2)\log T$. 

    We prove them by induction on $t$. Let $\pi^*=\argmin_{\pi'} \calC^{(j)}(\pi')$. Assume \eqref{eqn: induction 1} and \eqref{eqn: induction 2} hold for $t_j, \ldots, t-1$. 
By the induction hypothesis and Lemma~\ref{lemma: estimator var 2}, for any $\pi$, the conditional variance of $\ellhat_s(\phi_s(\pi(x_s)))$ can be upper bounded as follows: 
\begin{align*}
      &\Var\left[\elltil_s(\phi_s(\pi(x_s)))\right] \\
      &\leq \frac{4K\calE^*}{\mu T} + 20K\sigma^2 +  \frac{6.4\hatReg_s(\pi)}{120\alpha\log T} + \frac{80\sigma^2C_0}{(s-t_j)\mu^2} \\
      &\leq \frac{4K\calE^*}{\mu T} + 20K\sigma^2  + \frac{80\sigma^2C_0}{(s-t_j)\mu^2} \\
      &\qquad  + \frac{6.4}{120\alpha \log T}\left(2\Reg^{(j)}(\pi) + \frac{40\alpha K\calE^*}{\mu T} + 200\alpha K\sigma^2+ \frac{800\alpha\sigma^2 C_0\log T}{(s-t_j)\mu^2} + \frac{20\log(NT^2)\log T}{\alpha(s-t_j)}\right)\\
      &\leq \frac{\Reg^{(j)}(\pi)}{8\alpha} + \frac{6.5K\calE^*}{\mu T} + 32.5 K\sigma^2 + \frac{130\sigma^2C_0}{(s-t_j)\mu^2} + \frac{3.25\log(NT^2)}{\alpha^2(s-t_j)} \triangleq V_s. 
\end{align*}
Let $V=\sum_{s=t_j}^{t-1}V_s$. We first verify that $t-t_j \geq \alpha^2 V + 2\log(NT^2)$. This can be seen by the following: 
\begin{align}
&\alpha^2 V+2\log(NT^2)   \nonumber \\
&\leq \alpha^2\sum_{s=t_j}^{t-1} \left( \frac{\Reg^{(j)}(\pi)}{8\alpha} + \frac{6.5K\calE^*}{\mu T} + 32.5 K\sigma^2 + \frac{130\sigma^2C_0}{(s-t_j)\mu^2} + \frac{3.25\log(NT^2)}{\alpha^2(s-t_j)} \right) +2\log(NT^2) \nonumber \\
&\leq \alpha^2 (t-t_j) \left(\frac{1}{8\alpha} + \frac{6.5K\calE^*}{\mu T} + 32.5 K\sigma^2 + \frac{130\sigma^2C_0\log T}{(t-t_j)\mu^2} + \frac{3.25\log(NT^2)\log T}{\alpha^2(t-t_j)} \right) +2\log(NT^2)\nonumber \\
&\leq (t-t_j)\left(\frac{\alpha}{8} + \frac{6.5\alpha^2K\calE^*}{\mu T} + 32.5 \alpha^2K\sigma^2 + \frac{130\alpha^2\sigma^2C_0\log T}{(t-t_j)\mu^2} + \frac{3.25\log(NT^2)\log T}{(t-t_j)} \right) +2\log(NT^2)\nonumber \\
&\leq (t-t_j)\left(\frac{1}{8} + 0.1 + 0.1 + \frac{C_0\log T}{10(t-t_j)} + \frac{3.25}{20} \right) + 0.1(t-t_j) \tag{by the constraints on $\alpha$}\nonumber \\
&\leq (t-t_j)\left(\frac{1}{8} + 0.1 + 0.1 + \frac{16\log(NT^2)\log T}{10\times 20\log(NT^2)\log T} + \frac{3.25}{20} \right) + 0.1(t-t_j) \tag{by the definition of $C_0$}\nonumber \\
&\leq t-t_j.   \label{eqn: check catoni condition}
\end{align}
Because of Eq.~\eqref{eqn: check catoni condition}, we are now able to use Lemma~\ref{lemma: catoni concentration} for the samples $\left\{\elltil_s(\phi_s(\pi(x_s)))\right\}_{s=t_j}^{t-1}$ with $\delta=\frac{1}{NT^2}$. By Lemmas~\ref{lemma: catoni concentration} and~\ref{lemma: estimator mean 2}, we have with probability $1-\delta$ that
    \begin{align*}
        &\Reg^{(j)}(\pi) 
        = \calC^{(j)}(\pi) - \calC^{(j)}(\pi^*)\\
        &\leq \hatcalC_t(\pi) - \hatcalC_t(\pi^*)\\
        &\qquad + \frac{\alpha}{t-t_j}\sum_{s=t_j}^{t-1}\left(\frac{\Reg^{(j)}(\pi)}{8\alpha} + \frac{\Reg^{(j)}(\pi^*)}{8\alpha} + \frac{13K\calE^*}{\mu T} + 65 K\sigma^2 + \frac{260\sigma^2C_0}{(s-t_j)\mu^2} + \frac{6.5\log(NT^2)}{\alpha^2(s-t_j)}\right) \\
        &\qquad + \frac{4\log(NT^2)}{\alpha(t-t_j)} \\
        &\leq  \hatReg_t(\pi) + \frac{1}{8}\Reg^{(j)}(\pi) + \frac{13\alpha K\calE^*}{\mu T} + 65\alpha K\sigma^2 + \frac{260\alpha\sigma^2 C_0\log T}{(t-t_j)\mu^2} + \frac{10.5\log(NT^2)\log T}{\alpha(t-t_j)}  \tag{using $\Reg^{(j)}(\pi^*)=0$}.
    \end{align*}
    Rearranging the above inequality gives 
\begin{align*}
     \Reg^{(j)}(\pi) \leq \frac{8}{7}\hatReg_t(\pi)  +  \frac{15\alpha K\calE^*}{\mu T} + 75\alpha K\sigma^2 + \frac{300\alpha\sigma^2 C_0\log T}{(t-t_j)\mu^2} + \frac{12\log(NT^2)\log T}{\alpha(t-t_j)}, 
\end{align*}
    proving Eq.~\eqref{eqn: induction 1}.
    Similarly,
    \begin{align*}
        &\hatReg_t(\pi) \nonumber = \hatcalC_t(\pi) - \hatcalC_t(\widehat{\pi})\nonumber \\
        &\leq \calC_t(\pi) - \calC_t(\widehat{\pi}) \\ 
        &\qquad + \frac{\alpha}{t-t_j}\sum_{s=t_j}^{t-1}\left(\frac{\Reg^{(j)}(\pi)}{8\alpha} + \frac{\Reg^{(j)}(\widehat{\pi})}{8\alpha} + \frac{13K\calE^*}{\mu T} + 65 K\sigma^2 + \frac{260\sigma^2C_0}{(s-t_j)\mu^2} + \frac{6.5\log(NT^2)}{\alpha^2(s-t_j)}\right)  \\
        &\qquad + \frac{4\log(NT^2)}{\alpha(t-t_j)} \\
        &\leq \frac{9}{8}\Reg^{(j)}(\pi) + \frac{13\alpha K\calE^*}{\mu T} + 65\alpha K\sigma^2 + \frac{260\alpha\sigma^2 C_0\log T}{(t-t_j)\mu^2} + \frac{10.5\log(NT^2)\log T}{\alpha(t-t_j)} \\
        &\qquad + \frac{1}{8}\left(2\hatReg_t(\widehat{\pi}) + \frac{40\alpha K\calE^*}{\mu T} + 200\alpha K\sigma^2 + \frac{800\alpha\sigma^2 C_0\log T}{(t-t_j)\mu^2} + \frac{20\log(NT^2)\log T}{\alpha(t-t_j)} \right)   \tag{using \eqref{eqn: induction 1}, which we just proved above} \\
        &\leq \frac{9}{8}\Reg^{(j)}(\pi) + \frac{18\alpha K\calE^*}{\mu T} + 90\alpha K\sigma^2 + \frac{360\alpha\sigma^2 C_0\log T}{(t-t_j)\mu^2} + \frac{13\log(NT^2)\log T}{\alpha(t-t_j)}.    \tag{using $\hatReg_t(\widehat{\pi})=0$}
    \end{align*} 
This proves Eq.~\eqref{eqn: induction 2} and finishes the induction. Recall that we pick $\delta=\frac{1}{NT^2}$. Thus the total failure probability is at most $\frac{1}{NT^2}\times TN\leq \frac{1}{T}$. 
\end{proof}

\begin{lemma}
    \label{lemma: regret sum over Bt steps}
    With probability at least $1-\frac{1}{T}$, we have for any $\pi^*$, $j$, and $t\in[t_j, t_{j+1})$, 
    \begin{align*}
        &\E_{(x_t,m_t,\ell_t, a_t)}\left[B_t\ell_t(a_t)  \right] \\
        &\leq \mathbb{E}_{(x_t,m_t,\ell_t)}\left[B_t \ell_t(\phi_t(\pi^*(x_t)))  \right] \\
        &\qquad \qquad + \otil\left( \frac{\alpha K\calE^*}{\mu T} + \alpha K\sigma^2+ \frac{\alpha\sigma^2  \log N}{(t-t_j)\mu^2} + \frac{\log N}{\alpha(t-t_j)} + \mu\sqrt{\frac{\calE^*}{T}} + \mu\sigma\right). 
    \end{align*}
\end{lemma}

\begin{proof}
    By the way $a_t$ is chosen when $B_t = 1$, we have
    \begin{align}
        &\E_{(x_t,m_t,\ell_t, a_t)}\left[B_t \ell_t(a_t) \right]   \nonumber \\
        &=\E_{(x_t,m_t, \ell_t)}\left[B_t \sum_{a\in[K]} p_t(a)\ell_t(a)  \right] \nonumber  \\
        &=(1-\mu)\E_{(x_t,m_t, \ell_t)}\left[B_t \sum_{\pi\in\Pi} Q_t(\pi) \ell_t(\phi_t(\pi(x_t))) \right] + \mu\E_{(x_t,m_t,\ell_t)}\left[ \frac{B_t }{|\calA_t|} \sum_{a\in \calA_t} \ell_t(a) \right]. \label{eqn: minimonster final to bound}
    \end{align}
    We continue to bound the first term in Eq.~\eqref{eqn: minimonster final to bound} as: 
    \begin{align*}
        &\sum_{\pi\in\Pi} Q_t(\pi) \E_{(x_t, m_t, \ell_t)}\left[ B_t\ell_t(\phi_t(\pi(x_t)))  \right]\\
        &= \sum_{\pi\in\Pi} Q_t(\pi) \calC^{(j)}(\pi) + \E_{m_t}\left[B_t \min_a m_t(a)\right]  \tag{Definition~\ref{definition: Lj}} \\
        &\leq \sum_{\pi\in\Pi} Q_t(\pi) \Reg^{(j)}(\pi) + \calC^{(j)}(\pi^*) + \E_{m_t}\left[B_t \min_a m_t(a)\right] \tag{Definition~\ref{definition: Lj}} \\
        &\leq \sum_{\pi\in\Pi} Q_t(\pi) \Bigg(2\hatReg_t(\pi) +  \frac{40\alpha K\calE^*}{\mu T} + 200\alpha K\sigma^2+ \frac{800\alpha\sigma^2 C_0\log T}{(t-t_j)\mu^2} + \frac{20\log(NT^2)\log T}{\alpha(t-t_j)} \Bigg) \\
        &\qquad + \E_{(x_t,m_t,\ell_t)}\left[ B_t\ell_t(\phi_t(\pi^*(x_t)))  \right] \tag{Lemma~\ref{lemma: relating Reg and hatReg} and Definition~\ref{definition: Lj}}\\
        &= \order\left( \frac{\alpha K\calE^*}{\mu T} + \alpha K\sigma^2+ \frac{\alpha\sigma^2 C_0\log T}{(t-t_j)\mu^2} + \frac{\log(NT^2)\log T}{\alpha(t-t_j)} \right) +  \E_{(x_t,m_t,\ell_t)}\left[ B_t\ell_t(\phi_t(\pi^*(x_t))) \right]. \tag{Eq.\eqref{eqn: minimonster condition 1}}  \\
        &= \otil\left( \frac{\alpha K\calE^*}{\mu T} + \alpha K\sigma^2+ \frac{\alpha\sigma^2 \log N}{(t-t_j)\mu^2} + \frac{\log N}{\alpha(t-t_j)} \right) +  \E_{(x_t,m_t,\ell_t)}\left[ B_t\ell_t(\phi_t(\pi^*(x_t))) \right].
    \end{align*}
    The second term in Eq.~\eqref{eqn: minimonster final to bound} can be bounded as follows (without the $\mu$ factor): 
    \begin{align*}
        &\E_{(x_t,m_t,\ell_t)}\left[ \frac{1}{|\calA_t|} \sum_{a\in \calA_t} \ell_t(a)B_t  \right] \\
        &= \E_{(x_t,m_t,\ell_t)}\left[ \frac{1}{|\calA_t|} \sum_{a\in \calA_t} \left(  \ell_t(a) - m_t(a) + m_t(a) - m_t(\phi_t(\pi^*(x_t))) \right)B_t  \right] \\
        & \qquad  + \E_{(x_t,m_t,\ell_t)}\left[ \frac{1}{|\calA_t|} \sum_{a\in \calA_t} \left(   m_t(\phi_t(\pi^*(x_t))) - \ell_t(\phi_t(\pi^*(x_t))) \right)B_t  \right]   +  \E_{(x_t,m_t,\ell_t)}\left[ B_t\ell_t(\phi_t(\pi^*(x_t)))  \right] \\
        &\leq \E_{(x_t,m_t,\ell_t)}\left[ 2 \max_{a}|\ell_t(a)-m_t(a)|B_t  +  \sigma \right]  +  \E_{(x_t,m_t,\ell_t)}\left[ B_t\ell_t(\phi_t(\pi^*(x_t)))\right] \\
        &\leq \E_{(x_t,m_t,\ell_t)}\left[ 2 \max_{a}|\ell_t(a)-m_t^*(a)|B_t + \frac{2\sigma}{3}  +  \sigma \right]  +  \E_{(x_t,m_t,\ell_t)}\left[ B_t\ell_t(\phi_t(\pi^*(x_t))) \right] \tag{definition of $B_t$} \\
        &\leq 2\left(\sqrt{\frac{\calE^*}{T}} +\sigma\right)  +  \E_{(x_t,m_t,\ell_t)}\left[ B_t\ell_t(\phi_t(\pi^*(x_t))) \right]. \tag{definition of $\calE^*$ and Jensen's inequality} 
    \end{align*}
    Combining these two bounds finishes the proof. 
\end{proof}

\begin{lemma}
    \label{lemma: regret_1: wrt bar pi}
    Algorithm~\ref{alg:multi_pred_iid} ensures for any $\pi^*$,
    \begin{align*}
        &\E\left[\sum_{t=1}^T B_t\ell_t(a_t) - B_t\ell_t(\phi_t(\pi^*(x_t))) \right] \\
        &\leq \otil\left(\frac{\alpha K\calE^*}{\mu} + \alpha TK\sigma^2+ \frac{M\alpha\sigma^2 \log N}{\mu^2} + \frac{M\log N}{\alpha} + \mu\sqrt{T\calE^*} + T\mu\sigma\right)
    \end{align*}
\end{lemma}
\begin{proof}
    This is proven by summing the statement of Lemma~\ref{lemma: regret sum over Bt steps} over $t$ and noticing that with probability $1-\frac{1}{T}$, there exists a predictor with $\sum_{t=1}^T\|\ell_t-m_t^i\|_\infty\leq 2\calE^* + 8\log T$ and thus there are at most $M$ episodes (see Lemma~\ref{lemma: whp smaller than 2V*}).
\end{proof}

\begin{lemma}
    \label{lemma: lemma: regret_2: unexplored bias}
    Algorithm~\ref{alg:multi_pred_iid} ensures
    \begin{align*}
        \mathbb{E}\left[\sum_{t=1}^T B_t\ell_t(\phi_t(\pi^*(x_t))) - B_t\ell_t(\pi^*(x_t))\right] \leq \order\left(\frac{\calE^*}{\sigma}\right). 
    \end{align*}
\end{lemma}
\begin{proof}
The proof is similar to that of Lemma~\ref{lem:eps_greedy_bias}:
    \begin{align*}
        & \mathbb{E}\left[\sum_{t=1}^T B_t\ell_t(\phi_t(\pi^*(x_t))) - B_t\ell_t(\pi^*(x_t))\right] \\
        &= \mathbb{E}\left[B_t\sum_{t=1}^T\one[\pi^*(x_t) \notin \calA_t] \left(\ell_t(a_t^*) - \ell_t(\pi^*(x_t))\right)\right] \\
        &= \mathbb{E}\left[\sum_{t=1}^T B_t \one[\pi^*(x_t) \notin \calA_t] \left(\ell_t(a_t^*) - m_t(a_t^*) + m_t(a_t^*) - m_t(\pi^*(x_t)) + m_t(\pi^*(x_t)) - \ell_t(\pi^*(x_t))\right)\right] \\
        &\leq \mathbb{E}\left[\sum_{t=1}^T B_t \one[\pi^*(x_t) \notin \calA_t] \left(-\sigma + 2\|\ell_t-m_t^*\|_\infty + 2\|m_t-m_t^*\|_\infty \right)\right] \\  
        &\leq \mathbb{E}\left[\sum_{t=1}^T B_t \one[\pi^*(x_t) \notin \calA_t] \left(2\|\ell_t-m_t^*\|_\infty - \frac{\sigma}{3} \right)\right] \tag{Definition of $B_t$}\\  
        &\leq \order\left( \E\left[\sum_{t=1}^T \frac{\|\ell_t-m_t^*\|_\infty^2}{\sigma}\right]  \right)    \tag{AM-GM}\\
        &= \order\left(\frac{\calE^*}{\sigma}\right). 
    \end{align*}
\end{proof}

\begin{lemma}
    \label{lemma: regret_3: Bt=0}
    With probability $1-\frac{1}{T}$, 
    \begin{align*}
        \sum_{t=1}^T (1-B_t) \leq \otil\left(\frac{M(1+\calE^*)}{\sigma^2}\right). 
    \end{align*}
\end{lemma}
\begin{proof}
    With probability $1-\frac{1}{T}$, there exists a predictor with $\sum_{t=1}^T\|\ell_t-m_t^i\|_\infty\leq 2\calE^* + 8\log T$ and thus there are at most $M$ episodes (see Lemma~\ref{lemma: whp smaller than 2V*}). Under this event, every time when $B_t=0$, there exist $i, i' \in \mathcal{P}_t$ such that $|m_t^i(a_t)-m_t^{i'}(a_t)| \geq \frac{\sigma}{3}$. Therefore, the total budget $\sum_{i\in\mathcal{P}_t} \widehat{V}_i$ decreases by at least $\left( \ell_t(a_t) - m_t^i(a_t) \right)^2 + \left(\ell_t(a_t) - m_t^{i'}(a_t)\right)^2 \geq \frac{1}{2}\left(m_t^i(a_t)-m_t^{i'}(a_t)\right)^2 \geq \frac{\sigma^2}{18}$. Realizing that the initial total budget is $\otil(M(1+\calE^*))$ finishes the proof.
\end{proof}
Finally, we are ready to prove Theorem~\ref{thm:multi_pred_iid}.
\begin{proof}[of Theorem~\ref{thm:multi_pred_iid}] 
    Combining Lemmas~\ref{lemma: regret_1: wrt bar pi}, \ref{lemma: lemma: regret_2: unexplored bias}, \ref{lemma: regret_3: Bt=0} and picking the optimal parameters in each step, we bound the regret as:
    \begin{align*}
        &\E\left[\sum_{t=1}^T \ell_t(a_t) - \ell_t(\pi^*(x_t)) \right] \\
        &\leq \otil\left(  \alpha TK\sigma^2 + \frac{\alpha K\calE^*}{\mu} + \frac{M\alpha\sigma^2 \log N}{\mu^2} + \frac{M\log N}{\alpha} + \mu\sqrt{T\calE^*} + T\mu\sigma + \frac{M(1+\calE^*)}{\sigma^2}\right)\\
        &= \otil\left( \sqrt{MK(\log N)T\sigma^2} + \sqrt{\frac{MK(\log N)\calE^*}{\mu}} + \frac{M(\log N)\sigma }{\mu} +\mu\sqrt{T\calE^*} + T\mu\sigma + \frac{M(1+\calE^*)}{\sigma^2} \right)  \\
        &\qquad \qquad \qquad \qquad  \qquad \qquad \qquad \qquad+ \otil\left(M\log N \left(\sqrt{\frac{K\calE^*}{\mu T}} + \sqrt{K\sigma^2} + \frac{\sigma}{\mu}\right)\right)
        \tag{picking the optimal $\alpha$ under the constraints of $\alpha$}\\
        &= \otil\left( \sqrt{MdT\sigma^2} + \sqrt{\frac{Md\calE^*}{\mu}} + \frac{Md\sigma }{\mu} +\mu\sqrt{T\calE^*} + T\mu\sigma + \frac{M(1+\calE^*)}{\sigma^2} \right) \tag{assume $T\geq M\log N $} \\
        &= \otil\left( \sqrt{MdT}\sigma + M^{\frac{1}{3}}d^{\frac{1}{3}}\sqrt{\calE^*}T^{\frac{1}{6}} + M^{\frac{1}{3}}d^{\frac{1}{3}}\calE^{*\frac{1}{3}}\sigma^{\frac{1}{3}} T^{\frac{1}{3}} + \sqrt{Md}\calE^{*\frac{1}{4}}T^{\frac{1}{4}}\sqrt{\sigma} + \frac{M(1+\calE^*)}{\sigma^2}\right) \tag{picking the optimal $\mu$}\\
        &= \otil\left( (M^2 d)^{\frac{1}{3}}(1+\calE^*)^{\frac{1}{3}}T^{\frac{1}{3}} + (Md)^{\frac{1}{3}}\sqrt{\calE^*}T^{\frac{1}{6}} + (M^3 d^2)^{\frac{1}{7}} (1+\calE^*)^{\frac{3}{7}}T^{\frac{2}{7}} + M^{\frac{3}{5}}d^{\frac{2}{5}}(1+\calE^*)^{\frac{2}{5}}T^{\frac{1}{5}} \right) \tag{picking the optimal $\sigma$}\\
        &= \otil\left( M^{\frac{2}{3}}d^{\frac{2}{5}}(1+\calE^*)^{\frac{1}{3}}T^{\frac{1}{3}} \right), 
    \end{align*}
    where the last step uses the fact that we only care about the case when $1+\calE^* \leq \sqrt{T}$ to simplify the bound.
\end{proof}

\end{document}